\documentclass{article}

\usepackage[preprint]{neurips_2026}

\usepackage[utf8]{inputenc} 
\usepackage[T1]{fontenc}    
\usepackage{hyperref}       
\usepackage{url}            
\usepackage{booktabs}       
\usepackage{amsfonts}       
\usepackage{nicefrac}       
\usepackage{microtype}      
\usepackage{xcolor}         

\usepackage[font=footnotesize,skip=0pt]{caption}

\usepackage[utf8]{inputenc} 
\usepackage[T1]{fontenc}    
\usepackage{hyperref}       
\usepackage{url}            
\usepackage{booktabs}       
\usepackage{amsfonts}       
\usepackage{nicefrac}       
\usepackage{microtype}      
\usepackage{xcolor}         
\usepackage{adjustbox} 
\usepackage{tabularx}
\usepackage{pifont}

\usepackage{epsfig}
\usepackage{graphicx}
\usepackage{amsmath}
\usepackage{amsthm}
\usepackage{amssymb}
\usepackage{multirow}
\usepackage{booktabs}
\usepackage{subcaption}
\usepackage{algorithm}
\usepackage{algorithmic}
\usepackage{colortbl}
\definecolor{ForestGreen}{RGB}{34,139,34}
\definecolor{Crimson}{RGB}{220,20,60}

 \newtheorem{lemma}{Lemma}
  \newtheorem{proposition}{Proposition}

\newcommand{\dlt}[1]{{\footnotesize\color{gray}$#1$}}

\definecolor{LightCyan}{rgb}{0.88,1,1}
\definecolor{Gray}{gray}{0.85}
\definecolor{LightGray}{gray}{0.5}
\definecolor{LightBlue}{rgb}{0.678,0.847,0.902} 

\usepackage[font=footnotesize,skip=1pt]{caption}

\title{Diffusion Meta-Prompting and Steering for Generalizable Foundation Model Adaptation}

\author{
Deepak Sridhar$^{1}$ \And
Yi Li$^{2}$ \And
Kartikeya Bhardwaj$^{2}$ \And
Shuangjun Liu$^{2}$ \And
Taotao Jing$^{2}$ \And
Yuan Li$^{2}$ \And
Shuai Zhang$^{2}$ \And
Jiancheng Lyu$^{2}$ \And
Dashan Gao$^{2}$ \And
Nuno Vasconcelos$^{1}$\And
$^{1}$University of California, San Diego\quad
$^{2}$Qualcomm AI Research\thanks{Qualcomm AI Research is an initiative of Qualcomm Technologies, Inc.}\\
\small{\texttt{desridha@ucsd.edu, nvasconcelos@ucsd.edu}}
}

\begin{document}

\maketitle

\vspace{-8mm}
\begin{abstract}
  Prompt learning is a popular method for adapting foundation models, but learned prompts are typically task-specific and fail to generalize to new classes, domains, or compositions of tasks. In this paper, we introduce a \textbf{Diffusion Meta-Prompt (DMP)} model , a framework that models the distribution of learned prompts using diffusion models. Given a repository of previously learned prompts, DMP is trained and sampled without access to the original task examples or task losses, and synthesizes new prompts conditioned on natural language task descriptions. To improve the sampling stability, we introduce a test-time steering strategy for DMP, which uses the best training-selected prompt in the repository as a latent anchor during diffusion sampling, without retraining the DMP or accessing test classes. DMP improves generalization across classification, retrieval and text-to-image generation tasks, supports concept composition and negative prompting without explicit training. It reduces storage and inference costs by over 90\% compared to prompt retrieval methods. For composite classification, DMP achieves upto \textbf{2.0\%} average gain over prior meta-learning methods across 55 pairs of datasets with gains as high as \textbf{8.5\%} on specific pairs such as Eurosat and Flowers.  DMP also enhances cross-task generalization with $\sim$\textbf{2-9\%} improvement for hierarchical classification task. We further provide a theoretical guarantee bounding the expected task loss of prompts sampled from a DMP. Code is available here: \href{https://github.com/DeepakSridhar/dmp}{DMP}

\end{abstract}

\section{Introduction}
\label{sec:intro}

\begin{figure*}[t]
\centering
\begin{minipage}[t]{0.98\textwidth}
  \centering
  \includegraphics[width=\linewidth, trim=50 150 50 120, clip]{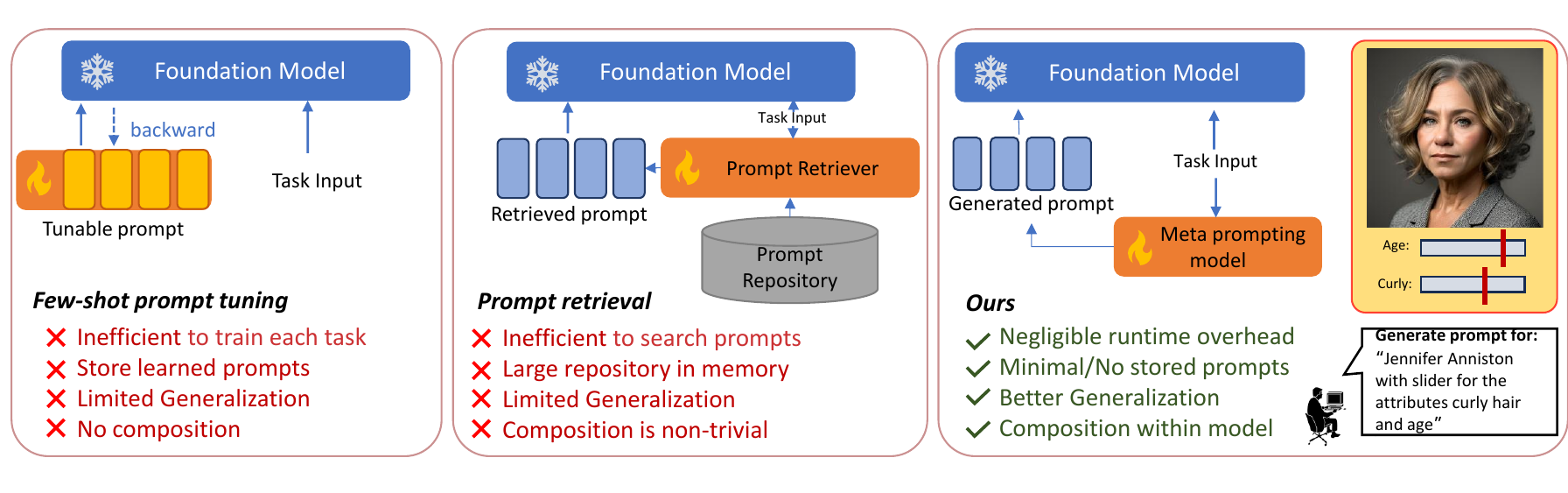}
  \caption{\label{fig:teaser_comparison} \textbf{DMP versus existing approaches}: DMP models provide a natural language interface for the adaptation of foundation models, while eliminating the complexities of large prompt repositories, offering better generalization and composition with a minimal runtime overhead due to the feasibility of using small DMP models and fast sampling.}
\end{minipage}
\hfill
\begin{minipage}[t]{0.98\textwidth}
\centering
\scriptsize
\captionof{table}{\textbf{Conceptual differences between DMP and prior meta-learning methods.} 
}
\setlength{\tabcolsep}{2pt}
\resizebox{\linewidth}{!}{
\begin{tabular}{lcccl}
\toprule
\textbf{Method} & \textbf{Task Data/} & \textbf{Supports } &
\textbf{Learns } & \textbf{Core Idea} \\
\textbf{} & \textbf{Loss Free} & \textbf{Multiple Tasks} &
\textbf{across Tasks} &  \\
\midrule
Bayesian Prompt Learning & {\color{red}\ding{55}} & {\color{red}\ding{55}} & {\color{red}\ding{55}} & Bayesian uncertainty modeling 
\\
(BPL)&&&&over task-specific prompts. \\
\hline
Prompt Learning via & {\color{red}\ding{55}} & {\color{green}\checkmark} & {\color{green}\checkmark} & Meta-regularization to improve prompt 
    \\
 Meta-Regularization (ProMetaR)&&&&generalization across tasks; requires training data. \\
 \hline
Gradient-Regulated Meta-Prompt & {\color{red}\ding{55}} & {\color{green}\checkmark} & {\color{red}\ding{55}} & Meta-learned prompt init + gradient 
\\
(GRAM)&&&&regulator for few-shot cross-domain generalization. \\
\hline
PRewrite (Prompt Rewriting w/ RL) & {\color{red}\ding{55}} & {\color{green}\checkmark} & {\color{red}\ding{55}} & LLM-based prompt rewriter trained with RL  \\
&&&&to improve downstream task performance. \\
\hline
    \midrule
    \textbf{DMP (Ours)}                & {\color{green}\checkmark} & {\color{green}\checkmark} & {\color{green}\checkmark} & \textbf{Learns a prompt distribution; one-shot  } 
    \\
&&&&\textbf{guided sampling, composition, negative prompts.} \\
    \bottomrule
\end{tabular}
}
\label{tab:difftablepriorintro}
\end{minipage}
\end{figure*}

Foundation models \citep{rombach2022high, radford2021learning} generalize to diverse tasks owing to their scale, and can be customized to downstream tasks with parameter-efficient methods~\citep{hu2022lora, zhang2023adding} such as prompt learning (PL)~\citep{zhou2022coop}. In vision-language models, prompts are learned in the visual~\citep{jia2022vpt}, textual~\citep{zhou2022cocoop}, or both spaces~\citep{CoPrompt,Hao_2025_CVPR_TAC} to improve fine-grained~\citep{eurosat} and taxonomic classification~\citep{protect}, class discrimination~\citep{zhou2022coop,zhou2022cocoop}, and robustness to domain shift~\citep{ge2022domainadapt}. In generative models, prompts customize synthesis to specific concepts, people, or objects~\citep{gal2022textual,ruiz2022dreambooth}, or control the strength of attributes such as age, emotion, or style~\citep{eccvsridhar24}.

Despite their power, learned prompts have key limitations (Figure~\ref{fig:teaser_comparison}, left). They are inherently task-specific, requiring task-specific data and losses: classification prompts generalize poorly beyond the few-shot base classes and must be retuned for each label set, while generative prompts for attributes such as age or smiling must be learned separately per entity or concept. Since different applications could share the same prompts, practical systems rely on large repositories of prompts or adaptation weights, such as HuggingFace or Civitai~\citep{civitai}, where prompts must be stored, searched, and manually composed~\citep{luo2024stylus}, limiting reuse.

In this work, we propose a different perspective: instead of treating prompts as isolated optimization outcomes, they should be considered as samples from a structured distribution. Across tasks and domains, learned prompts exhibit regularities that admit smooth interpolation, and support meaningful composition and negation. Capturing this structure would enable the reuse of prior prompt learning effort and improve generalization to new tasks. Rather than learning one set of prompts at a time, there is a need for meta-prompting techniques capable of learning to generate prompts across many tasks and models. However, to be scalable, the learning should be feasible \textit{without revisiting the original task data or losses}. Instead, one would like to rely on the large prompt repositories that are already available, whose prompts were themselves obtained by standard prompt optimization on their respective source tasks. This rules out classical meta-learning techniques~\citep{finn2017model,
nichol2018first,
li2017metasgd,
andrychowicz2016learning}, based on a double optimization loop, where a prompt generator is trained jointly with the prompted models across several tasks. 

To accomplish this goal we propose an approach, \emph{Diffusion Meta-Prompting (DMP)}, where a generative model is learned over prompt embeddings, using diffusion models. A DMP is trained on a repository of previously learned prompts,  and learns to synthesize new prompts conditioned on natural language task descriptions. As illustrated on the right of Figure~\ref{fig:teaser_comparison}, at inference, prompts are simply sampled with the DMP model, enabling one-shot adaptation of foundation models to unseen classes, composite tasks, and novel concept combinations. Table~\ref{tab:difftablepriorintro} compares DMP against prior meta-learning approaches, showing its uniqueness in terms of not requiring access to the original task data or losses once a prompt repository is available, while learning joint prompt structure over tasks, and requiring a minimal amount of prompt storage. The combination of these properties makes DMP much more scalable in the number of tasks and capable of prompt interpolation  and generalization.

Although the ability of diffusion models to synthesize prompts is not surprising, different diffusion trajectories can produce prompts with different downstream generalization behavior. To mitigate these effects, we introduce a simple \emph{test-time steering} mechanism, which allows DMP to reuse the prompt repository at inference time without any additional training. For each dataset, we select the prompt seed with the highest training accuracy, and use this as an anchor during DDIM sampling. At each denoising step, the predicted clean prompt is interpolated with this anchor. This biases the sampling trajectory toward a known low-risk region of the learned prompt manifold while still allowing the diffusion model to synthesize new prompts conditioned on the task specification. Importantly, the anchor is selected once, using the source-task training performance already associated with the repository, and does not use target-task examples or novel test classes. This allows task generalization.
As shown in the rightmost panel of
Figure~\ref{fig:teaser_comparison}, DMPs eliminate the complexities of dealing with prompt repositories, offering better generalization and compositionality with a minimal runtime overhead. For example, given a subject name (``Jennifer Aniston'') and attributes (``hair'' and ``age''), a DMP can generate a set of prompts to personalize a foundation diffusion model like Stable Diffusion XL (SDXL) to the generation of images with this combination of concepts.

We choose diffusion as the generative model for prompts since it offers better generalization than alternatives like autoregressive methods~\citep{radford2019language} (shown later in Table ~\ref{tab:supp_copromptcomparison}, Appendix ~\ref{app:ablation-autoregressive}). This is demonstrated in two ways. First, we show that {\it DMPs learn to produce a range of sophisticated prompt operations\/}, such as concept composition, guided sampling, novel subject generation,
editing and negative prompting
{\it without explicit training\/}. Second, we show that {\it DMPs can learn multiple tasks,} by showing that a single model can be trained to produce prompts for both subject personalization and slider attributes. This replaces the combinatorial complexities of searching separate prompt repositories for the two (and potentially more) tasks into a single DMP model. DMPs are also shown applicable to a diversity of fundamentally different tasks, ranging from image synthesis to composite classification. For the latter, we show that learning the distribution of prompts across datasets and label sets generalizes to unseen classes better than prior meta-learning approaches and the standard approach of individual prompt learning per dataset. The fact that this happens {\it even though the DMP model is trained on the individual prompts of the baseline approach} shows that there is structure in prompt space, which DMPs learn for improved generalization. 

Beyond prompt embeddings, the same principle can be applied to foundation-model
representations themselves. We therefore evaluate DMP in a sparse multi-view retrieval
setting, where a query view of an object must retrieve the remaining views of the same
object from a gallery in which many views are missing. In this setting, DMP is trained on objects
disjoint from the gallery and synthesizes embeddings for missing views. Retrieval is then
performed with an aggregated gallery containing both observed and synthesized
embeddings. We demonstrate that DMP can also model such structured distributions
beyond learned prompts, and can improve retrieval for a different representation space
and downstream objective without accessing the gallery objects during training (Table~\ref{tab:mvimgnet_retrieval}). 

Overall, the paper makes the following key contributions

\begin{itemize}
    \vspace{-.1in} 
   
    \item We introduce \emph{Diffusion Meta-Prompting}, a meta-learning framework that treats learned prompts as samples from a shared distribution, which is learned with diffusion models. 
    Unlike prompt learning, retrieval, or refinement methods, DMP is trained once on a repository of previously learned prompts and synthesizes new task-specific prompts at inference \textbf{without revisiting the original task data or task losses, and without additional prompt optimization.}
    \vspace{-.05in} 
    \item We introduce a test-time steering mechanism at inference  without any additional training that produces prompts with better generalization (upto \textbf{2.0\%}) over standard  sampling.
    \vspace{-.05in} 
    \item We provide a theoretical guarantee bounding the expected task loss of prompts sampled from a DMP in terms of the empirical loss of the prompt repository and the denoising error of the diffusion model.
    \vspace{-.05in} 
    \item Through extensive experiments, we show that DMP improves generalization in the most challenging regimes where prompt learning typically fails, including cross-task, and composite classification settings (\textbf{up to 2.0\% accuracy gains}), while also \textbf{reducing storage and inference costs by over 90\%} compared to prompt retrieval approaches.
    \vspace{-.05in} 
    \item We also demonstrate that DMP is representation-agnostic by applying it to sparse
multi-view retrieval experiment which
improves retrieval by \textbf{+3.44 mAP} over direct DINOv2
retrieval.
    \vspace{-.1in} 
\end{itemize}

\vspace{-2mm}
\section{Related work}

\textbf{Foundation models.} 
We focus on two classes of foundation models: vision-language representation models~\citep{radford2021learning} and text-to-image (T2I) generation models~\citep{rombach2022high}. Contrastively trained vision-language models, such as CLIP, are widely used for open-set classification, detection, and segmentation, with prompting techniques enabling adaptation to fine-grained tasks~\citep{zhou2022coop}. In T2I models, personalization methods like Textual Inversion~\citep{gal2022textual} allow generation of images of custom concepts using only a few examples. Prompt Sliders~\citep{eccvsridhar24} and~\citep{baumann2024attributecontrol} learn concepts or attributes in textual space, either globally or locally. In this work, we propose to unify the prompt learning for foundation models by training a diffusion model to synthesize the prompts conditioned on natural language across tasks.

 \noindent{\bf Prompt learning.}
Prompt tuning learns textual~\citep{zhou2022cocoop}, visual~\citep{jia2022vpt}, or multi-modal~\citep{Khattak_2023_ICCV_promptsrc, yang2024mma, li2025atprompt, li2025dpc} prompts for CLIP~\citep{radford2021learning}. Textual prompt learning, pioneered by CoOp~\citep{zhou2022coop} and CoCoOp~\citep{zhou2022cocoop}, fine-tunes a CLIP model for few-shot transfer by optimizing continuous prompt vectors in the language branch. Visual prompt tuning~\citep{jia2022vpt} introduces task-specific learnable prompts in the visual encoder while keeping the backbone fixed. Bayesian prompt learning~\citep{derakhshani2023variationalbayesian} formulated prompt learning as a variational inference problem and demonstrated its ability to generalize to unseen classes at the expense of base class accuracy. Multi-modal prompt learning~\citep{CoPrompt,Hao_2025_CVPR_TAC} optimizes prompts in both vision and language encoders to improve cross-modal alignment. Table~\ref{tab:difftablepriorintro} and ~\ref{tab:difftableprior} in Appendix shows how DMP differs from prior prompt-learning approaches across key axes. Unlike prior approaches, which refine prompts for a single task or single example, DMP uniquely enables \textit{multi-task prompt learning and generation, without access to task specific data or model training, simply requiring access to a repository of pre-computed prompts}.

\noindent\textbf{Meta learning.}  
Meta-
learning methods aim to amortize learning across tasks,
enabling rapid adaptation to new problems~\citep{ha2017hypernetworks,hospedales2021metalearning}. These methods typically have a double-loop optimization, e,g, iterating between leaning of the prompt generator and of the downstream model. Unlike DMP, this is usually not scalable to large numbers of tasks. Recently, diffusion
models have been explored as generative priors
for weights and representations~\citep{zhang2024metadiffmetalearningconditionaldiffusion,du2024promptdiffusionrobustifiesanymodality}. ~\citep{du2024promptdiffusionrobustifiesanymodality} proposed a diffusion model for \textit{refining} CLIP prompts for classification. It is trained on a \textit{specific example and classification task} to improve CLIP prompts for that task.  Unlike
this type of refinement approach, DMP is a generalist method, trained to sample prompts across tasks and downstream functionalities.

\section{Diffusion Meta-Prompting}\label{sec:dmp}
We introduce the \emph{Diffusion Meta-Prompting} (DMP) framework for synthesizing prompts conditioned on task descriptions.
A community of users first learn a set of prompts $x(c)$ for each task $c$ in a set $\mathcal{C}$, using standard prompt learning techniques.  The prompts $x(c)$ may cover multiple tasks, e.g. SDXL image synthesis or CLIP-based image classification. The learned prompts $x(c) \in\mathbb{R}^d$ are stored in a {\it prompt repository} $\mathcal{R}$, together with textual descriptions $y(c)$ of the task, e.g. 
``SDXL prompt for tattoo synthesis".
The repository prompts 
$x(c)$, are then used to train the DMP model. 
This is a conditional prompt generator $p_\theta(x \mid y)$, parameterized by a diffusion model. Since the repository prompts $x(c)$ are already learned with a downstream task loss $\mathcal{L}_{\text{task}}(S;c)$, a DMP trained with the standard MSE loss suffices to synthesize prompts of low downstream risk for task $c$ when conditioned on $y(c)$, as shown theoretically in Proposition  \ref{thm:combined_main}.

\noindent\textbf{Training.}
DMPs are diffusion models~\citep{nethermo-sohl-dickstein15,ddpm_neurips_20} that synthesize prompts by iteratively denoising  a noise seed. DMP training is based on a pair of forward and backward Markov chains. Given a task or concept $c$, an associated prompt $x(c)$ is retrieved from  $\cal R$ and noised according to  a forward process that progressively adds noise to $x_0 = x(c)$, according to  
\begin{equation}
    x_t = \sqrt{{\alpha}_t} x_0 + \sqrt{1 - {\alpha}_t} \, \epsilon_t, \label{eq:forward}
\end{equation}
where $t$ is a timestep, \( {\alpha}_t := \prod_{s=1}^t (1 - \beta_s) \), \( \{ \beta_t \}_{t=1}^T \) is a variance schedule, $\epsilon_t \sim N(0, {\bf I})$ and $N$ is a Gaussian distribution.
In the reverse process, a neural network $\epsilon_{\theta}$ recurrently denoises $x_t$ to recover $x_0$. 
The network
$\epsilon_\theta(x_t, t)$ is a U-Net~\citep{ronneberger2015u} with self and cross-attention layers~\citep{vaswani2017attention}. The latter are conditioned by a text prompt $y(c)$ that specifies the task $c$, e.g. ``a personalization prompt for Jennifer Anniston'', in the example of Figure~\ref{fig:teaser_comparison}.  A text embedding \( \tau_{\theta}\) maps $y(c)$ into a conditioning vector \( \tau_{\theta}(y) \), where we omit the argument $c$ for brevity. The denoising network $\epsilon_{\theta}(x_t, \tau_{\theta}(y), t)$ is trained to predict noise $\epsilon_t$, by minimizing the risk  
    \begin{equation}
\mathcal{L}_{\mathrm{denoise}}(\theta) = \mathbb{E}_{t, x_0, y, \epsilon} \left[ \left\| \epsilon_t - \epsilon_{\theta}(x_t, \tau_{\theta}(y), t) \right\|^2 \right]. \label{eq:loss}
\end{equation}

In our implementation, the DMP model is trained with classifier-free guidance, where an empty text or null prompt is used 20\% of the time, to allow for better guidance during sampling. After training, a user simply specifies the task text prompt $y(c)$. The diffusion model samples a prompt $x(c)$ 
with 
\begin{equation}
x_{t-1}
=
\sqrt{\alpha_{t-1}} \hat x_0
+
\sqrt{1-\alpha_{t-1}-\sigma_t^2}\,
\epsilon_\theta(x_t, \tau_{\theta}(y), t)
+
\sigma_t \epsilon_t
\label{eq:step}
\end{equation}
where $\hat x_0 =
(x_t - \sqrt{1-\alpha_t}\,\epsilon_\theta(x_t, \tau_{\theta}(y), t))/\sqrt{\alpha_t}$ is a prediction of the denoised prompt according to (\ref{eq:forward}), $x_T \sim N(0, {\bf I})$ is a noise seed, and $x(c) = x_0$. 
From a meta-learning perspective, $\theta$ are meta-parameters amortizing the production of prompts across tasks. Unlike per-task optimization, DMP enables \emph{one-shot prompt synthesis} by sampling from the learned generator. 

\textbf{Theoretical Guarantee.} 
The following result provides a theoretical guarantee for DMP performance, by bounding the distance between the task loss of the prompts learned by traditional prompt learning and those sampled with a DMP. The proof is given in Appendix~\ref{sec:theory_analysis}. 
\begin{proposition}[DMP performance guarantee]
\label{thm:combined_main}
Let $\theta$ be a DMP such that the the empirical version of the denoising loss of (\ref{eq:loss}) over a repository of $n$ i.i.d.\ prompts 
satisfies
$\mathcal{L}_{\mathrm{denoise}}(\theta)\le\tau$ 
for condition $y$. Then with probability at least $1-\delta$,
\begin{equation}
\left|
\mathbb{E}_{x\sim p_\theta(x\mid y)}\!\left[\mathcal{L}_{\mathrm{task}}(x)\right]
-
\mathbb{E}_{x\sim \widehat p_{\mathrm{data}}(x\mid y)}\!\left[\mathcal{L}_{\mathrm{task}}(x)\right]
\right|
\le
L_{\max}\sqrt{2 C\,\tau}
+
L_{\max}\sqrt{\frac{\log(2/\delta)}{2n}}
+
o(1).
\end{equation}
where $p_{\mathrm{\theta}}(x\mid y)$ is probability distribution of the DMP model conditioned on $y$, $\hat{p}_{data}(x\mid y)$ is the empirical distribution of the data conditioned on $y$, $\mathcal L_{\mathrm{task}}(x)$ is a task loss bounded by $[0,L_{\max}]$, and $L_{max}$, $C$ are constants.
\label{eq:combined_bound_main}
\end{proposition}
Hence, given large enough $n$ and small enough $\tau$, DMP is guaranteed to meet the performance of prompt learning.

\noindent\textbf{Architecture.} To implement the DMP model, we develop a 1D variant of the popular Stable Diffusion model, where 2D operations are replaced by their 1D counterpart (e.g., 1D-conv). Since the size of the prompt embeddings ($d$) is relatively small, we have found that, for most applications, the model can operate directly in the space ${\cal P}$ of prompts  $x(c)\in \mathbb{R}^d$. However, for applications involving unusually long prompts ($d\geq2048$), we have also trained a variational autoencoder that maps prompts from $\cal P$ into a space of lower dimensionality, for faster training and convergence. This is a 1D variant of stable diffusion autoencoder with less than 1M parameters and a latent dimension of 128. See appendix section \ref{app:implementation} for more details.

\noindent\textbf{Concept Composition.}
By framing diffusion models as Energy Based Models,~\citep{liu2022compositional} showed that it is possible to compose multiple tasks by conjunction or negation. 
Given a diffusion model $\epsilon_{\theta}(x_t, t)$, $n$ tasks $c_i$ are combined by implementing the denoising chain with 
\begin{equation}
    \hat{\epsilon}(x_t, t) = \epsilon_{\theta}(x_t, t) + \sum_{i=1}^{n} \eta_i \left(\epsilon_{\theta}(x_t, \tau_{\theta}(y(c_i)), t) - \epsilon_{\theta}(x_t, t) \right)
    \label{eq:7}
\end{equation}
where $\eta_i$ is a hyperparameter corresponding to a temperature scaling of task $c_i$. 
If the conditioning is just empty text, this reduces to classifier-free guidance.  The standard implementation of task composition is to run the {\it downstream} diffusion model $n$ times (once per task) and average noise predictions with (\ref{eq:7})~\citep{liu2022compositional}. This is significantly more complex than performing the task composition of (\ref{eq:7}) in the (much less complex) DMP, which enables the sampling of {\it single\/} prompt $x_0$ for the downstream model that combine {\it all} tasks $c_1, \ldots, c_n$.

\noindent\textbf{Test-Time Steering.}
While DMP can synthesize prompts in one shot, different samples from
$p_{\theta}(x \mid y(i))$ may lie in prompt regions with different downstream
generalization behavior. To address this, we introduce a simple test-time steering
mechanism that biases diffusion sampling toward a high-performing prompt already
available in the repository, without retraining the DMP model.

For each task dataset $i$, let $\{x_{i,m}\}_{m=1}^{M}$ denote the prompts obtained
from $M$ different prompt-learning initializations. We evaluate prompt $x$, by measuring the accuracy $\mathrm{Acc}_{\mathrm{train}}(x)$ of the downstream model prompted with $x$  on the training split, and select the best
\begin{equation}
    x_i^\star = \arg\max_m \; \mathrm{Acc}_{\mathrm{train}}(x_{i,m}).
    \label{eq:xistar}
\end{equation}
This prompt is then used as an anchor for DMP sampling by replacing, in eq. (\ref{eq:step}),  $\hat x_0$ with 
an estimate of the denoised prompt steered  by $x_i^\star$
\begin{equation}
    \tilde x_0(x_t, y(i), t)
    =
    (1-\lambda_t)\hat x_0(x_t, y(i), t)
    +
    \lambda_t x_i^\star ,
    \label{eq:best_latent_steering}
\end{equation}
where $\lambda_t \in [0,1]$ controls the steering strength. In practice, we use a
constant value $\lambda_t=\lambda$ for all denoising steps. The DDIM
transition is then computed with $\tilde x_0$ in place of $\hat x_0$ in (\ref{eq:step}) with $\sigma_t=0$. 
This procedure can be interpreted as imposing a test-time prior over the prompt manifold.
The diffusion model still conditions on the dataset text $y(i)$ and samples from
a random noise seed, but the synthesized prompt is gently pulled toward a
known low-risk region of the repository. Note that the generated prompt is not a copy
of $x_i^\star$. Instead, $x_i^\star$ acts as an anchor that stabilizes sampling
around a prompt basin of empirically strong performance on the training data.

For composite tasks, e.g. classification problems where the conditioning string contains class names from
$K$ datasets of anchors $\{x_{i_j}^\star\}_{j=1}^{K}$, we use the average anchor 
$
        x_{\mathrm{comp}}^\star =
    \frac{1}{K}\sum_{j=1}^{K} x_{i_j}^\star,
    \label{eq:avganchor}
    $
as $x_i^\star$ in Eq.~\ref{eq:best_latent_steering}.
Importantly, the prompt anchor is selected only from training-set performance and
does not require access to novel test data. The method is therefore a
lightweight test-time scaling strategy: it reuses the repository to identify a
reliable anchor and improves the stability of DMP sampling while
preserving the ability to synthesize new prompts conditioned on task specifications.

\vspace{-4mm}

\section{Experiments}
\vspace{-2mm}

\textbf{Implementation Details.}\label{sec:implementation}
We conduct all experiments on 24GB (NVIDIA-A10 or 3090-RTX) GPUs using pytorch. DMP models are trained with the standard hyperparameters of~\citep{rombach2022high}, a learning rate of $1e^{-6}$, and batch size of $320$, for $2,000$ epochs. This takes about a day to train on 4 GPUs for 3,000 SDXL personalization prompts. For classification prompts, the DMP autoencoder and diffusion is trained for $10,000$ epochs, which takes about 10 hours\footnote{All datasets used in the paper were solely downloaded and evaluated by the university}. DMP uses 50 DDIM timesteps to sample one prompt which takes about  1 sec and it can be sped up with faster sampling methods~\citep{liu2022pseudo}.  For fair comparisons, we use the codebase of respective prompt learning methods and run all our experiments under this setup. For test-time steering, we use $\lambda=0.2$. More details on the architecture and models are in the appendix \ref{app:implementation}.

\vspace{-4mm}
\subsection{DMP for classification}\label{sec:dmpgeneralization}\label{sec:coop-coprompt}
\vspace{-2mm}
We trained a DMP model, DMPClass, for prompting a CLIP-style classifier. The goal is not to beat the SOTA in prompt based classifiers, but to show that DMP can synthesize prompts that generalize to multiple tasks, outperforming other meta-prompting approaches. We utilize CoOp prompts to build the prompt repository with imagenet and the 10 fine-grained datasets commonly used in the prompt learning literature~\citep{zhou2022coop,zhou2022cocoop}.  We note that DMP is agnostic to the prompt repository and can be applied to other prompting methods. However, all methods other than CoOp use additional weights, projection matrices, or architectural enhancements that involve high dimensional parameter matrices. Modeling these parameters  requires much larger diffusion models and compute than we used, which is sufficient to learn prompts. More importantly, most of the meta-prompting baselines that we compare to simply do not scale to the generation of such parameters. Appendix section~\ref{sec:appdmpclass} discusses the application of DMP to other prompting methods, where it is only used to sample prompts, maintaining the additional parameters from the original methods. The appendix shows that DMP still achieves consistent gains over the original prompts, although these are smaller than for CoOp, which has no additional parameters.

\noindent\textbf{Prompt Repository.} Given a classification problem, a prompt set is trained for $N$ epochs (where $N$ follows the original paper settings), using a context vector of size $K$. The CoOp paper reports full results only for context length 16. We use the context length of 4 ($K=4$) suggested in CoCoOp~\citep{zhou2022cocoop}.  All prompt learning methods are trained under the Base2New setting, which assesses generalization to unseen (New) classes by only using half of dataset (Base) classes for training. DMP is trained with the prompts originally learned on the Base set. We evaluate the DMP-sampled prompts on training (Base), remaining (New), and all  (All) classes. We use $M = 40$ different initializations per dataset and save the corresponding prompt embeddings to produce a training repository $\mathcal{R} \in \mathbb{R}^{11 \times 40 \times K \times d}$. The best performing prompt of (\ref{eq:xistar}) ($x_{\mathrm{comp}}^\star$ for composite tasks) is used as anchor for DMP sampling and as the baseline CoOp prompt, against which DMP prompts are compared. 

\noindent\textbf{Training.} Due to the high dimensionality of the prompt embeddings ($K\times d=2048$), DMPClass models are trained with an autoencoder. To obtain the text inputs, we concatenate all the class labels $c_k, k \in \{1,\ldots C_i\},$ of  dataset $i$ into a text string $y(i) = \{c_1, \ldots, c_{C_i}\}$ whose text embedding $\tau(y(i))$ is used to condition the DMPClass model. This setup allows us to flexibly mix and match class names across datasets to enable straightforward compositionality (Table~\ref{table:combinedclasscompareBPL}) without requiring handcrafted or semantically enriched descriptions. Once the model is trained, it suffices to sample conditioned on a string of mixed class labels from any of the datasets. Since the CLIP text encoder has a limit of 77 tokens (around 50 words), these strings can be too long for many classes (e.g. the 1,000 Imagenet classes). To overcome this, we consider each class name independently and obtain the CLIP embedding for each class, resulting in $C$ embeddings for $C$ classes, which are concatenated into a vector that is used to prompt the DMPClass model. 
See~\ref{app:ablationtextinputs} for ablation on text inputs.

\begin{figure*}[t]
\centering
\scriptsize

\begin{minipage}{0.22\columnwidth}
\centering
\scriptsize
\captionof{table}{Zero-shot accuracy comparison on combined dataset classes. }
\label{table:combined11classes}
\setlength{\tabcolsep}{2pt}
\begin{tabular}{lccccc}
    \toprule
    \textbf{Method} & \textbf{All} &\textbf{Base} & \textbf{New} \\
    \midrule
CoOp  & 69.6& 73.6 & 76.7   \\ 
BPL  & 72.0 & \textbf{77.5} & 78.4  \\ 
ProMetaR   & 71.2 & 75.8 & 77.6 \\
\textbf{DMP } & \textbf{73.6} & 76.0 & \textbf{78.8} \\
\midrule
\textbf{$\Delta_{\mathrm{CoOp}}$} & \textcolor{ForestGreen}{\textbf{+4.0}} & \textcolor{ForestGreen}{\textbf{+2.4}} & \textcolor{ForestGreen}{\textbf{+2.1}} \\
    \bottomrule
\end{tabular}

\end{minipage}
\hfill
\begin{minipage}{0.24\columnwidth}
\centering
\scriptsize
\captionof{table}{Accuracy comparison \\on combined EuroSAT and Flowers dataset pair. }
\label{table:combinedclasscompareBPL}
\setlength{\tabcolsep}{1.5pt}
\begin{tabular}{lccccc}
    \toprule
    \textbf{Method} & \textbf{All} &\textbf{Base} & \textbf{New} \\
    \midrule
CoOp (ZS)& 35.5& 50.2 & 59.4   \\ 
BPL (ZS) & 52.6& 61.3 &  73.4  \\ 
ProMetaR (ZS)  & 48.0 & 59.9 & 67.1 \\
\textbf{DMP (ZS)} & \textbf{61.1} & \textbf{74.5} & \textbf{77.8} \\
\midrule
BPL (Trained)& 56.2& 76.1 & 75.1   \\
    \bottomrule
\end{tabular}

\end{minipage}
\hfill
\begin{minipage}{0.52\columnwidth}
\centering
\setlength{\tabcolsep}{1pt}
\captionof{table}{Ablation of test-time steering on 55 composite dataset-pairs classification.  
$\Delta_{\mathrm{steer}}$, $\Delta_{\mathrm{CoOp}}$ denotes the absolute accuracy gain over respective baselines.}
\begin{tabular}{lccccccc}
\toprule
Split & CoOp & BPL & ProMetaR & DMP & DMP & $\Delta_{\mathrm{steer}}$ & $\Delta_{\mathrm{CoOp}}$ \\
 &  &  &  & w/o Steering & w/ Steering &  & \\
\midrule
All  & 58.7 & 60.8 & 61.2 & 61.4 & {\bf 63.1} & \textcolor{ForestGreen}{\textbf{+1.7}} & \textcolor{ForestGreen}{\textbf{+4.4}} \\
Base & 65.1 & 66.3 & 67.4 & 67.3 & {\bf 69.4} & \textcolor{ForestGreen}{\textbf{+2.1}} & \textcolor{ForestGreen}{\textbf{+4.3}} \\
New  & 65.1 & 68.7 & 68.9 & 68.6 & {\bf 70.8} & \textcolor{ForestGreen}{\textbf{+2.2}} & \textcolor{ForestGreen}{\textbf{+5.7}} \\
\midrule
Avg. & 63.0 & 65.3 & 65.8 & 65.8 & {\bf 67.8} & \textcolor{ForestGreen}{\textbf{+2.0}} & \textcolor{ForestGreen}{\textbf{+4.8}} \\
\bottomrule
\end{tabular}
\label{tab:ablation_test_time_steering}
\label{ref:coopdatasetpairs1}
\end{minipage}

\end{figure*}

\begin{figure}[t]
\centering
\begin{minipage}{0.22\columnwidth}
\centering
\scriptsize
\captionof{table}{Memory and complexity requirements for DMP against retrieval.}
\label{tab:effcomparison}
\setlength{\tabcolsep}{1pt}
\begin{tabular}{ l| l | l } 
\toprule
 Method & Mem & Time \\ 
 & (GB) $\downarrow$ & (s) $\downarrow$ \\
\midrule
Stylus &   1.30  & 12.1 \\ 
DMP & {\textbf{0.12}}  & {\textbf{1}}  \\ 
$\Delta$& (\textcolor{ForestGreen}{\textbf{91\%}$\downarrow$}) & (\textcolor{ForestGreen}{\textbf{92\%}$\downarrow$}) \\
\bottomrule
\end{tabular}

\end{minipage}
\hfill
    \begin{minipage}{0.26\columnwidth}
        \scriptsize
        \centering
        \captionof{table}{Comparison of DMPVariation with Textual Inversion (TI) for variation synthesis.}
        \label{tab:dmpvariation_faceid}
        \setlength{\tabcolsep}{1pt}
        \begin{tabular}{llll}
            \toprule
            \textbf{Method } & \textbf{Face-ID} & \textbf{CLIP-T} & \textbf{CSD} \\
            (SD-RV)& $\downarrow$ & $\uparrow$ &$\uparrow$\\
            \midrule
            TI & 0.428 & 27.0& 50.4 \\
            DMP & {\textbf{0.299}} & {\textbf{28.4}}& \textbf{61.8} \\ 
            \textbf{$\Delta$} 
& \textcolor{ForestGreen}{\textbf{-0.13}} 
& \textcolor{ForestGreen}{\textbf{+1.4}}
& \textcolor{ForestGreen}{\textbf{+11.4}}\\
            \bottomrule
        \end{tabular}
    \end{minipage}
\hfill
\begin{minipage}{0.48\linewidth}
\centering
\scriptsize
\captionof{table}{Sparse multi-view retrieval on MVImgNet. The gallery contains 5,447 objects,
with only one to four available views per object. DMP is trained on objects disjoint from
the gallery and synthesizes embeddings for missing views.}
\label{tab:mvimgnet_retrieval}
\begin{tabular}{lcccc}
\toprule
Method & NDCG & MRR & Top-1 & mAP \\
\midrule
DINOv2 & 91.64 & 94.36 & 91.44 & 86.20 \\
DMP        & \textbf{93.47} & \textbf{94.42} & \textbf{91.74} & \textbf{89.64} \\
$\Delta$          & \textcolor{ForestGreen}{\textbf{+1.83}} & \textcolor{ForestGreen}{\textbf{+0.06}} & \textcolor{ForestGreen}{\textbf{+0.30}} & \textcolor{ForestGreen}{\textbf{+3.44}} \\
\bottomrule
\end{tabular}
\end{minipage}

\end{figure}

{\textbf{Composite classification.} 
To evaluate the  generalization ability of DMP, we introduce a composite classification task. We consider two settings under this task: (1)  classify over the set of all classes of all datasets and (2) classify over pairs of class label sets from the fine-grained datasets we considered.  
For setting (2), we
 created 55 datasets ${\cal T}_i, i \in \{1, \ldots, 55\}$ containing pairs of All, Base and New classes from the original 11 datasets.  These settings test out-of-domain generalization, since prompts trained on one dataset are used to classify a mix of their classes and classes from  other datasets.  
 For example, while Oxford Flowers classes require modeling flower attributes, the composition with FGVC Aircrafts requires the modeling of both flower {\it and} airplane attributes, on which the original prompts were never trained {\it simultaneously}. For each dataset ${\cal T}_i$, we performed classification with  
(1) the prompt sampled from DMP (trained on  the repository of individual dataset prompts), and
(2) the average of two anchor prompts $x^\star$ learned  from the two datasets that compose ${\cal T}_i$. 

For the first setting, we compare DMP against the baseline CoOp prompts as well as prior meta-prompt learning methods such as BPL~\citep{derakhshani2023variationalbayesian} and ProMetaR~\citep{Prometar_Park_2024_CVPR}. Table~\ref{table:combined11classes} shows the results under All, Base and New Classes. For evaluation over all classes, DMP improves over CoOp by 4\% and outperforms the best meta-prompting baseline (BPL) by 1.6\%. Table~\ref{table:combinedclasscompareBPL} shows, that, in setting (2), the DMP gains can be more even pronounced for specific dataset class pairs, such as  EuroSAT and Flowers. For this pair, DMP exceeds BPL’s zero-shot baseline by \textbf{+8.5/+4.4\%} and  CoOp by \textbf{+25.6/18.4\%}for all/new classes, respectively.  We further compare DMP with BPL (Trained),  which is trained on the entire dataset pair. As shown in Table~\ref{table:combinedclasscompareBPL}, DMP without any task-specific training even outperforms this approach by \textbf{+2.7\%} on new classes. We also note that BPL is computationally expensive, requiring approximately 10 hours on a 40GB GPU even for a single dataset pair, which limits its scalability to composite tasks as compared to DMP. 
 
Table~\ref{ref:coopdatasetpairs1} (left) compares  the performance of the DMP prompts to others, showing  
 the average accuracy across all 55 dataset combination pairs. DMP obtains gains of \textbf{+2.0\%} for Base,\textbf{+1.9\%} for New, and \textbf{+1.9\%} for All classes over ProMetaR, the best baseline for the pairwise setting. DMP further outperforms  CoOp by \textbf{+4.3\%} for Base,\textbf{+5.7\%} for New, and \textbf{+4.4\%} for All classes.

\textbf{Effect of test-time steering.}
Table~\ref{tab:ablation_test_time_steering} also isolates the contribution of the test-time steering mechanism of (\ref{eq:best_latent_steering}). 
Compared with DMP sampling by the same model without steering, test-time steering consistently improves accuracy across all evaluation splits. 
The gain is {\bf +1.7} points on the All split, {\bf +2.1} on the Base split, and {\bf +2.2} on the New split, yielding an average improvement of {\bf +2.0} points. 
This suggests that steering the DDIM sampling trajectory toward a training-selected high-performing prompt is a useful test-time regularization towards high-performing regions of the prompt manifold.  Importantly, the steering target is selected using only training-split performance and does not require access to base or novel test classes, so the improvement is obtained without additional prompt training or test-label supervision. Ablation on the steering strength is shown in Appendix~\ref{app:steeringstrength}

\textbf{Taxonomic.} We next consider the more  challenging cross-task generalization setting of taxonomic classification~\citep{protect}. This tests the ability of the downstream models to classify images with respect to different class subsets in a class hierarchy, using the metric of Mean Treecut Accuracy (MTA) over 25 treecuts. Since the training classes are the leaves of taxonomy, hierarchical classification requires generalization from finer to coarser classes. For example, while `cats' and `dolphins' are two very different classes, they both belong to `mammals'. \citep{protect} showed that standard prompt learning methods perform poorly on this task, where effective prompt learning requires hierarchical training and sampling of label sets across the entire taxonomy. 

Figure~\ref{fig:mtacomparison} shows the MTA performance of CoOp, BPL, and DMP prompts. Since only ImageNet and SUN provide labels with respect to the entire class taxonomy, we used these two datasets, plus the OOD variants of ImageNet on this experiment.
Table ~\ref{tab:mta_hca} presents the full results of the experiment.
DMP prompts obtain an average gain between  \textbf{2.4} and \textbf{9.4\%} over BPL and a larger gain between \textbf{5.8} and \textbf{11.7\%} over CoOp, even though {\it no prompts are ever trained for hierarchical classification}. These results are consistent with the previous findings and show that DMP excels as the generalization challenge increases, learning to produce prompts that are much more robust than those of the CoOp prompt learning method in which it was trained.

\begin{figure*}[t]
\centering
\begin{minipage}[t]{0.4\linewidth}
\centering
\includegraphics[width=\columnwidth]{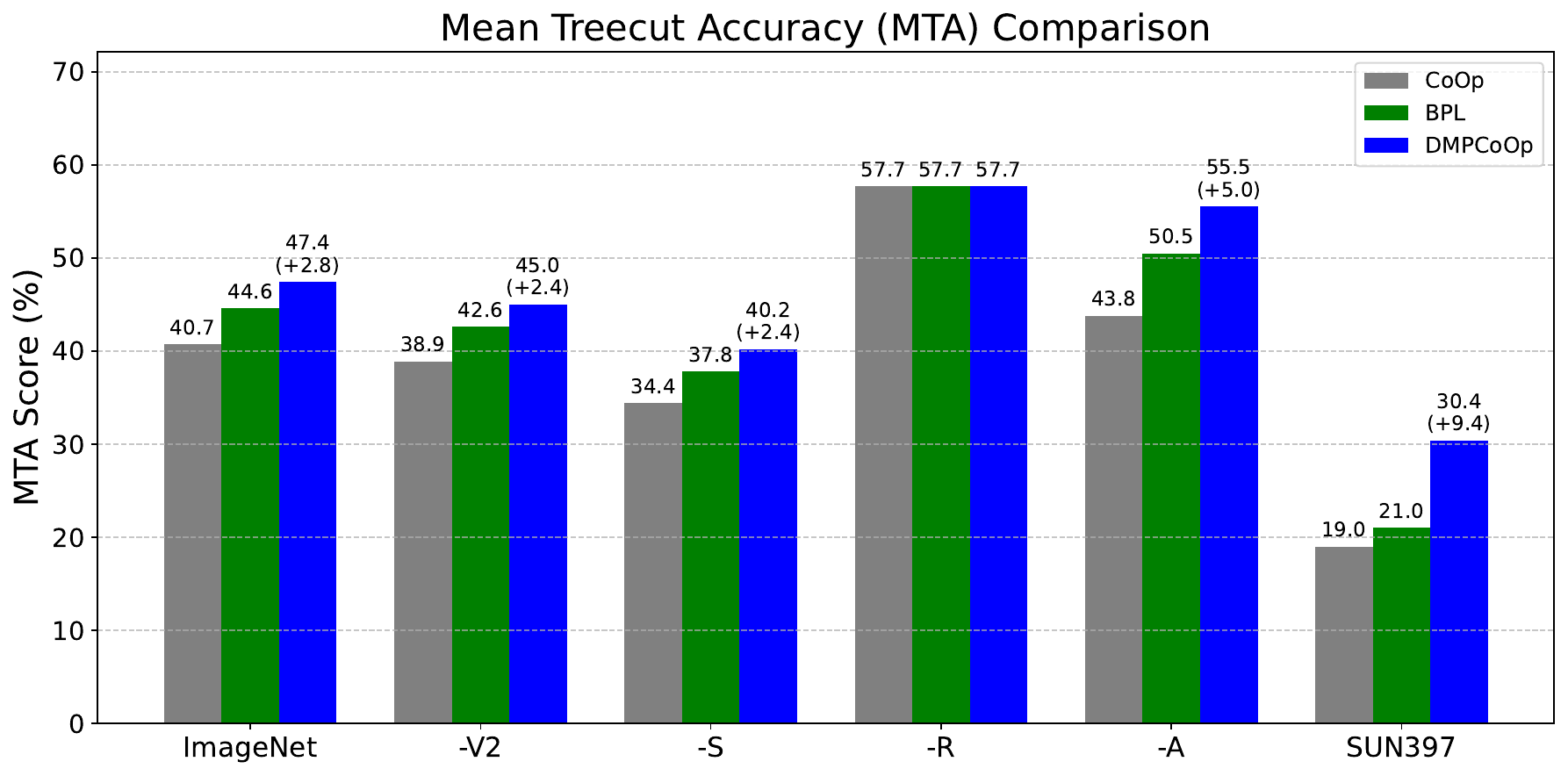}
\caption{\label{fig:mtacomparison} \textbf{Taxonomic classification generalization}: Comparison of DMP against CoOp and BPL prompts for hierarchical classification. DMP prompts generalize better than BPL prompts with $\approx$ \textbf{2-9\%} accuracy gains.}
\end{minipage}
\begin{minipage}[t]{0.58\textwidth}
    \centering
\includegraphics[width=\columnwidth, trim= 15 100 20 80, clip]{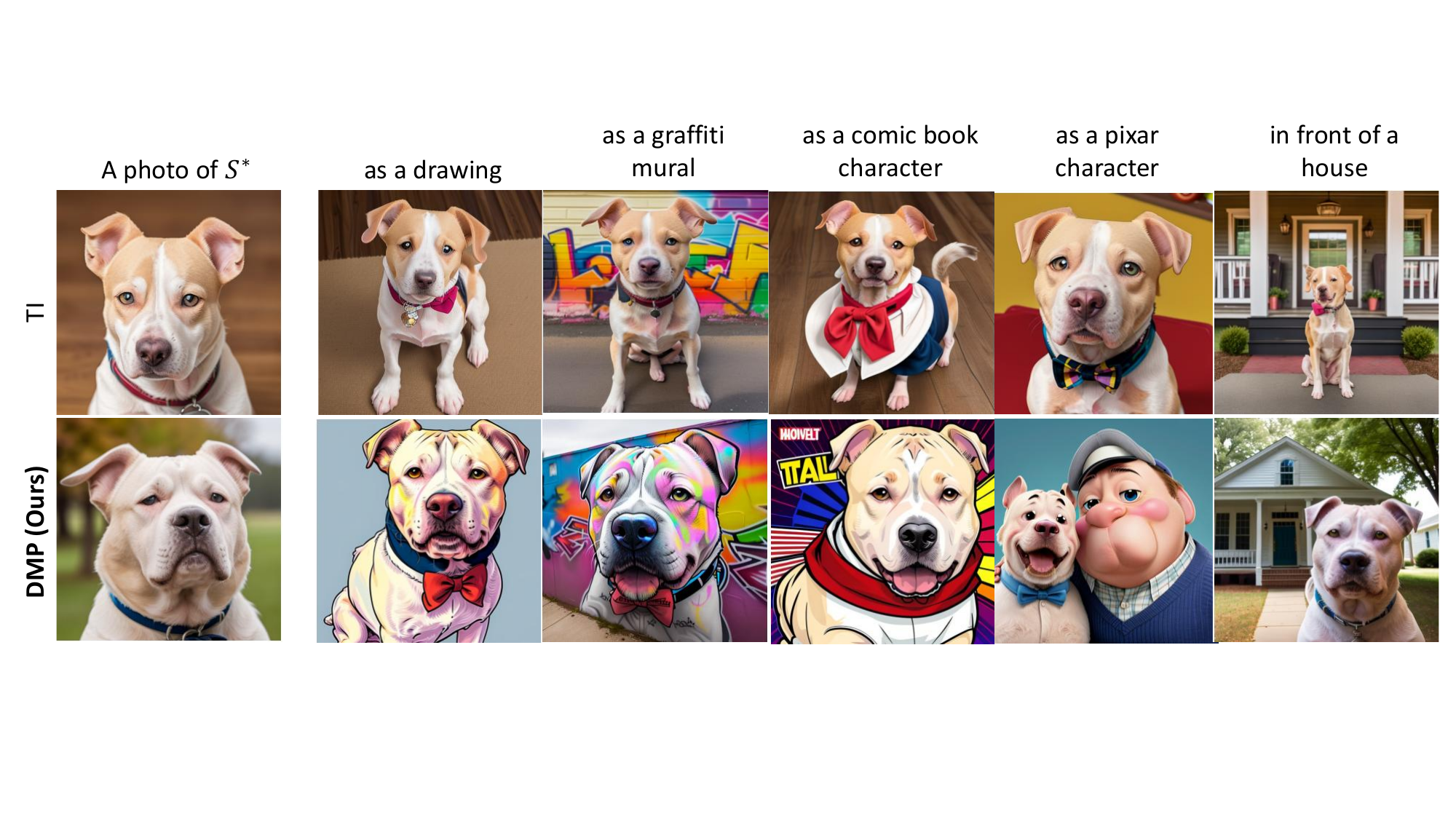}
\caption{\label{fig:comparison_mdp_ti_images} \textbf{Generalization of DMPVariation model}. Images generated using prompts synthesized by TI (top) and DMPVariation (bottom) for various downstream model text prompts (shown on top). DMPVariation prompts are more robust than TI prompts, which impair the ability of the downstream model to generalize. See Fig.~\ref{fig:supp_comparison_mdp_ti_images} for additional results.}
\end{minipage}
\end{figure*}

\vspace{-4mm}

\subsection{DMPVariation}
\vspace{-2mm}

Prompt learning is an approach to personalize diffusion models to the synthesis of images of specific concepts.
DMPVariation is a DMP model that synthesizes personalization prompts learned via textual inversion (TI)~\citep{gal2022textual}, to prompt a diffusion model to synthesize variations of a subject. 

\noindent\textbf{Prompt Repository.}
To produce a prompt repository $\cal R$, we split the CelebA dataset into unique identities using the groundtruth labels, and use 3,000 of these as concepts $c$, for which we train personalized prompts $x(c)$ with Textual Inversion using 1,000 steps of gradient descent. To obtain prompts for variations, we note that earlier optimization steps do not fully encode face attributes, as compared to steps later in the optimization. We save the prompts $x(c)$ from 40 gradient steps (between $200-400$ steps in intervals of $5$), per identity. This produces a repository $\cal R$ of 40 prompt variations per identity to obtain a training repository $\mathcal{R} \in \mathbb{R}^{3000 \times 40 \times d}$. The text string $y(c)$ associated with task $c$ is ``a photo of <identity-c>".

\noindent\textbf{Training.} The DMPVariation model was trained directly in prompt space $\cal P$, i.e. without autoencoder. Each personalized prompt produced by TI is used as conditional input to DMPVariation, by concatenating it with the noise vector. During training, the model learns to denoise the 40 different variations of the conditional prompt input. Hence, sampling from the DMP produces a diversity of prompts that induce the downstream model to synthesize variations of the subject face (see Fig.~\ref{fig:teaser_variations} for an example). At inference, prompts that induce variations of novel subjects can be sampled with DMP by simply conditioning on the TI prompt for the new subject. 

DMPVariation demonstrates the effectiveness of DMP for generative tasks, namely the task of synthesizing images with variations of a subject or concept.
To evaluate prompt generalization in this context, we compare DMPVariation prompts to the prompts originally learned by TI on the CelebA dataset. Performance is measured with Face-ID (VGGFace2), CLIP-Score across 28 diverse prompts (listed on Table~\ref{tab:prompt_variations_eval}), and the HPSv2 metric on 100 prompt concepts downloaded from the Huggingface repository. Note that none of these concepts are on the CelebA dataset. 
Table~\ref{tab:dmpvariation_faceid} shows that DMPVariation achieves lower Face-ID similarity, indicating diverse yet semantically similar prompts, and improves CLIP-Text similarity by \textbf{1.4\%} and the CSD style metric by \textbf{+11.4\%}, demonstrating better generalization to novel prompts compared to TI embeddings.
However, as is typical for generative modeling, these metrics do not capture the substantial qualitative gap between the two approaches. While TI prompts produce images that {\it do not fully comply\/} with the instruction, DMP prompts consistently produce images that satisfy the instructions. 

Figure~\ref{fig:comparison_mdp_ti_images}  illustrates this gap by presenting images synthesized with DMPVariation and Textual Inversion (TI) prompts. The downstream model, personalized to a subject with prompts produced by TI or DMPVariation, is asked to generate an image of the subject in a new context, specified as a text prompt atop each image. The personalization prompts produced by DMPVariation demonstrate greater robustness and generalization, effectively adapting to different contexts. This is unlike TI prompts, which tend to overfit to the subject, leading to poor generalization across contexts. In result, TI prompts fail to generate suitable images of the dog for all contexts other than \textit{``in front of a house"}. The figure also shows that DMPVariation produces successful variation prompts for general objects like dogs {\it despite being trained only on CelebA-face identities\/}. Fig.~\ref{fig:supp-mdpvariations-generalization} presents images for other objects, such as bird and statue. Fig.~\ref{fig:supp_comparison_mdp_ti_images} shows additional results where TI completely fails as compared to DMPVariation. 

\textbf{DMP for sparse multi-view retrieval.}
We next evaluate whether DMP can generalize beyond prompt synthesis by applying it to
a representation-level retrieval task. We consider sparse multi-view retrieval, where the
query is one view of an object and the goal is to retrieve all other views of the same
object from a gallery. Unlike densely populated multi-view retrieval settings, we assume
that gallery coverage is incomplete: each object has only between one and four available
gallery views. This creates a missing-view problem, where direct retrieval from the
observed gallery may fail when the query view differs substantially from the available
views of the same object.

We construct this task from a subset of MVImgNet~\citep{yu2023mvimgnet} containing 5,447 gallery objects. The baseline is a  retrieval operation using DINOv2 embeddings of the
available gallery views. DMP
is trained to synthesize missing views on a set of objects disjoint from those in the gallery, to guarantee that \textit{all objects
are novel at test time}. DMP-based retrieval complements the observed DINOv2
embeddings with DMP-synthesized embeddings of the missing views of
each object. Retrieval is then performed against an aggregated gallery containing both
the observed and DMP-synthesized embeddings.

Table~\ref{tab:mvimgnet_retrieval} shows that DMP consistently improves retrieval
performance over direct DINOv2 embeddings. The largest gain is observed for mAP
($+3.44$), indicating that synthesized missing-view embeddings improve the ranking of
all relevant views, rather than only the top prediction. DMP also improves NDCG by
$+1.83$ and Top-1 accuracy by $+0.30$, while maintaining comparable MRR. These results
show that DMP learns useful structure in the space of foundation-model embeddings and
can synthesize representations that compensate for sparse gallery coverage. Importantly,
this experiment demonstrates that the benefits of DMP are not limited to learned prompt
spaces: the same meta-generative framework can provide gains for different
representations and downstream tasks.

\textbf{Storage and Runtime Efficiency.}
 One of the benefits of the DMP framework is its high efficiency. DMP eliminates the need to manage a repository of prompts and associated \textit{metadata}, such as text descriptions. This contrasts with methods like Stylus~\citep{luo2024stylus}, that automatically search, retrieve, and compose LoRAs from a repository. Table \ref{tab:effcomparison} compares the storage and processing requirements of DMP and Stylus, when used for the tasks that we consider in this work. DMP is significantly more efficient, reducing storage needs by \textbf{91\%} and improving inference speed by \textbf{92\%}.

\vspace{-4mm}
\section{Ablation Studies}
\vspace{-4mm}


\textbf{Robustness to anchor perturbation.} To assess whether the gains of
test-time steering depend on the exact anchor $x^\star_i$ of (6), we perturb
it in random directions at increasing $L_2$ distances before using it in (7).
Table~\ref{tab:noisy_best_prompt} shows that accuracy remains essentially
unchanged under moderate to large perturbations, varying by at most 0.27
points across splits. This suggests that steering does not rely on a single
privileged prompt, but instead guides sampling toward a broad
high-performing neighborhood of the prompt manifold. Since the anchor is
fixed throughout sampling, however, these results alone do not distinguish
task-specific guidance from the stabilizing effect of a fixed target.

To separate the two, we replace the training-selected anchor with a fixed,
randomly selected training prompt. Accuracy decreases from 63.1/69.4/70.8
to 62.4/69.1/70.2 on the All/Base/New splits, i.e., by 0.7/0.3/0.6 points.
Since this alternative anchor also remains fixed during sampling, the
degradation cannot be attributed to a loss of stabilization alone,
indicating that the training-selected anchor provides task-relevant guidance
beyond the text condition.

Finally, we resample the perturbation direction at every denoising step,
which removes the consistency of the steering target across the trajectory.
This leads to a larger reduction, 1.8/1.9/2.2 points
below the unperturbed anchor, showing that a consistent steering target is
also important. Overall, the anchor plays two complementary roles: its
location in prompt space guides sampling toward a task-relevant region, and
its consistency across denoising steps stabilizes the sampling trajectory.

\begin{table}[t]
\centering
\small
\caption{Robustness of test-time steering to perturbations of the anchor
prompt. Results are averaged over dataset pairs,
and over repeated noise samples within each pair. $\Delta$ is relative to the
unperturbed anchor. \emph{Fixed anchor}: a randomly selected training prompt
is used as the anchor. \emph{Resampled}: the perturbation direction is
resampled at every denoising step.}
\label{tab:noisy_best_prompt}
\begin{tabular}{@{}l r| rr rr rr@{}}
\toprule
& & \multicolumn{2}{c}{All} & \multicolumn{2}{c}{Base} & \multicolumn{2}{c}{New} \\
\cmidrule(lr){3-4} \cmidrule(lr){5-6} \cmidrule(l){7-8}
\multicolumn{2}{@{}l}{Setting} & Acc. & $\Delta$ & Acc. & $\Delta$ & Acc. & $\Delta$ \\
\midrule
$L_2$ & 0.00 & \textbf{63.10} & \dlt{\text{--}} & \textbf{69.40} & \dlt{\text{--}} & 70.80 & \dlt{\text{--}} \\
$L_2$ & 0.25 & 63.01 & \dlt{-0.09} & 69.13 & \dlt{-0.27} & 70.77 & \dlt{-0.03} \\
$L_2$ & 0.50 & 62.94 & \dlt{-0.16} & 69.13 & \dlt{-0.27} & 70.73 & \dlt{-0.07} \\
$L_2$ & 1.00 & 63.09 & \dlt{-0.01} & 69.28 & \dlt{-0.12} & \textbf{70.90} & \dlt{+0.10} \\
\midrule
\multicolumn{2}{@{}l}{Fixed anchor} & 62.36 & \dlt{-0.74} & 69.12 & \dlt{-0.28} & 70.21 & \dlt{-0.59} \\
\multicolumn{2}{@{}l}{Resampled}    & 61.27 & \dlt{-1.83} & 67.47 & \dlt{-1.93} & 68.56 & \dlt{-2.24} \\
\bottomrule
\end{tabular}
\end{table}

\textbf{Sensitivity to the class-name condition.} We next examine how DMP
responds to changes in its text condition, with and without steering.
Starting from the Caltech101 class-name condition, we progressively replace
40--80\% of the class names with distinct ImageNet class names.
Table~\ref{tab:replaced-names-steering-comparison} reports the $L_2$ distance
of the resulting condition embedding and of the sampled prompt from those of
the original condition, together with the downstream accuracy. As more names
are replaced, the condition embedding moves away from the
original, with its $L_2$ distance increasing from 17.4 to 34.0.

Without steering, these changes propagate to the sampled prompts, whose
$L_2$ distance from the original prompt reaches 1.42--1.62, and replacing
40\% of the names changes ImageNet accuracy by 0.3 points. Hence, neither
the text encoder nor DMP collapses related class vocabularies into the same
representation or prompt. With steering, the prompt distances decrease to
0.025--0.049, and accuracy remains stable at 93.0--93.2\% on Caltech101 and
64.5--64.9\% on ImageNet. These results separate the two effects: DMP remains
sensitive to the class-name condition, while the training-selected anchor
limits how strongly this variation changes the sampled prompt and its
downstream behavior, stabilizing generation within a locally
high-performing region of prompt space.

\begin{table}[t]
\centering
\caption{Sensitivity of DMP to the class-name condition, with and without
test-time steering. A fraction of the Caltech101 class names in the condition
is replaced with ImageNet class names. $L_2$ distances of the condition
embedding and of the sampled prompt are measured with respect to the
original (0\%) condition.}
\label{tab:replaced-names-steering-comparison}
\small
\resizebox{\textwidth}{!}{%
\begin{tabular}{ccccc|ccc}
\toprule
\textbf{\shortstack{Replaced\\names}}
& \textbf{\shortstack{Condition\\$L_2$}}
& \multicolumn{3}{c|}{\textbf{w/o Steering}}
& \multicolumn{3}{c}{\textbf{w/ Steering}} \\
\cmidrule(lr){3-5}
\cmidrule(lr){6-8}
&
& \textbf{\shortstack{Prompt\\$L_2$}}
& \textbf{Caltech101}
& \textbf{ImageNet}
& \textbf{\shortstack{Prompt\\$L_2$}}
& \textbf{Caltech101}
& \textbf{ImageNet} \\
\midrule
0\%  & 0.0  & 0.000 & 92.0 & 64.7 & 0.000 & 93.0 & 64.9 \\
40\% & 24.2 & 1.623 & 92.1 & 64.4 & 0.025 & 93.0 & 64.8 \\
60\% & 29.8 & 1.606 & 92.6 & 64.3 & 0.049 & 93.1 & 64.5 \\
80\% & 34.0 & 1.600 & 92.5 & 64.5 & 0.049 & 93.2 & 64.5 \\
\bottomrule
\end{tabular}%
}
\end{table}

See Appendix~\ref{app:dmpmulti} for application of DMP for personalization and image editing, Sec.~\ref{sec:ablation_appendix} for additional ablations and Sec.~\ref{app:limitations} for a discussion on the limitations and scope for future works.

\vspace{-4mm}
\section{Conclusion}
\vspace{-3mm}
We introduced a \emph{Diffusion Meta-Prompting} (DMP) framework that learns a distribution over prompts and enables synthesis of task-specific prompts directly from natural language. Across recognition, retrieval, and generation settings, DMP produces representations and
prompts that generalize beyond the tasks, label sets, views, and concepts seen during
training, while avoiding the overhead of storing and searching large repositories. Empirically, DMP yields consistent gains in the most challenging generalization regimes improving cross-task and composite classification, supports compositional control and semantic negation, and improves prompt compliance for subject variations. Finally, we provided a theoretical bound linking the expected downstream task loss of DMP-sampled prompts to repository quality and diffusion denoising error, offering a principled perspective on when meta-prompt generation succeeds.

\subsection*{Acknowledgements}
This work was partially funded by Qualcomm Innovation Fellowship 2025, NSF awards IIS-2303153 and NAIRR-240300. We also acknowledge and thank the use of the Nautilus platform for some of the experiments.

\subsection*{Broader Impact}\label{broader-impact}
We introduce a new meta-learning framework for generating prompts for foundation models. It carries the risks associated with generative foundation models. While the proposed method offers benefits of better generalization with storage and runtime efficiency, it uses existing pretrained models which are shown to contain harmful biases that maybe elicited by the prompts. It can also be potentially misused to propagate harmful, unlawful or unethical information with the personalization of celebrities. Since, the framework is meta-learning, any harmful prompts can be identified before the image generation step where the embeddings can be inspected with nearest neighbor tokens in the text space. Additionally, advancements in image watermarking \citep{luo2022leca,watermarksurvey,synthid,gowal2025synthidimage} can help to identify generated image contents to protect against these risks.

\bibliography{neurips_2026}

\newpage
\appendix
\section{Appendix}

\subsection{Theoretical Analysis}\label{sec:theory_analysis}
\setcounter{proposition}{0}
Let $\mathcal{C}$ denote a distribution over tasks or concepts $c \sim \mathcal{C}$. For each concept $c$, we assume a textual description $y(c)$ (e.g., ``a personalization prompt for Jennifer Aniston''), a downstream task with loss $\mathcal{L}_{\text{task}}(S;c)$ that evaluates the performance of a candidate prompt $S$ in prompt space $\mathcal{P}$ when applied to a frozen foundation model $F$, and a repository $\mathcal{R}$ of exemplar prompts $x(c)$ obtained from existing prompt-learning techniques.

Our objective in this section is to bound the \emph{expected downstream loss} when sampling prompts from the pretrained generator:
\[
\mathcal{J}_{\mathrm{pre}}(\theta)
:= \mathbb{E}_{c\sim\mathcal{C}}\Big[ \mathbb{E}_{S\sim p_\theta(\cdot\mid y(c))} \big[ \mathcal{L}_{\mathrm{task}}(S;c) \big] \Big].
\]

\textbf{High-level intuition.}
When the trained diffusion model $p_\theta(\cdot\mid y)$ closely matches the repository/data distribution $p_{\mathrm{data}}(\cdot\mid y)$ in distributional distance, the expected task loss under $p_\theta$ is close to the expected \emph{repository} task loss under $p_{\mathrm{data}}$. If the repository was constructed to contain useful prompts for downstream tasks (i.e., $p_{\mathrm{data}}$ has low expected task loss), then a small distributional discrepancy implies low expected task loss for $p_\theta$ as well. We make this statement precise with the following bounds.

\subsubsection{Distributional discrepancy bound}
We begin with a straightforward decomposition and use standard total-variation and Pinsker inequalities~\citep{cover2006elements,tsybakov2008introduction}.

\begin{proposition}[Distributional discrepancy bound]
\label{prop:tv_bound}
Let $p_{\mathrm{data}}(\cdot\mid y)$ and $p_\theta(\cdot\mid y)$ be two distributions on prompt space $\mathcal P$ for a fixed condition $y$. Assume the task loss $\mathcal L_{\mathrm{task}}(x)$ is bounded in $[0,L_{\max}]$. Then
\begin{equation}
\begin{aligned}
\Big|
\mathbb{E}_{S\sim p_\theta}
\big[\mathcal{L}_{\mathrm{task}}(x)\big]
&-
\mathbb{E}_{S\sim p_{\mathrm{data}}}
\big[\mathcal{L}_{\mathrm{task}}(x)\big]
\Big| \\
&\le
2 L_{\max}\,
\mathrm{TV}\!\big(p_\theta, p_{\mathrm{data}}\big) .
\end{aligned}
\label{eq:tv_ineq_prop}
\end{equation}

and by Pinsker's inequality,
\begin{equation}
\mathrm{TV}(p_\theta,p_{\mathrm{data}}) \le \sqrt{\tfrac{1}{2}\,\mathrm{KL}(p_{\mathrm{data}}\|p_\theta)}.
\label{eq:pinsker_prop}
\end{equation}
Consequently,
\begin{equation}
\begin{aligned}
\mathbb{E}_{S\sim p_\theta}
\big[\mathcal{L}_{\mathrm{task}}(x)\big]
&\le
\mathbb{E}_{S\sim p_{\mathrm{data}}}
\big[\mathcal{L}_{\mathrm{task}}(x)\big] \\
&\quad
+ L_{\max}
\sqrt{2\,\mathrm{KL}\!\left(p_{\mathrm{data}}\|p_\theta\right)} .
\end{aligned}
\label{eq:final_kl_bound_prop}
\end{equation}

\end{proposition}

\begin{proof}
Let $\ell(x)=\mathcal L_{\mathrm{task}}(x)$ and denote $\Delta := \mathbb{E}_{p_\theta}[\ell]-\mathbb{E}_{p_{\mathrm{data}}}[\ell]$.
By the definition of total variation (and the fact that $0\le\ell\le L_{\max}$),
\begin{equation}
\begin{aligned}
|\Delta|
&= \Big| \int \ell(x)\, (p_\theta - p_{\mathrm{data}})(dx) \Big| \\
&\le \int |\ell(x)|\, |p_\theta - p_{\mathrm{data}}|(dx) \\
&\le L_{\max} \int |p_\theta - p_{\mathrm{data}}|(dx).
\end{aligned}
\end{equation}

By definition $\mathrm{TV}(p_\theta,p_{\mathrm{data}})=\tfrac{1}{2}\int|p_\theta-p_{\mathrm{data}}|$, hence~\eqref{eq:tv_ineq_prop} holds.

Pinsker's inequality (see e.g.\ \citep{cover2006elements}) gives~\eqref{eq:pinsker_prop}. Combining the two inequalities yields~\eqref{eq:final_kl_bound_prop}.
\end{proof}

\noindent\textbf{Remarks.} Inequality~\eqref{eq:final_kl_bound_prop} reduces expected downstream loss under the model to two terms: the expected loss of the repository distribution (which is a function of data collection quality) and the KL divergence between the dataset distribution and the pretrained generator. The latter is controlled by how well the denoiser is trained (see next subsection).

\subsubsection{Relating denoising loss to model-data KL}
We now recall a standard link between the denoising objective used in DDPMs and a divergence between the model and data distributions (see, e.g.,~\citep{vincent2011connection, song2020score}). Under standard DDPM assumptions and appropriate variance schedule, minimizing the simplified denoising loss~\eqref{eq:loss} is equivalent (up to constants and time discretization effects) to score matching / denoising score-matching which estimates the score function $\nabla_x \log p_{\mathrm{data}}(x)$. A well-trained denoiser implies an accurate score estimator, which in turn implies a small KL divergence between the model and data distributions in $x_0$-space (the space of clean prompts). Formally, one can show:

\begin{lemma}[Denoising loss controls KL (informal)]
\label{lem:denoise_to_kl}
Under the standard DDPM/score-matching correspondence and mild regularity conditions, if the denoising risk satisfies
\[
\mathcal{L}_{\mathrm{denoise}}(\theta) \le \varepsilon,
\]
then the KL divergence between $p_{\mathrm{data}}(x_0\mid y)$ and $p_\theta(x_0\mid y)$ admits the bound
\[
\mathrm{KL}\big( p_{\mathrm{data}}(\cdot\mid y) \big\| p_\theta(\cdot\mid y) \big) \le C\,\varepsilon + o(1),
\]
where $C>0$ is a constant that depends on the variance schedule, the discretization, and model parametrization; the $o(1)$ term vanishes as the diffusion discretization becomes finer.
\end{lemma}

\noindent\textbf{Remarks.} The lemma is qualitative: precise constants follow from score-matching and likelihood bounds in the DDPM literature (e.g., \citep{ho2020denoising, song2020score}). The key message is that a small denoising loss implies a small divergence between the learned and data distributions.

Combining Lemma~\ref{lem:denoise_to_kl} with Proposition~\ref{prop:tv_bound} yields a bound of the form
\begin{equation}
\begin{aligned}
\mathbb{E}_{S\sim p_\theta}
\big[\mathcal{L}_{\mathrm{task}}(x)\big]
&\le
\mathbb{E}_{S\sim p_{\mathrm{data}}}
\big[\mathcal{L}_{\mathrm{task}}(x)\big] \\
&\quad
+ L_{\max}\sqrt{2 C\,\varepsilon}
+ o(1) .
\end{aligned}
\end{equation}

\subsubsection{Concentration from finite repository}
So far we have related the model expectation to the (population) data distribution. In practice we only train on finite $\mathcal{R}$ with $n$ samples per condition $y$. Let $\widehat p_{\mathrm{data}}$ denote the empirical distribution formed by the repository. By Hoeffding's inequality (or McDiarmid)~\citep{hoeffding1963probability,mcdiarmid1989method}, with probability at least $1-\delta$ over the draw of the repository,
\begin{equation}
\begin{aligned}
\Big|
\mathbb{E}_{S\sim \widehat p_{\mathrm{data}}}
\big[\mathcal{L}_{\mathrm{task}}(x)\big]
&-
\mathbb{E}_{S\sim p_{\mathrm{data}}}
\big[\mathcal{L}_{\mathrm{task}}(x)\big]
\Big| \\
&\le
L_{\max}
\sqrt{\tfrac{\log(2/\delta)}{2n}} .
\end{aligned}
\label{eq:hoeffding_repo}
\end{equation}

\subsubsection{DMP performance guarantee}
Combining the above pieces yields the following performance guarantee for DMP.

\begin{proposition}[DMP performance guarantee]
\label{thm:combined}
Assume the denoising loss satisfies $\mathcal{L}_{\mathrm{denoise}}(\theta)\le\varepsilon$ and the repository contains $n$ i.i.d.\ prompts for the condition $y$. Then with probability at least $1-\delta$ over the repository sample,
\begin{equation}
\begin{aligned}
\mathbb{E}_{S\sim p_\theta(\cdot\mid y)}\big[\mathcal{L}_{\mathrm{task}}(x)\big]
&\le \mathbb{E}_{S\sim \widehat p_{\mathrm{data}}(\cdot\mid y)}\big[\mathcal{L}_{\mathrm{task}}(x)\big] \\
&\quad + L_{\max}\sqrt{2C\,\varepsilon} \\
&\quad+ L_{\max}\sqrt{\tfrac{\log(2/\delta)}{2n}} + o(1),
\end{aligned}
\label{eq:combined_bound}
\end{equation}
where $C$ is the constant from Lemma~\ref{lem:denoise_to_kl} and the $o(1)$ term accounts for discretization error in the diffusion approximation.
\end{proposition}

\begin{proof}
Start from Proposition~\ref{prop:tv_bound} applied to $p_\theta$ and $p_{\mathrm{data}}$:
\[
\mathbb{E}_{p_\theta}[\ell] \le \mathbb{E}_{p_{\mathrm{data}}}[\ell] + 2L_{\max}\sqrt{\tfrac{1}{2}\,\mathrm{KL}(p_{\mathrm{data}}\|p_\theta)}.
\]
Apply Lemma~\ref{lem:denoise_to_kl} to bound the KL by $C\varepsilon + o(1)$; hence
\[
\mathbb{E}_{p_\theta}[\ell] \le \mathbb{E}_{p_{\mathrm{data}}}[\ell] + L_{\max}\sqrt{2C\,\varepsilon} + o(1).
\]
Now replace the population expectation $\mathbb{E}_{p_{\mathrm{data}}}[\ell]$ by the empirical expectation $\mathbb{E}_{\widehat p_{\mathrm{data}}}[\ell]$ and apply Proposition~\ref{eq:hoeffding_repo} (Hoeffding) which with probability at least $1-\delta$ yields the stated sampling error term $L_{\max}\sqrt{\tfrac{\log(2/\delta)}{2n}}$. Combining these terms gives~\eqref{eq:combined_bound}.
\end{proof}

\noindent\textbf{Interpretation.} Bound~\eqref{eq:combined_bound} decomposes the expected task loss under the pretrained generator into (i) the empirical repository loss (quality of collected prompts), (ii) an approximation term controlled by the denoising risk (how well DMP models the repository), and (iii) a sampling term that vanishes as the repository size $n$ grows.


\begin{figure}[h]
  \centering
  \includegraphics[width=\columnwidth, trim=220 150 220 60, clip]{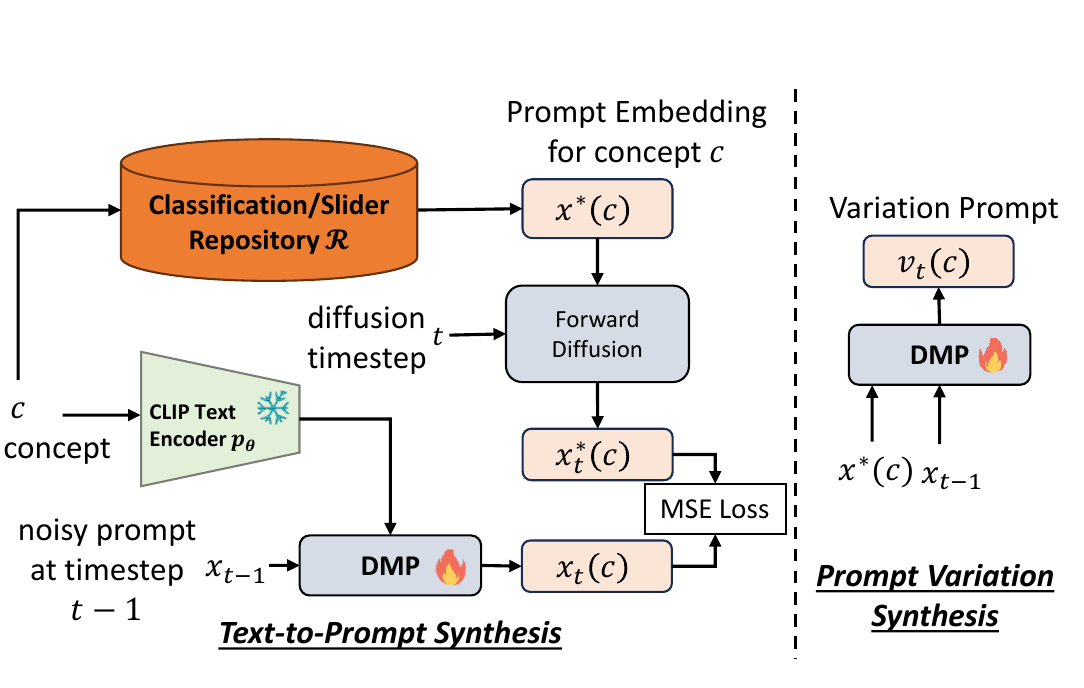}
  \caption{\label{fig:zoomedarch} \textbf{Left:} Diffusion Meta-Prompt framework for Text-to-Prompt synthesis. \textbf{Right:} Prompt Variation synthesis conditioned on learned Textual Inversion Prompts.}
\end{figure}
Figure~\ref{fig:zoomedarch} shows the DMP framework where the left part denotes the training of text-to-prompt diffusion model and right part denotes the diffusion model for generating prompt variations conditioned on the original prompts. 

\subsection{Relation to Prior Methods}

 Table~\ref{tab:difftableprior} illustrates how DMP differs from prior prompt-learning approaches across key axes such as requiring task-specific data or loss, supporting multiple tasks in a zero-shot manner, learning across multiple tasks and the need for storage of prompts. Classical prompt learning techniques are {\it task and model specific}, and require {\it task specific data and losses.} DMP instead learns the distribution of prompts from a prompt repository (produced by these prompt learning techniques). It can then be used to sample prompts for {\it many tasks}. Note that {\it DMP is not task specific and does not require task specific losses or data, just a prompt repository.} This makes DMP as a general prompt generator, free from the data and optimization constraints that characterize prior prompt-learning techniques. We show that, in many cases, DMP sampled prompts even outperform prompt learning prompts. For example, the classification results of Table~\ref{tab:supp_copromptcomparison} shows that DMP outperforms the base method on the unseen classes of the very same dataset, despite never accessing a single image.
\begin{table*}[h]
\centering
\scriptsize
\caption{\textbf{Conceptual differences between DMP and prior prompt learning methods.} 
Unlike prior works that refine prompts using data for a specific task or example, DMP treats prompts as a \emph{distribution} and trains a diffusion meta-model that amortizes prompt generation across many tasks and repositories. DMP does not require access to the original data used to train the prompts (L159--L161) and supports one-shot sampling, composition, and negative prompting without per-task optimization.}
\setlength{\tabcolsep}{2pt}
\resizebox{\linewidth}{!}{
\begin{tabular}{lccccl}
\toprule
\textbf{Method} & \textbf{Task Data/} & \textbf{Supports } &
\textbf{Learns } & \textbf{Minimal} & \textbf{Core Idea} \\
\textbf{} & \textbf{Loss Free} & \textbf{Multiple Tasks} &
\textbf{across Tasks} & \textbf{Storage} &  \\

\midrule
    Stylus / Retrieval-based        & {\color{green}\checkmark} & {\color{red}\ding{55}} & {\color{green}\checkmark} & {\color{red}\ding{55}} & Searches large repositories; limited generalization. \\
    \midrule
    \hline
    \multicolumn{6}{c}{Diffusion-based Methods}
    \\
    \hline
    \midrule
    Prompt Diffusion & {\color{red}\ding{55}} & {\color{red}\ding{55}} & {\color{red}\ding{55}} & {\color{red}\ding{55}} & Refines prompts for a specific task using data. \\
    \hline
    Diff-Prompt          & {\color{red}\ding{55}} & {\color{red}\ding{55}} & {\color{red}\ding{55}} & {\color{red}\ding{55}} & Mask-supervised diffusion for a single task. \\
    \hline
    Neural Network Diffusion        & {\color{red}\ding{55}} & {\color{red}\ding{55}} & {\color{red}\ding{55}} & {\color{green}\checkmark} & Diffusion to generate model weights  \\
    &&&&&for a particular task/dataset (weight-generation). \\
    \hline
Conditional LoRA Param Gen.    & {\color{red}\ding{55}} & {\color{green}\checkmark} & {\color{green}\checkmark} & {\color{green}\checkmark} & Diffusion/hypernetwork that generates LoRA adapters \\
    &&&&&conditioned on task (Cond P-Diff / CondLoRA). \\
    \hline
DiffLoRA                       & {\color{red}\ding{55}} & {\color{green}\checkmark} & {\color{green}\checkmark} & {\color{green}\checkmark} & Diffusion model predicts personalized low-rank (LoRA) \\
    &&&&&weights at inference (zero-shot personalization). \\
    \midrule
    \hline
    \multicolumn{6}{c}{Prompt Learning Methods}
    \\
    \hline
    \midrule
Hierarchical Variational TTP & {\color{red}\ding{55}} & {\color{red}\ding{55}} & {\color{red}\ding{55}} & {\color{red}\ding{55}} & Test-time variational prompt generator\\
&&&&& relying on task-specific data features. \\
\hline
Language-Aware Soft  Prompting     & {\color{red}\ding{55}} & {\color{red}\ding{55}} & {\color{red}\ding{55}} & {\color{red}\ding{55}} & Text-conditioned optimization  \\
 (LASP) &&&&&for soft prompts, task-specific. \\
\hline
Consistency-guided PL & {\color{red}\ding{55}} & {\color{red}\ding{55}} & {\color{red}\ding{55}} & {\color{red}\ding{55}} & Consistency loss improves 
\\
(CoPrompt)&&&&&prompt robustness; task-specific. \\
\hline
PromptKD     & {\color{red}\ding{55}} & {\color{red}\ding{55}} & {\color{red}\ding{55}} & {\color{red}\ding{55}} & Distills prompts from teacher to student 
\\
&&&&&in unsupervised setting; task-specific. \\
\hline
Dual-Prompt Collaboration          & {\color{red}\ding{55}} & {\color{red}\ding{55}} & {\color{red}\ding{55}} & {\color{red}\ding{55}} & Dual-prompt collaboration requiring
\\
(DPC)&&&&&tuned prompt parameters; task-specific. \\
\hline
Bayesian Prompt Learning & {\color{red}\ding{55}} & {\color{red}\ding{55}} & {\color{red}\ding{55}} & {\color{red}\ding{55}} & Bayesian uncertainty modeling 
\\
(BPL)&&&&&over task-specific prompts. \\
\hline
Patch-Prompt Aligned  & {\color{red}\ding{55}} & {\color{red}\ding{55}} & {\color{red}\ding{55}} & {\color{red}\ding{55}} & Bayesian hierarchical prompt generation 
\\
Bayesian Prompt Tuning&&&&&(label-specific stochastic prompts); per-task tuning. \\
    \midrule
    \hline
    \multicolumn{6}{c}{Meta-Learning Methods}
    \\
    \hline
    \midrule
    Prompt Learning via & {\color{red}\ding{55}} & {\color{green}\checkmark} & {\color{green}\checkmark} & {\color{red}\ding{55}} & Meta-regularization to improve prompt 
    \\
 Meta-Regularization (ProMetaR)&&&&&generalization across tasks; requires training data. \\
 \hline

Gradient-Regulated Meta-Prompt & {\color{red}\ding{55}} & {\color{green}\checkmark} & {\color{green}\checkmark} & {\color{red}\ding{55}} & Meta-learned prompt init + gradient 
\\
(GRAM)&&&&&regulator for few-shot cross-domain generalization. \\
\hline
PRewrite (Prompt Rewriting w/ RL) & {\color{red}\ding{55}} & {\color{green}\checkmark} & {\color{red}\ding{55}} & {\color{red}\ding{55}} & LLM-based prompt rewriter trained with RL  \\
&&&&&to improve downstream task performance. \\
\hline
    \midrule
    \textbf{DMP (Ours)}                & {\color{green}\checkmark} & {\color{green}\checkmark} & {\color{green}\checkmark} & {\color{green}\checkmark} & \textbf{Learns a prompt distribution; one-shot sampling, } 
    \\
&&&&&\textbf{composition, negative prompts.} \\
    \bottomrule
\end{tabular}
}
\label{tab:difftableprior}
\end{table*}

\subsection{DMP for Classification}\label{sec:appdmpclass}

\paragraph{Ablation on Steering Strength}\label{app:steeringstrength}
Table~\ref{app:ablation-strength} shows the effect of the steering strength on composite classification accuracy. The  trend suggests that moderate steering provides the best trade-off between stability and diversity during diffusion sampling.

\begin{figure}[h]
\scriptsize
\centering
\setlength{\tabcolsep}{6pt}
\captionof{table}{Ablation of steering strength $\lambda$ for DMP on 55 composite dataset-pairs classification. Moderate steering gives the strongest average accuracy, while overly aggressive steering begins to over-constrain the sampling trajectory.}
\label{app:ablation-strength}
\begin{tabular}{ccccc}
\toprule
$\lambda$ & All & Base & New & Avg. \\
\midrule
$0.1$ & 62.5 & 68.5 & 69.8 & 66.9 \\
$0.2$ & \textbf{63.1} & \textbf{69.4} & \textbf{70.8} & \textbf{67.8} \\
$0.3$ & 62.8 & 69.1 & 70.4 & 67.4 \\
$0.4$ & 62.2 & 68.4 & 69.6 & 66.7 \\
$0.5$ & 61.7 & 67.8 & 68.9 & 66.1 \\
$0.6$ & 61.6 & 67.9 & 69.1 & 66.2 \\
$0.7$ & 61.5 & 67.6 & 68.7 & 65.9 \\
$0.8$ & 61.3 & 67.2 & 68.3 & 65.6 \\
$0.9$ & 61.0 & 66.8 & 67.9 & 65.2 \\
$1.0$ & 60.7 & 66.4 & 67.5 & 64.9 \\
\bottomrule
\end{tabular}
\end{figure}

\begin{figure}
\begin{minipage}{\columnwidth}
\centering
\scriptsize
\captionof{table}{Domain generalization accuracy (\%) for DMPCoOp prompts sampled on ImageNet.}
\label{table:domaingen}
\setlength{\tabcolsep}{1.5pt}
\begin{tabular}{lcccccc}
    \toprule
    & \multicolumn{1}{c|}{Source} & \multicolumn{5}{c}{Target} \\
    Method & ImageNet & -V2 & -S & -A & -R & Average \\
    \midrule
    CoOp & 68.5 & 61.2 & 45.2 & 48.2 & 74.1 & 59.4 \\  
    DMPCoOp & {\textbf{68.7}} & {\textbf{62.1}} & {\textbf{46.5}} & {\textbf{50.6}} & {\textbf{75.1}} & {\textbf{60.6}} \\
    \bottomrule
\end{tabular}

\end{minipage}
    \begin{minipage}{\columnwidth}
\centering
\scriptsize
\captionof{table}{Domain generalization accuracy (\%) for DMPTAC prompts sampled on ImageNet.}
\label{table:tacdomaingen}
\setlength{\tabcolsep}{1.5pt}
\begin{tabular}{lcccccc}
    \toprule
    & \multicolumn{1}{c|}{Source} & \multicolumn{5}{c}{Target} \\
    Method & ImageNet & -V2 & -S & -A & -R & Average \\
    \midrule
    
    TAC & 71.3 & 64.6 & 48.3 & 48.6 & 76.6 & 61.9 \\ 
    
    DMPTAC & {\textbf{71.4}} & {\textbf{64.7}} & {\textbf{48.5}} & {\textbf{48.8}} & {\textbf{76.7}} & {\textbf{62.0}} \\
    \bottomrule
\end{tabular}

\end{minipage}
\hfill
\begin{minipage}{\columnwidth}
    \centering
\scriptsize
\captionof{table}{Cross-task generalization per dataset: performance for Mean Treecut Accuracy (MTA).}
\label{tab:tac_mta_only}
\begin{adjustbox}{max width=\textwidth}
\begin{tabular}{@{}l|c|c|c|c|c|c@{}}
\toprule
 \textbf{Method} & \textbf{ImageNet} & \textbf{-V2} & \textbf{-S} & \textbf{-R} & \textbf{-A} & \textbf{SUN397} \\ 
\cmidrule{2-7}
TAC & 71.3 & 34.1 & 28.3 & \textbf{5.5} & 36.5 & 30.0 \\ 
\textbf{DMPTAC} & \textbf{71.4} & \textbf{39.6} & \textbf{37.3} & \textbf{5.7} & \textbf{38.9} & \textbf{40.0} \\ 
\textbf{$\Delta$} 
& \textcolor{ForestGreen}{+0.1}
& \textcolor{ForestGreen}{+5.5}
& \textcolor{ForestGreen}{+9.0}
& \textcolor{ForestGreen}{+0.2}
& \textcolor{ForestGreen}{+2.4}
& \textcolor{ForestGreen}{+10.0} \\
\hline
\end{tabular}
\end{adjustbox}
\end{minipage}

\begin{minipage}{\columnwidth}
        \scriptsize
        \centering
        \captionof{table}{Comparison of DMPMulti with TI prompts.}
        \label{tab:comp_ids}
        \setlength{\tabcolsep}{2pt}
        \begin{tabular}{lcccc}
            \toprule
            \textbf{Method (SD-RV)} & \textbf{Face-ID$\uparrow$} & \textbf{DINO $\downarrow$} & \textbf{CLIP-I $\downarrow$} & \textbf{CLIP-T $\uparrow$} \\
            \midrule
            Textual Inversion & 0.428 & 0.627 & 0.696 & 0.244 \\
            Stylus (Top-1)&	\textbf{0.434}&	0.645&	0.706&	\underline{0.246} \\
            DMPMulti & \textbf{0.434} & \textbf{0.558} & \underline{0.653} & 0.245 \\
            DMPMulti(20)&	\underline{0.429}&	\underline{0.595}&	\textbf{0.599}&	\textbf{0.285} \\
            \bottomrule
        \end{tabular}
    \end{minipage}
        \end{figure}

\begin{figure}[h]
        \centering
        
        \begin{minipage}{\columnwidth}
            \centering
            \scriptsize
            \captionof{table}{Comparison of slider prompts generated by a separate DMP (DMPSlider) against DMPMulti.}
            \label{tab:slidersynthetic_app}
            \setlength{\tabcolsep}{2pt}
            \begin{tabular}{ l| l |l} 
            \toprule
             \textbf{Method}& \textbf{CLIP-s} $\uparrow$& \textbf{LPIPS}$\downarrow$\\
            \midrule
            Prompt Slider &   {\textbf{30.00}}& 0.219\\
            DMPMulti & \textbf{29.86}& {\textbf{0.126}}\\
            DMPSlider (separate) & \textbf{29.88}& {\textbf{0.121}}\\
            \bottomrule
            \end{tabular}
            
        \end{minipage}
        \begin{minipage}{\columnwidth}
        \scriptsize
        \centering
            \captionof{table}{Comparison of separately trained DMP (DMPIdentity) with Textual Inversion and DMPMulti for identity synthesis.}
            \label{tab:comp_ids_app}
            \setlength{\tabcolsep}{2pt}
            \begin{tabular}{lcccc}
                \toprule
                \textbf{Method (SD-RV)} & \textbf{Face-ID$\uparrow$} & \textbf{DINO $\downarrow$} & \textbf{CLIP-I $\downarrow$} & \textbf{CLIP-T $\uparrow$} \\
                \midrule
                Textual Inversion & 0.428 & 0.627 & 0.696 & 0.244 \\
                DMPMulti & {\textbf{0.434}} & {\textbf{0.558}} & {\textbf{0.653}} & {\textbf{0.245}} \\
                DMPIdentity (separate) & {\textbf{0.435}} & {\textbf{0.550}} & {\textbf{0.659}} & {\textbf{0.246}} \\
                \bottomrule
            \end{tabular}
        \end{minipage}
\end{figure}

\paragraph{Composite classification with other prompt-learning methods.}
We further evaluate DMP on the composite classification task using recent prompt-learning
methods, including MaPLe~\citep{khattakMaPLe} and TAC~\citep{Hao_2025_CVPR_TAC}.
The results are reported as the average accuracy over the 55 pairwise dataset compositions
defined in Sec.~\ref{sec:dmpgeneralization}. 

Table~\ref{tab:composite_stronger_methods} shows that DMP yields consistent improvements
across prompt-learning backbones. The gains are largest for CoOp since we learn the full prompt distribution using a diffusion model. For other prompt learning
methods such as MaPLe and TAC, the gains are smaller but remain consistent, reaching up to
$+1.2/+0.6$ points for unseen classes for MaPLe/TAC respectively. This is expected because these methods learn additional parameters beyond the
soft prompts, such as projection matrices or task-aware pre-context modules, which are too
high-dimensional to synthesize with the current DMP implementation and available training
resources. In these cases, DMP only synthesizes the prompt parameters while keeping the
method-specific auxiliary parameters fixed. Extending DMP to jointly synthesize these
higher-dimensional components would likely require compression or low-rank parameterization
strategies, which we leave for future work.

\begin{table}[t]
\centering
\scriptsize
\caption{Composite classification accuracy (\%) averaged over 55 dataset pairs. DMP yields
consistent improvements across prompt-learning methods. Gains are largest for CoOp, which
only learns soft prompts, and smaller for MaPLe and TAC, which include additional
high-dimensional learned components that are kept fixed in our current implementation.}
\label{tab:composite_stronger_methods}
\begin{tabular}{lccc}
\toprule
Method & All & Base & New \\
\midrule
CoOp & 58.7 & 65.1 & 65.1 \\
DMPCoOp & 63.1 & 69.4 & 70.8 \\
$\Delta$ & $+4.4$ & $+4.3$ & $+5.7$ \\
\midrule
MaPLe & 57.2 & 64.3 & 63.0 \\
DMPMaPLe & 57.9 & 64.7 & 64.2 \\
$\Delta$ & $+0.7$ & $+0.4$ & $+1.2$ \\
\midrule
TAC & 63.2 & 69.6 & 69.7 \\
DMPTAC & 63.5 & 70.0 & 70.4 \\
$\Delta$ & $+0.3$ & $+0.4$ & $+0.6$ \\
\bottomrule
\end{tabular}
\end{table}

\paragraph{Same dataset (Base to New generalization).}

In this setting, prompts are trained in the Base classes of each dataset and evaluated on all classes of the same dataset. This is the least demanding setting, since the prompts have to generalize only to unseen classes of the dataset where they were trained.

Table \ref{tab:supp_copromptcomparison} summarizes the performance of PL and DMP prompts for four prompt learning methods, CoOp, CoPrompt, Maple, and TAC over the baselines across 11 datasets for Base-to-novel class generalization. The reported results are the average over three seed runs. As expected, DMP prompts have slightly lower performance for the base classes, since they are inherently bounded by the accuracy of the PL prompts used to train the DMP. Further, Proposition~\ref{prop:tv_bound} suggests that the gap could be bridge by using more than 40 prompts per dataset to train the DMP model, we show that this is true by ablating the number of prompts used for training the DMP in Table~\ref{tab:ablationpromptnumber}.  The table shows that, even under these conditions, the DMP prompts generalize better than the baseline, achieving an average accuracy gain between \textbf{+0.4\%} and \textbf{4.8\%} for New (unseen) classes, across prompting methods. For new classes, DMP outperforms CoOp on \textbf{11/11} datasets for all methods listed on the table. The table details the results for three datasets for brevity.

With regards to prompting methods, the table shows that the DMP gains are larger for CoOp, then Co Prompt, MaPLe and finally TAC. This is expected because DMP only synthesizes prompts. However, all methods other than CoOp, complement these prompts with additional  learned weights or projection matrices. For these methods, we used the original learned weights together with the prompts synthesized by DMP. Despite only modifying prompts, DMP achieves gains of \textbf{+0.4-0.8\%} over these methods for unseen classes. While a DMP-type technique could be used to sample the additional weights, this is a more complex endeavor due to the high dimensionality of the latter and compute requirements. We leave the synthesis of large parameter matrices for future work. 

In any case, it can be concluded that, even on the same dataset setting, {\it meta-prompting with a diffusion model outperforms the existing prompting techniques that are used to train it}.

These findings suggest that the original learned prompts may be somewhat overfitted to the base classes, whereas DMP prompts induce a classifier of stronger generalization. Over all classes (both seen and unseen) and datasets, DMP has an accuracy gain of upto \textbf{+3.0\%}. It can be concluded that {\it meta-prompting with a diffusion model outperforms the existing prompting techniques that are used to train it}.

\begin{figure*}[h]
\centering
\scriptsize
\setlength{\tabcolsep}{2pt}
\captionof{table}{Base2new generalization per dataset: performance for All, Base, and New classes, and HM (Harmonic Mean). The results reported are the average over three seed runs. The Average column reports the average over 11 datasets. * denotes our implementation as the code is not publicly available.}
\label{tab:supp_copromptcomparison}
\begin{adjustbox}{max width=\textwidth}
\begin{tabular}{@{}l|cccc|cccc|cccc|cccc@{}}
\toprule
 & \multicolumn{4}{c|}{\textbf{(a) Average}} & \multicolumn{4}{c|}{\textbf{(b) ImageNet}} & \multicolumn{4}{c|}{\textbf{(c) Caltech101}} & \multicolumn{4}{c}{\textbf{(d) OxfordPets}} \\ 
\cmidrule{2-17}
\textbf{Method} & All & Base & New & HM & All & Base & New & HM & All & Base & New & HM & All & Base & New & HM \\ 
\midrule
VAE~\citep{kingma2013auto}	&58.6&	62.7&	68.3&	65.0&	57.0&	64.1&	57.8&	60.8&	89.1&	91.7&	94.3&	92.9&	90.1&	92.0&	97.1&	94.5 \\
Top-1 Retrieval~\citep{luo2024stylus} & 67.5 & 80.7 & 72.8 & 76.5 & 50.5 & 78.4 & 47.9 & 59.5 & 37.3 & 76.9 & 46.4 & 57.9 & 69.9 & 85.8 & 66.8 & 75.1 \\
Transformer~\citep{vaswani2017attention}  & 68.3 &82.3  & 69.4 & 75.3 &  68.1& 76.4 &  66.8& 71.3 & 94.8 &  98.2& 94.8 & 96.5 & 89.9 & 94.6 & 95.4 & 95.0 \\
GPT-2~\citep{radford2019language}  & 68.1 &81.8  & 69.1 & 74.9 &  68.0& 76.2 &  66.8& 71.2 & 94.8 &  98.3& 95.0 & 96.6 & 90.1 & 94.7 & 95.9 & 95.3 \\ 
Prompt Diffusion*~\citep{du2024promptdiffusionrobustifiesanymodality}  &  67.0& 69.9& 73.0& 71.4& 54.7& 60.6& 56.4& 58.4& 91.3& 94.4& 93.2& 93.8& 92.3& 94.3& {\textbf{97.5}}&95.9\\
\midrule
CoOp~\citep{zhou2022coop} & \textbf{67.1} & \textbf{80.9} & 68.6 & \textbf{74.2} & 68.0 & \textbf{76.2} & 66.8 & 71.2 & 93.8 & 98.1 & 93.0 & 95.5 & 90.6 & 95.2 & 96.5 & 95.8 \\
\textbf{DMPCoOp} & {\textbf{70.1}} & 80.3 & {\textbf{73.4}} & {\textbf{76.5}} & {\textbf{68.7}} & 75.4 & {\textbf{68.8}} & {\textbf{72.0}} & {\textbf{94.8}} & {\textbf{98.3}} & {\textbf{95.3}} & {\textbf{96.8}} & {\textbf{91.7}} & {\textbf{95.4}} & {\textbf{97.3}} & {\textbf{96.3}} \\
\textbf{$\Delta$} & \textcolor{ForestGreen}{+3.0} & \textcolor{Crimson}{-0.6} & \textcolor{ForestGreen}{+4.8} & \textcolor{ForestGreen}{+2.3} & \textcolor{ForestGreen}{+0.7} & \textcolor{Crimson}{-0.8} & \textcolor{ForestGreen}{+2.0} & \textcolor{ForestGreen}{+0.8} & \textcolor{ForestGreen}{+1.0} & \textcolor{ForestGreen}{+0.2} & \textcolor{ForestGreen}{+2.3} & \textcolor{ForestGreen}{+1.3} & \textcolor{ForestGreen}{+1.1} & \textcolor{ForestGreen}{+0.2} & \textcolor{ForestGreen}{+0.8} & \textcolor{ForestGreen}{+0.5} \\
\midrule
CoPrompt~\citep{CoPrompt} & 72.3 & \textbf{83.1} & 74.6 & 78.3 & \textbf{70.7} & \textbf{76.7} & 71.4 & \textbf{73.9} & \textbf{95.8} & \textbf{98.7} & 95.3 & \textbf{97.0} & 91.1 & \textbf{95.3} & 97.0 & 96.1 \\
\textbf{DMPCoPrompt} & \textbf{72.9} & 82.5 & \textbf{75.4} & \textbf{78.5} & \textbf{70.7} & 76.6 & \textbf{71.5} & \textbf{73.9} & 95.7 & \textbf{98.7} & \textbf{95.4} & \textbf{97.0} & \textbf{91.6} & 95.2 & \textbf{97.2} & \textbf{96.2} \\
\textbf{$\Delta$} 
& \textcolor{ForestGreen}{+0.6} 
& \textcolor{Crimson}{-0.6} 
& \textcolor{ForestGreen}{+0.8} 
& \textcolor{ForestGreen}{+0.2} 
& 0.0 
& \textcolor{Crimson}{-0.1} 
& \textcolor{ForestGreen}{+0.1} 
& 0.0 
& \textcolor{Crimson}{-0.1} 
& 0.0 
& \textcolor{ForestGreen}{+0.1} 
& 0.0 
& \textcolor{ForestGreen}{+0.5} 
& \textcolor{Crimson}{-0.1} 
& \textcolor{ForestGreen}{+0.2} 
& \textcolor{ForestGreen}{+0.1} \\
\midrule
Maple~\citep{khattakMaPLe}  & 72.0 & \textbf{82.2} & 75.1 & 78.2 & 70.2 & 76.7 & 70.5 & 73.5 & 94.5 & \textbf{98.0} & 94.3 & 96.1 & \textbf{92.3} & \textbf{95.4} & \textbf{97.8} & \textbf{96.6} \\
\textbf{DMPMaple } & \textbf{72.4} & 82.0 & \textbf{75.9} & \textbf{78.6} & \textbf{70.3} & \textbf{76.8} & \textbf{70.6} & \textbf{73.6} & \textbf{94.8} & \textbf{98.0} & \textbf{95.5} & \textbf{96.7} & \textbf{92.3} & \textbf{95.4} & 97.6 & 96.5 \\
\textbf{$\Delta$} 
& \textcolor{ForestGreen}{+0.4} 
& \textcolor{Crimson}{-0.2} 
& \textcolor{ForestGreen}{+0.8} 
& \textcolor{ForestGreen}{+0.4} 
& \textcolor{ForestGreen}{+0.1} 
& \textcolor{ForestGreen}{+0.1} 
& \textcolor{ForestGreen}{+0.1} 
& \textcolor{ForestGreen}{+0.1} 
& \textcolor{ForestGreen}{+0.3} 
& 0.0 
& \textcolor{ForestGreen}{+1.2} 
& \textcolor{ForestGreen}{+0.6} 
& 0.0 
& 0.0 
& \textcolor{Crimson}{-0.2} 
& \textcolor{Crimson}{-0.1} \\
\midrule
TAC~\citep{Hao_2025_CVPR_TAC}  & 74.6 & \textbf{85.2} & 77.1 & 80.8 & 71.3 & \textbf{78.5} & 71.0 & \textbf{74.6} & \textbf{95.2} & \textbf{98.6} & \textbf{95.0} & 96.7 & \textbf{93.1} & \textbf{96.0} & 98.0 & \textbf{97.0} \\
\textbf{DMPTAC } & \textbf{75.0} & 85.1 & \textbf{77.5} & \textbf{80.9} & \textbf{71.4} & \textbf{78.5} & \textbf{71.2} & \textbf{74.6} & \textbf{95.2} & \textbf{98.6} & \textbf{95.0} & \textbf{96.8} & \textbf{93.1} & 95.9 & \textbf{98.2} & \textbf{97.0} \\
\textbf{$\Delta$} 
& \textcolor{ForestGreen}{+0.4} 
& \textcolor{Crimson}{-0.1} 
& \textcolor{ForestGreen}{+0.4} 
& \textcolor{ForestGreen}{+0.1} 
& \textcolor{ForestGreen}{+0.1} 
& 0.0 
& \textcolor{ForestGreen}{+0.2} 
& 0.0 
& 0.0 
& 0.0 
& 0.0 
& \textcolor{ForestGreen}{+0.1} 
& 0.0 
& \textcolor{Crimson}{-0.1} 
& \textcolor{ForestGreen}{+0.2} 
& 0.0 \\
\bottomrule
\end{tabular}
\end{adjustbox}

\end{figure*}

Table~\ref{tab:mta_hca} shows the full results of cross-task generalization experiment described in section~\ref{sec:dmpgeneralization}. We note that DMPCoOp obtains higher accuracies consistently across both the Mean Treecut Accuracy and Hierarchical Consistency Accuracy metrics on all the datasets considered. Notably, DMPCoOp obtains \textbf{+11.4\%} and \textbf{+5\%} improvement on SUN dataset for MTA and HCA respectively. This shows that DMPCoOp prompts are more robust and generalize well beyond the task on which it was trained.

We further provide the results of cross-domain experiments for TAC method in Table~\ref{table:tacdomaingen}, and cross-task generalization in Table~\ref{tab:tac_mta_only}. The gains over baseline TAC are larger for cross-task generalization where DMPTAC prompts obtains \textbf{+10\% on SUN and +9\% on Imagenet-Sketch} datasets for hierarchical classification.

\textbf{t-SNE visualization.} To assess the fidelity and diversity of prompts synthesized by the DMP framework, we conducted an embedding space analysis comparing real prompts from the repository $\mathcal{R}$ with DMP-generated prompts sampled using different random noise seeds. Figure~\ref{fig:tsne} presents a two-dimensional t-SNE projection of both sets of embeddings, with real prompts in blue and generated prompts in orange for SUN397, FGVC, Oxford Flowers and Stanford Cars datasets respectively from left to right.

The visualization reveals that generated prompts broadly overlap with the real prompt manifold, while also exhibiting a greater spread, indicative of higher variability. This suggests that the model captures the underlying structure of prompt space without resorting to memorization, while also producing novel variations.

To quantify these observations, we computed the fraction of a prompt’s 5 nearest neighbors (in embedding space) that share the same class label between real and generated prompts across 11 datasets. On average, 70.7\% of generated prompts share the same nearest-neighbor labels as their real counterparts, showing strong semantic alignment between real and generated embeddings while ensuring diversity.

The clusters that are seen in the plots show that the prompts sampled by DMP are more consistent than those originally learned by the prompt learning method. This could explain the improved generalization of these prompts and the lower variance of the performance of the prompted model to different noise seeds. We observed a standard deviation of 0.7\% from the mean accuracy for prompts sampled by DMPCoOp with 200 different noise seeds. In contrast, the standard deviation of the baseline prompt learning method used to build the repository from different initialization seeds show a higher variation of 4.7\%.

We additionally evaluated five prompts generated from different random seeds on the combined Caltech101--ImageNet task and measured prediction disagreement, per-class accuracy correlation, and error-set overlap. Table \ref{tab:seed-agreement} shows that, across the five  prompts, the standard deviation of accuracy is only 0.09 percentage points (range: 0.2 points). The per-class accuracy correlation of 0.999, and the error set overlap of 98.66\% indicate highly consistent decision boundaries across random seeds. On the other hand, the fact that different sampled prompts disagree on approximately 1.08\% of individual predictions shows that they are not identical. Overall, these results indicate that diffusion sampling yields slight local variations while preserving essentially the same functional behavior and generalization performance.

\begin{table}[htbp]
\centering
\caption{Agreement statistics across different random seeds.}
\label{tab:seed-agreement}
\small
\begin{tabularx}{\textwidth}{
    l
    c
    c
    c
    c
}
\toprule
\textbf{Prompt pair}
& \textbf{Accuracy (\%)}
& \textbf{\shortstack{Prediction\\disagreement (\%)}}
& \textbf{\shortstack{Per-class accuracy\\correlation}}
& \textbf{\shortstack{Error-set\\overlap (\%)}} \\
\midrule
Seed 1 vs.\ Seed 2 & 59.4 / 59.5 & 1.01 & 0.9992 & 98.81 \\
Seed 1 vs.\ Seed 3 & 59.4 / 59.5 & 1.12 & 0.9991 & 98.68 \\
Seed 1 vs.\ Seed 4 & 59.4 / 59.3 & 1.52 & 0.9982 & 97.97 \\
Seed 1 vs.\ Seed 5 & 59.4 / 59.5 & 0.97 & 0.9992 & 98.78 \\
Seed 2 vs.\ Seed 3 & 59.5 / 59.5 & 0.79 & 0.9994 & 99.07 \\
Seed 2 vs.\ Seed 4 & 59.5 / 59.3 & 1.43 & 0.9985 & 98.19 \\
Seed 2 vs.\ Seed 5 & 59.5 / 59.5 & 0.71 & 0.9995 & 99.18 \\
Seed 3 vs.\ Seed 4 & 59.5 / 59.3 & 1.21 & 0.9986 & 98.44 \\
Seed 3 vs.\ Seed 5 & 59.5 / 59.5 & 0.79 & 0.9994 & 99.05 \\
Seed 4 vs.\ Seed 5 & 59.3 / 59.5 & 1.19 & 0.9987 & 98.43 \\
\midrule
\textbf{Mean} & \textbf{59.44} & \textbf{1.08} & \textbf{0.9990} & \textbf{98.66} \\
\bottomrule
\end{tabularx}
\end{table}

\begin{figure*}[t]
\centering
\includegraphics[width=0.24\textwidth, trim=30 0 30 0, clip]{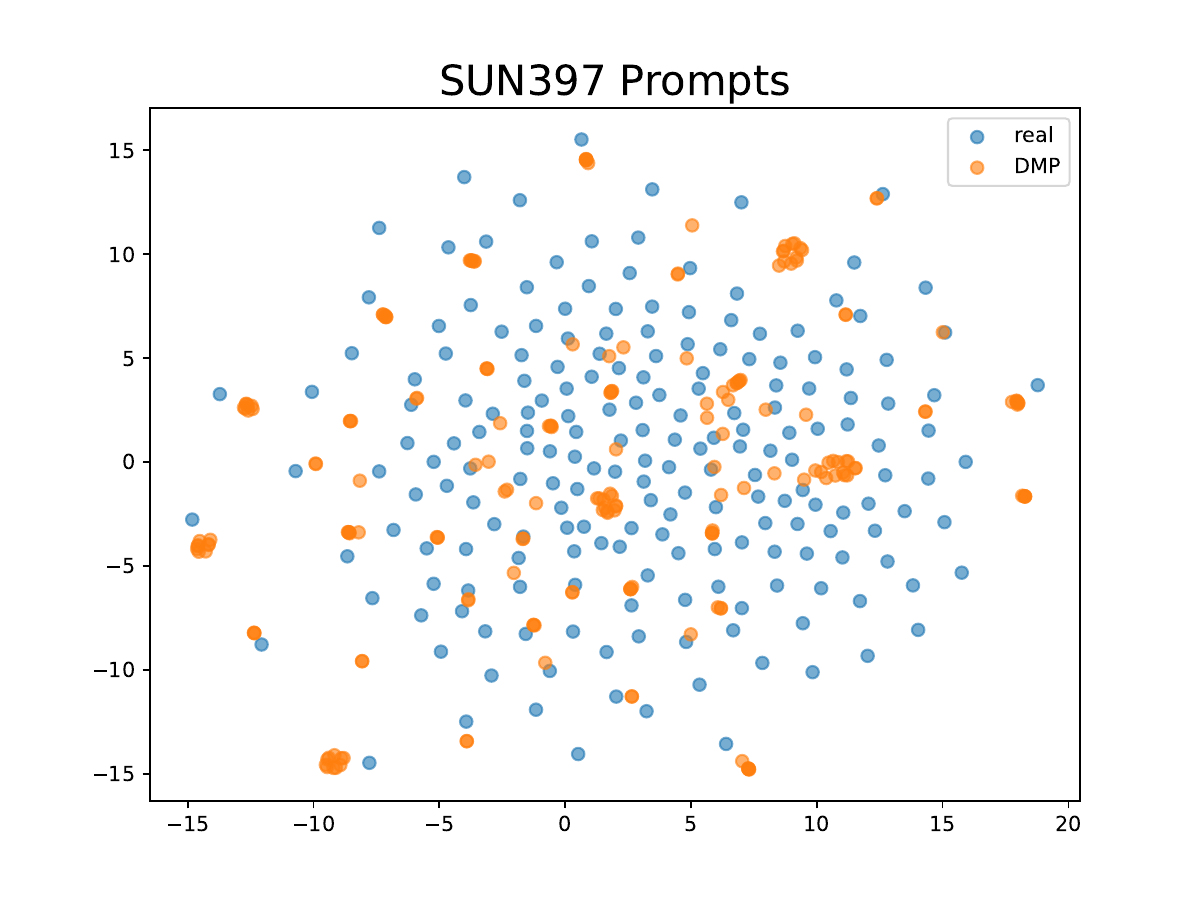}
\hfill
\includegraphics[width=0.24\textwidth, trim=30 0 30 0, clip]{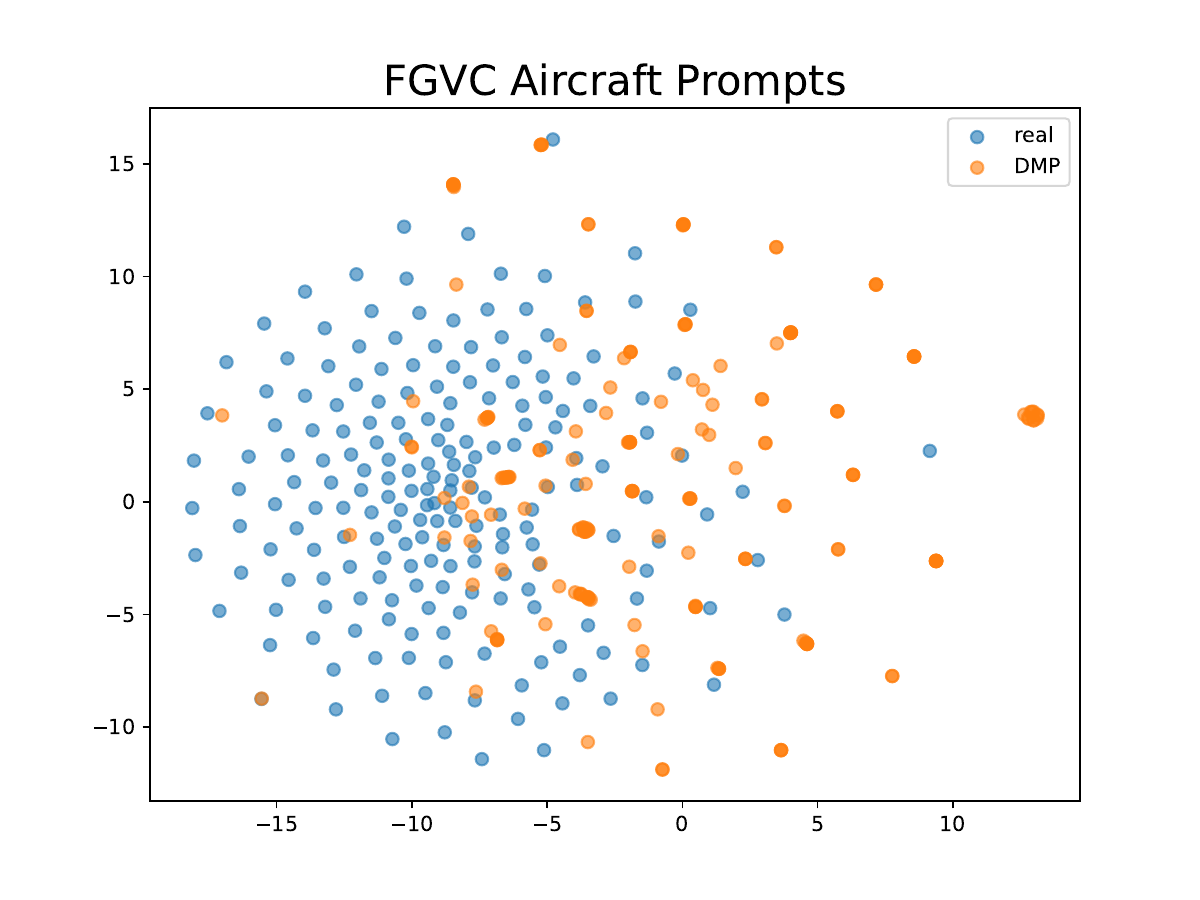}
\hfill
\includegraphics[width=0.24\textwidth, trim=30 0 30 0, clip]{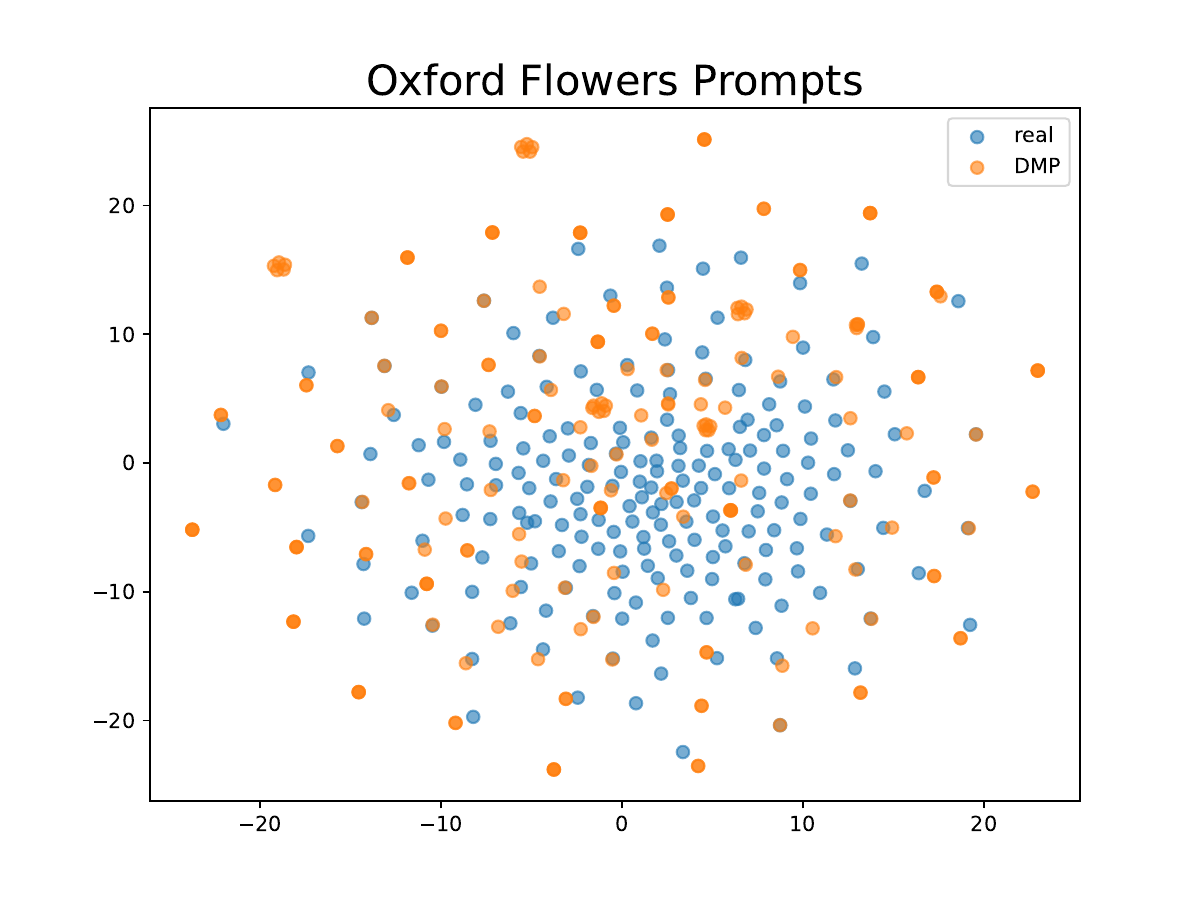}
\hfill
\includegraphics[width=0.24\textwidth, trim=30 0 30 0, clip]{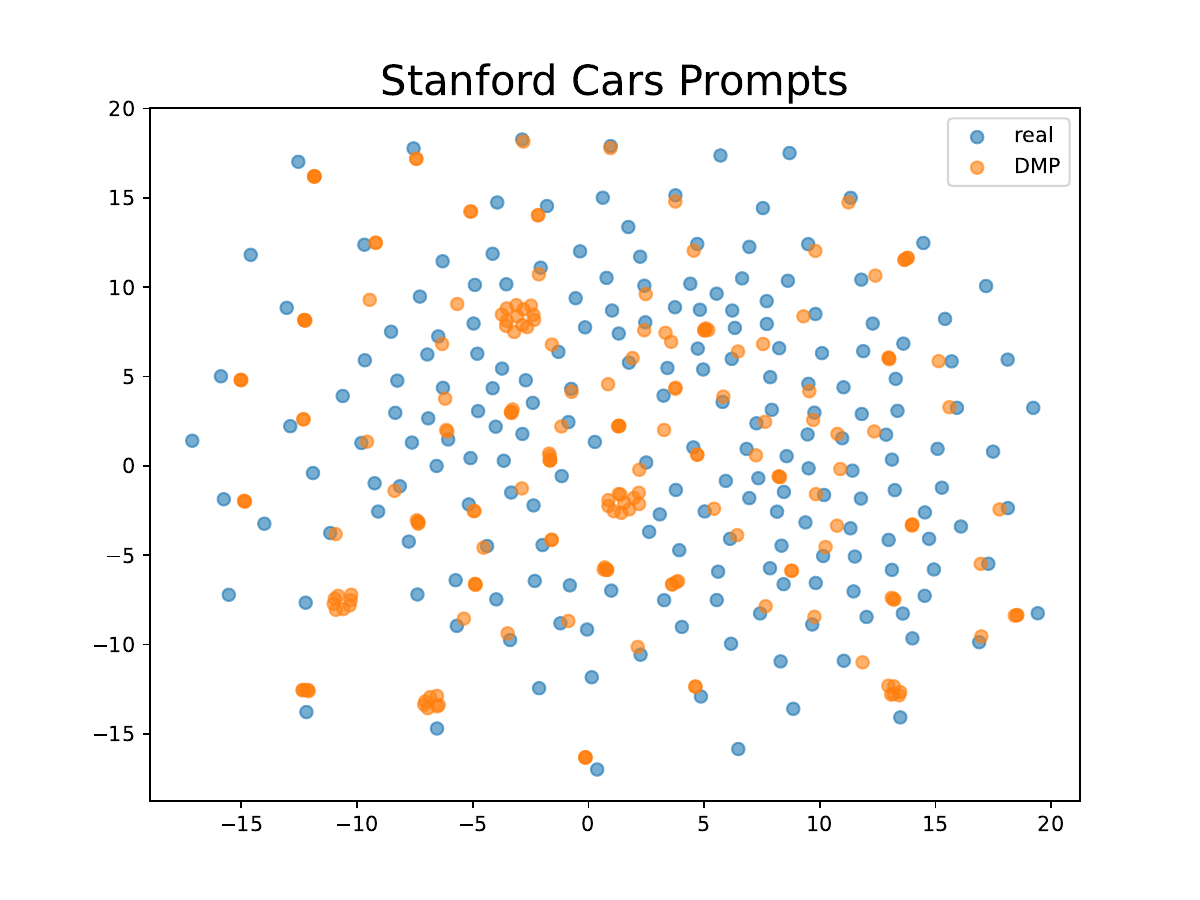}

\caption{\label{fig:tsne} 
\textbf{t-SNE visualizations of prompt embeddings.} Each figure shows real prompts (blue) and DMP-generated prompts (orange) with different noise seeds for different datasets: 
(a) SUN397, 
(b) FGVC-Aircraft, 
(c) Oxford Flowers, 
(d) Stanford Cars. 
Generated prompts broadly overlap with the real prompt manifolds while exhibiting greater spread, demonstrating both fidelity and diversity across domains.}
\end{figure*}

\begin{figure*}[h]
\centering
\scriptsize
\captionof{table}{Cross-task generalization per dataset: performance for Mean Treecut Accuracy (MTA)~\citep{protect}, and Hierarchical Consistency Accuracy (HCA)~\citep{protect}. }
\label{tab:mta_hca}
\begin{adjustbox}{max width=\textwidth}
\begin{tabular}{@{}l|cc|cc|cc|cc|cc|cc@{}}
\hline
 & \multicolumn{2}{c|}{\textbf{ImageNet}} & \multicolumn{2}{c|}{\textbf{-V2}} & \multicolumn{2}{c|}{\textbf{-S}} & \multicolumn{2}{c|}{\textbf{-R}} & \multicolumn{2}{c|}{\textbf{-A}} & \multicolumn{2}{c}{\textbf{SUN397}} \\ 
\cmidrule{2-13}
\textbf{Method} & MTA & HCA & MTA  & HCA & MTA & HCA & MTA & HCA & MTA & HCA & MTA & HCA \\ 
\hline
CoOp & 40.7 & 0.8 & 38.9 & 0.8 & 34.4 & 0.5 & \textbf{57.7} & 9.9 & 43.8 & 4.1 & 19.0 & 31.3 \\ 
\textbf{DMPCoOp} & \textbf{47.4} & \textbf{2.6} & \textbf{45.0} & \textbf{2.1} & \textbf{40.2} & \textbf{2.0} & \textbf{57.7} & \textbf{18.7} & \textbf{55.5} & \textbf{5.5} & \textbf{30.4} & \textbf{36.3} \\ 
\textbf{$\Delta$} 
& \textcolor{ForestGreen}{+6.7} 
& \textcolor{ForestGreen}{+1.8} 
& \textcolor{ForestGreen}{+6.1} 
& \textcolor{ForestGreen}{+1.3} 
& \textcolor{ForestGreen}{+5.8} 
& \textcolor{ForestGreen}{+1.5} 
& 0.0
& \textcolor{ForestGreen}{+8.8} 
& \textcolor{ForestGreen}{+11.7} 
& \textcolor{ForestGreen}{+1.3} 
& \textcolor{ForestGreen}{+11.4} 
& \textcolor{ForestGreen}{+5.0} \\
\hline
\end{tabular}
\end{adjustbox}

\end{figure*}

\textbf{Impact of Classifier-Free Guidance (CFG) scales.}
The performance of DMP with different guidance values is shown in Table~\ref{tab:supp_guidancecomparison}. The trends show that the average accuracy across all datasets decreases with increasing guidance scales.
\begin{figure*}[h]
\centering
\scriptsize
\captionof{table}{Comparison with different classifier free guidance scales. Base2new generalization per dataset: performance for All, Base, and New classes, and HM (Harmonic Mean).}
\label{tab:supp_guidancecomparison}
\begin{adjustbox}{max width=\textwidth}
\begin{tabular}{@{}l|cccc|cccc|cccc|cccc@{}}
\hline
 & \multicolumn{4}{c|}{\textbf{(a) Average (scale 4.5)}} & \multicolumn{4}{c|}{\textbf{(b) Average (scale 5.5)}} & \multicolumn{4}{c|}{\textbf{(c) Average (scale 6.5)}} & \multicolumn{4}{c}{\textbf{(d) Average (scale 7.5)}} \\ 
\cmidrule{2-17}
\textbf{Method} & All & Base & New & HM & All & Base & New & HM & All & Base & New & HM & All & Base & New & HM \\ 
\hline
\textbf{DMPCoOp} & {\textbf{70.1}} & {\textbf{80.3}} & {\textbf{73.4}} & {\textbf{76.5}} & {\textbf{69.6}} & {\textbf{80.3}} & {\textbf{70.6}} & {\textbf{76.3}} & {\textbf{69.6}} & {\textbf{79.3}} & {\textbf{70.4}} & {\textbf{75.0}} & {\textbf{68.1}} & {\textbf{79.4}} & {\textbf{70.1}} & {\textbf{75.0}} \\ 
\bottomrule
\end{tabular}
\end{adjustbox}
\end{figure*}

\textbf{Variational Autoencoder Reconstruction.} Table \ref{table:ablation_autoencoder} shows the reconstruction accuracy of the CoOp prompts for the trained Variational Autoencoder (VAE) across different datasets. The autoencoder reconstructs the prompts almost perfectly with only 0.1\% difference on average across all the datasets.

\begin{table*}[h]
    \centering
    \scriptsize
    \caption{Quantitative results of VAE reconstruction: Comparison against CoOp prompts across various datasets. The VAE reconstructs the CoOp prompts baseline almost perfectly with only 0.1\% difference on average.} 
    \label{table:ablation_autoencoder}
    \renewcommand{\arraystretch}{1.2} 
    \begin{adjustbox}{width=\textwidth,center}
    \begin{tabularx}{\textwidth}{l|c|XXXXXXXXXX|c}
        \hline
        \textbf{Model} & \rotatebox{60}{\textbf{ImageNet}} & \rotatebox{60}{\textbf{Oxford flowers}} & \rotatebox{60}{\textbf{Oxford pets}} & \rotatebox{60}{\textbf{Stanford cars}} & \rotatebox{60}{\textbf{Caltech 101}} & \rotatebox{60}{\textbf{Food101}} & \rotatebox{60}{\textbf{FGVC Aircraft}} & \rotatebox{60}{\textbf{SUN397}} & \rotatebox{60}{\textbf{DTD}} & \rotatebox{60}{\textbf{EuroSAT}} & \rotatebox{60}{\textbf{UCF101}} & \rotatebox{60}{\textbf{Average}} \\ 
        \hline
        CoOp & \textbf{68.0} & \textbf{72.6} & \textbf{90.1} & \textbf{68.7} & \textbf{94.8} & 84.9 & \textbf{25.1} & \textbf{67.6} & 51.9 & 58.9 & \textbf{66.2} & \textbf{68.1} \\ 
Autoencoder & \textbf{68.0} & 72.0 & 89.6 & 68.6 & 94.6 & \textbf{85.1} & 25.0 & 67.0 & \textbf{52.0} & \textbf{59.5} & 66.1 & 68.0 \\
\textbf{$\Delta$} 
& 0.0 
& \textcolor{Crimson}{-0.6} 
& \textcolor{Crimson}{-0.5} 
& \textcolor{Crimson}{-0.1} 
& \textcolor{Crimson}{-0.2} 
& \textcolor{ForestGreen}{+0.2} 
& \textcolor{Crimson}{-0.1} 
& \textcolor{Crimson}{-0.6} 
& \textcolor{ForestGreen}{+0.1} 
& \textcolor{ForestGreen}{+0.6} 
& \textcolor{Crimson}{-0.1} 
& \textcolor{Crimson}{-0.1} \\
        \hline
    \end{tabularx}
    \end{adjustbox}

\end{table*}

\begin{table*}[t]
    \centering
    \scriptsize
    \caption{Ablation study on Cross-dataset generalization of DMPCoOp Imagenet prompts: Comparison of classnames against dataset names as prompt inputs to the DMPCoOp model.} 
    \label{table:ablation_coopcomparison}
    \renewcommand{\arraystretch}{1.2} 
    \begin{adjustbox}{width=\textwidth,center}
    \begin{tabularx}{\textwidth}{l|c|XXXXXXXXXX|c}
        \hline
        \textbf{Model} & \rotatebox{60}{\textbf{ImageNet}} & \rotatebox{60}{\textbf{Oxford flowers}} & \rotatebox{60}{\textbf{Oxford pets}} & \rotatebox{60}{\textbf{Stanford cars}} & \rotatebox{60}{\textbf{Caltech 101}} & \rotatebox{60}{\textbf{Food101}} & \rotatebox{60}{\textbf{FGVC Aircraft}} & \rotatebox{60}{\textbf{SUN397}} & \rotatebox{60}{\textbf{DTD}} & \rotatebox{60}{\textbf{EuroSAT}} & \rotatebox{60}{\textbf{UCF101}} & \rotatebox{60}{\textbf{Average}} \\ 
        \hline
        CoOp & 68.0 & 66.5 & 88.5 & 61.9 & \textbf{92.7} & 84.8 & 15.2 & 60.7 & 40.9 & \textbf{46.9} & 65.3 & 62.9 \\ 
DMPCoOp (Dataset Names) & 67.8 & 65.2 & 88.0 & \textbf{63.4} & \textbf{92.7} & 83.9 & 16.8 & 62.6 & \textbf{41.0} & 46.3 & 66.6 & 63.2 \\ 
\textbf{$\Delta$ (Dataset)} 
& \textcolor{Crimson}{-0.2} 
& \textcolor{Crimson}{-1.3} 
& \textcolor{Crimson}{-0.5} 
& \textcolor{ForestGreen}{+1.5} 
& 0.0 
& \textcolor{Crimson}{-0.9} 
& \textcolor{ForestGreen}{+1.6} 
& \textcolor{ForestGreen}{+1.9} 
& \textcolor{ForestGreen}{+0.1} 
& \textcolor{Crimson}{-0.6} 
& \textcolor{ForestGreen}{+1.3} 
& \textcolor{ForestGreen}{+0.3} \\
DMPCoOp (Class Names) & \textbf{68.7} & \textbf{69.0} & \textbf{89.2} & 62.4 & 91.4 & \textbf{85.7} & \textbf{20.5} & \textbf{62.9} & 40.4 & 46.2 & \textbf{66.7} & \textbf{63.9} \\
\textbf{$\Delta$ (Class)} 
& \textcolor{ForestGreen}{+0.7} 
& \textcolor{ForestGreen}{+2.5} 
& \textcolor{ForestGreen}{+0.7} 
& \textcolor{ForestGreen}{+0.5} 
& \textcolor{Crimson}{-1.3} 
& \textcolor{ForestGreen}{+0.9} 
& \textcolor{ForestGreen}{+5.3} 
& \textcolor{ForestGreen}{+2.2} 
& \textcolor{Crimson}{-0.5} 
& \textcolor{Crimson}{-0.7} 
& \textcolor{ForestGreen}{+1.4} 
& \textcolor{ForestGreen}{+1.0} \\
        \hline
    \end{tabularx}
    \end{adjustbox}
    
\end{table*}

\subsection{Multi-task DMP}\label{app:dmpmulti}

\textit{DMPMulti} is a multi-task DMP model trained to synthesize prompts for personalized subjects and prompt sliders. Both tasks are unified within a single framework through the shared CLIP-ViT L/16 text encoder.

\noindent\textbf{Prompt Slider Repository.}
Prompt Sliders  is a technique to learn text prompts in the CLIP text embedding space that allow control of particular attributes of a designated concept. The prompt is learned through the optimization 
\begin{equation}\label{eq4}
S^* = \arg\min_S \mathbb{E}_{x \sim \mathcal{E}(x), y, \epsilon \sim \mathcal{N}(0,1), t} \left\| \epsilon_t - \epsilon^d_{\theta}(x_t, \tau_{\theta}(y, S), t) \right\|_2^2
\end{equation}
where $x_t$ is an image, and $y$ is any additional prompt text, such as ``a photo of a".
Both \(\tau_{\theta}\) and \(\epsilon_{\theta}\) are fixed during the optimization.  Given a target concept $c_t$, a prompt slider $S^*$  is learned to encourage the distribution of images of \(c_t\) to exhibit more positive attributes \(c^+\) and fewer negative attributes \(c^-\). This is implemented by replacing $\epsilon_t$  with
\begin{equation}\label{eq5}
\epsilon_t(\alpha) = \epsilon^d_{\theta}(x_t, \tau_{\theta}(y(c_t)), t) +
\alpha \eta \sum_{p \in P}
(\epsilon^d_{\theta}(x_t, \tau_{\theta}(y(c^+, p)), t) - \epsilon^d_{\theta}(x_t, \tau_{\theta}(y(c^-, p)), t))
\end{equation}
and $S$ by $\alpha S$ in~(\ref{eq4}), where $\eta$ is a guidance scale, $\alpha$ a scaling parameter, $y(c_t)$ the concept name, and $P$ a set of concepts that the attribute manipulation should preserve (for example,
race or gender). The positive \(c^+\), and negative \(c^-\) attributes are sampled from a template predefined for concept $c_t$. 
To create a slider repository ${\cal R}_S$, we trained prompt sliders for 20 different concepts using the SD-XL model and $3,000$ backpropagation steps. For each concept $c$, we save 40 prompts $S(c)$ (from the last $200$ steps in intervals of $5$) to obtain a training tensor $\mathcal{R}_S \in \mathbb{R}^{20 \times 40 \times d}$. 

\noindent\textbf{Prompt Identity Repository.} We use a random subset of 20 identities from the 3000 celebrity faces in the DMPVariation prompt repository to create a prompt identity repository ${\cal R}_I$ with each identity containing 40 prompts obtained from the last $200$ steps in intervals of $5$ to obtain a training tensor ${\cal R}_I \in \mathbb{R}^{20 \times 40 \times d}$.

\noindent\textbf{Training.}
The DMPMulti model was trained directly in prompt space $\cal P$, i.e. without autoencoder, using the full set of prompts from both repositories, i.e. ${\cal R} = {\cal R}_S \cup {\cal R}_I$. For each slider concept $c$ (e.g., age, smiling etc.) we use the concept name as the text condition for DMPMulti. For identities, we use ``identity-$c$" as the text condition where $c=1,\ldots, 20$ is associated with each identity $c$ in ${\cal R}$. 

\noindent\textbf{Inference.} At inference, DMPMulti can be conditioned with the slider concept name, to generate prompt sliders, or with ``identity-$c$" to generate identity prompts. Novel identities can be sampled by specifying a new identity "identity-$c$" with $c > 20$. Slider and identity prompts can then be fed to the downstream model in isolation or together.

Table~\ref{tab:comp_ids} compares Identity prompts from DMPMulti with Stylus and Textual Inversion across face recognition accuracy, image-to-image similarity, and prompt fidelity. DMPMulti achieves higher identity scores, maintains prompt compliance comparable to TI/Stylus, and shows lower similarity to training images, indicating stronger generalization. In contrast, TI tends to overfit, consistent with prior findings~\citep{yuan2023celebbasis}. The last row of Table~\ref{tab:comp_ids} shows an ablation study of using fewer samples (20 vs 40) per identity for training DMP. Notably, DMPMulti (20) achieves similar identity fidelity and even higher prompt compliance ($\uparrow$17\%) than the 40-sample variant.

\begin{figure*}[t]
\centering
\begin{minipage}[t]{\textwidth}
    \centering
    \includegraphics[width=\columnwidth, trim= 120 100 120 90, clip]{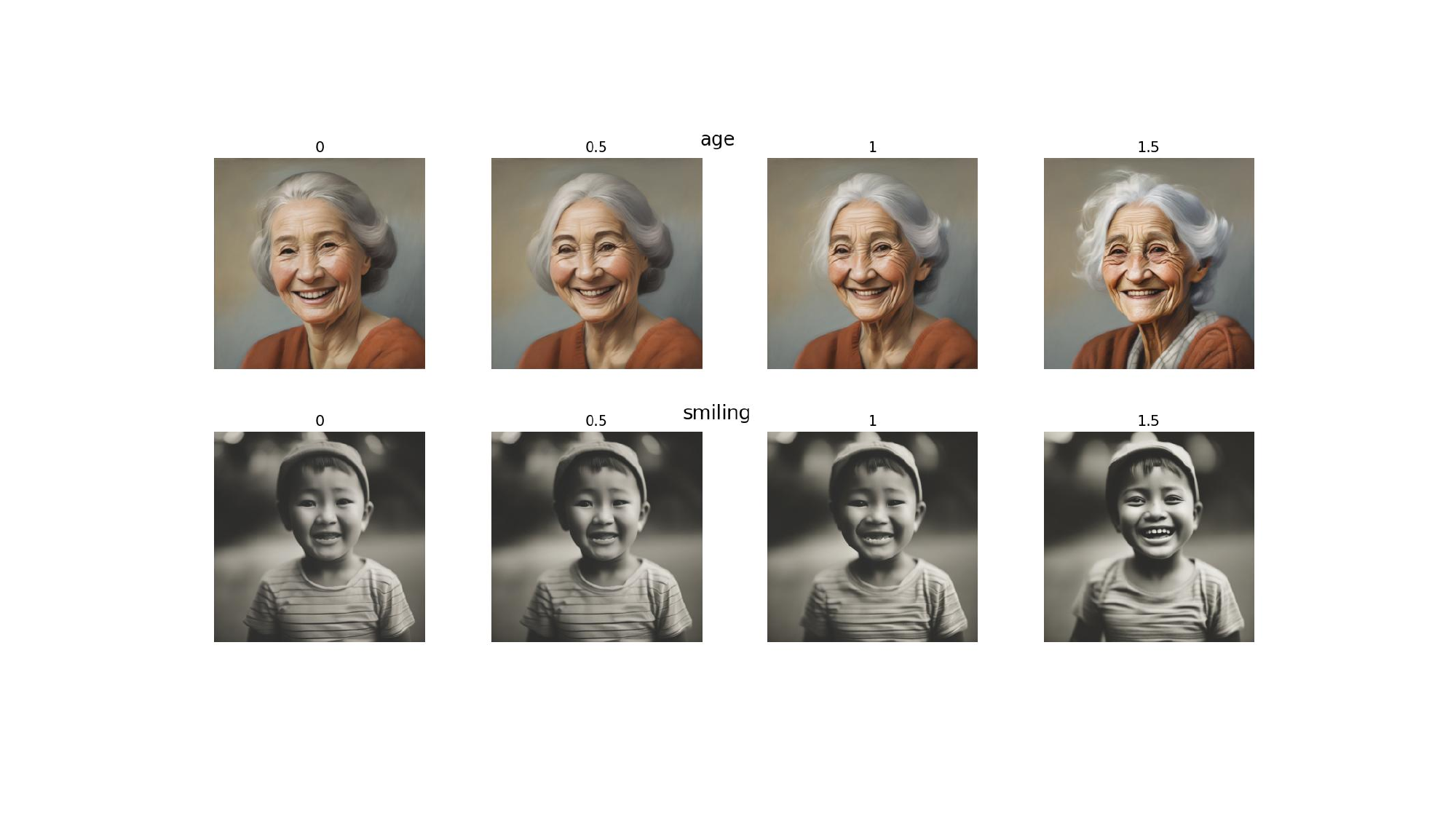}
\caption{\label{fig:qual_syn_sliders_age} \textbf{Prompt Sliders} images synthesized with sliders sampled by DMPSlider when prompted for age and smiling. The prompt used for the SD-XL model is \textit{"A photo of a beautiful man"}.}
\end{minipage}
\end{figure*}
\begin{figure*}[t]
\centering
\includegraphics[width=\textwidth, trim= 0 180 0 180, clip]{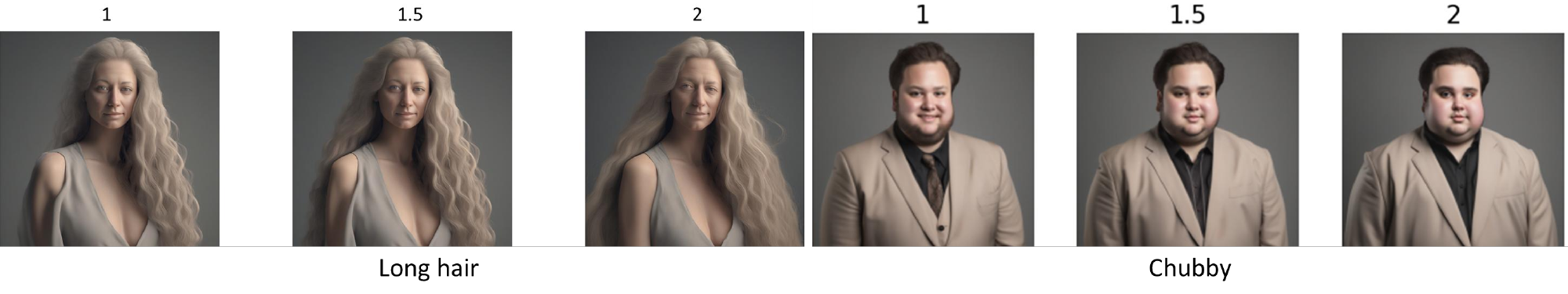}
\caption{\label{fig:qual_syn_sliders} \textbf{Qualitative results of DMPSlider prompts} depicting the concepts ``long hair" and ``chubby". The prompts used for the SD-XL model for the images shown from left to right are as follows. \textit{"A closeup photo of a person", "Professional headshot of a person"}.}
\end{figure*}

\begin{figure*}[h]
\centering
\includegraphics[width=\columnwidth, trim= 150 80 150 80, clip]{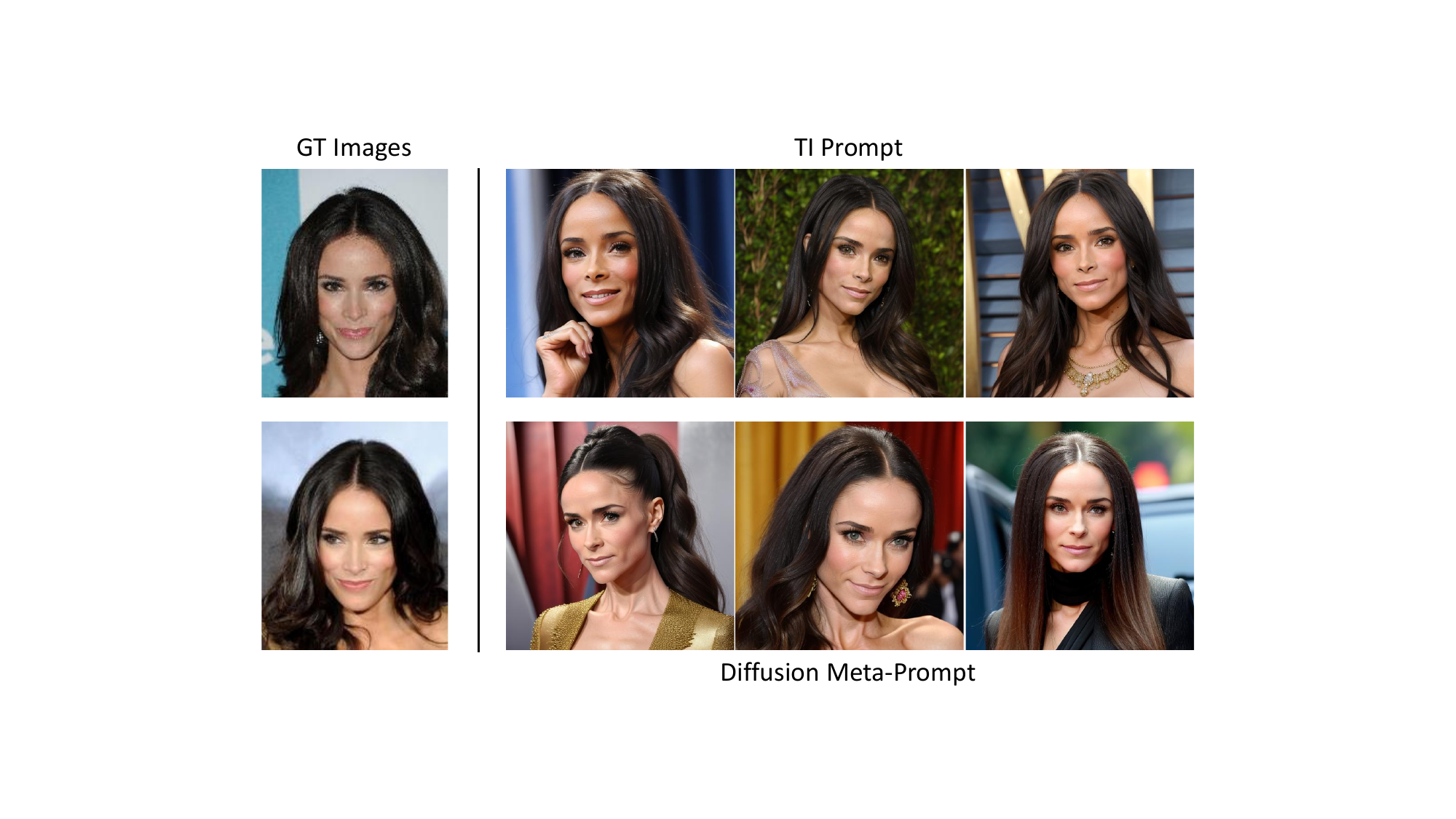}
\includegraphics[width=\columnwidth, trim= 150 80 150 80, clip]{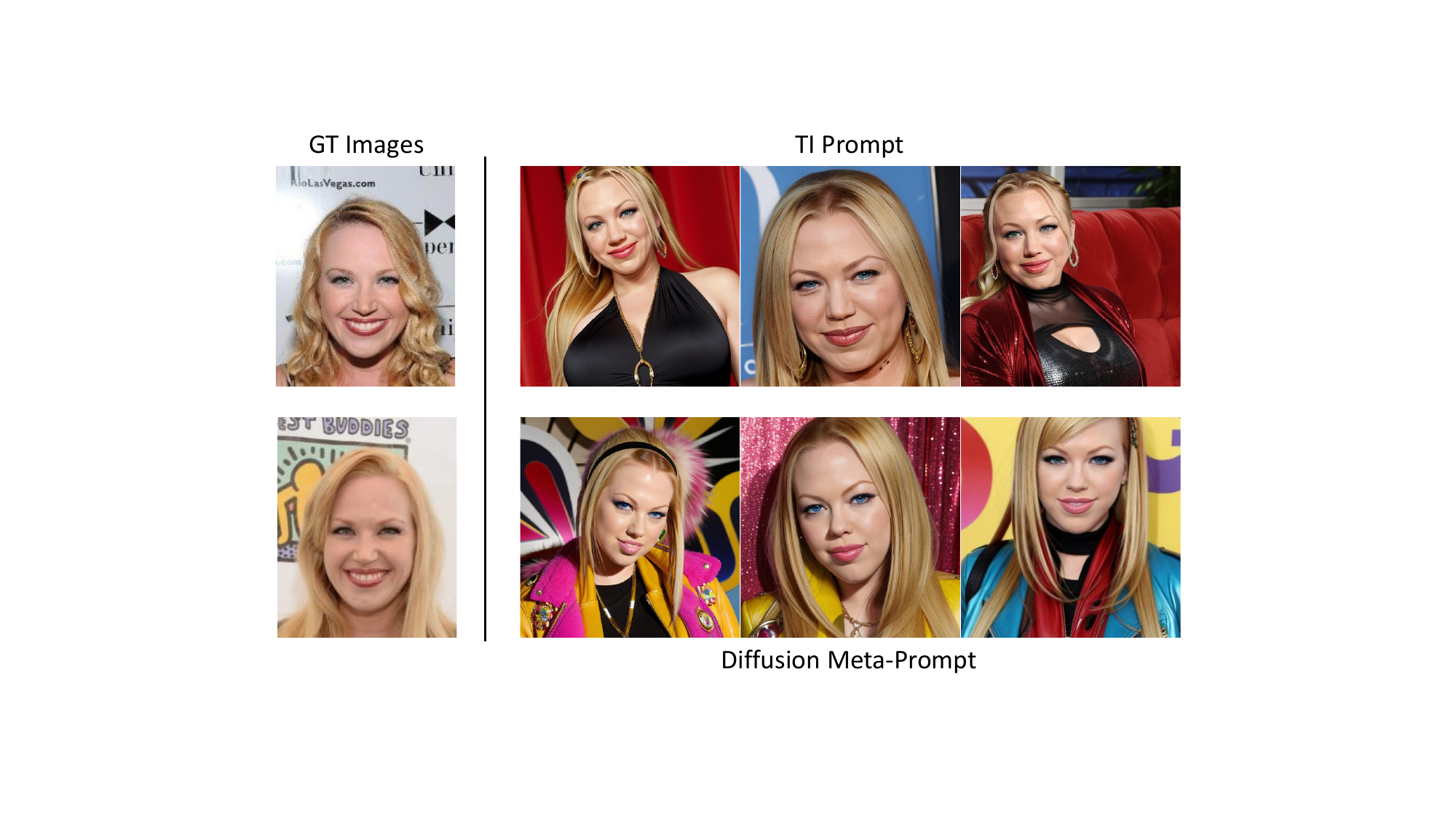}
\caption{\label{fig:qual_syn_id} \textbf{Qualitative comparison} of image synthesis with TI and DMPMulti prompts. Left: groundtruth images for Identity-38 and Identity-96 in the training set. Right: images synthesized by SD for the original TI prompts (first row) and prompts sampled from the DMPMulti model (second row):}
\end{figure*}
Figure~\ref{fig:qual_syn_sliders} shows qualitative results of slider prompts generated by DMP for the attributes "long hair" and "chubby".\\
Figure~\ref{fig:qual_syn_vs_prompt_sliders} presents a qualitative comparison of images generated by Concept Sliders (using the author-provided models), Prompt Sliders, and DMP. The results show that Concept Sliders struggle to induce the intended attributes, even at higher scales, due to their sensitivity to training hyperparameters, which requires careful tuning as noted in \citep{eccvsridhar24}. We observe that DMP-generated prompts produce images that are qualitatively similar to those from Prompt Sliders. However, the original Prompt Sliders method has limitations in maintaining subject identity at higher scales, as discussed in \citep{eccvsridhar24}. In contrast, DMP—despite being trained with the Prompt Sliders embeddings—demonstrates greater robustness, effectively preserving subject identity even at relatively higher scales. These result corroborate the results observed in Table 2 of the paper.

Figure \ref{fig:qual_syn_neg_sliders} illustrates the results of pure negative prompts,  where the sliders yield images with attributes opposite to the specified concepts, such as shorter hair or a neutral (non-smiling) expression.

\begin{figure}[h]
\centering
\includegraphics[width=\columnwidth, trim= 150 160 150 180, clip]{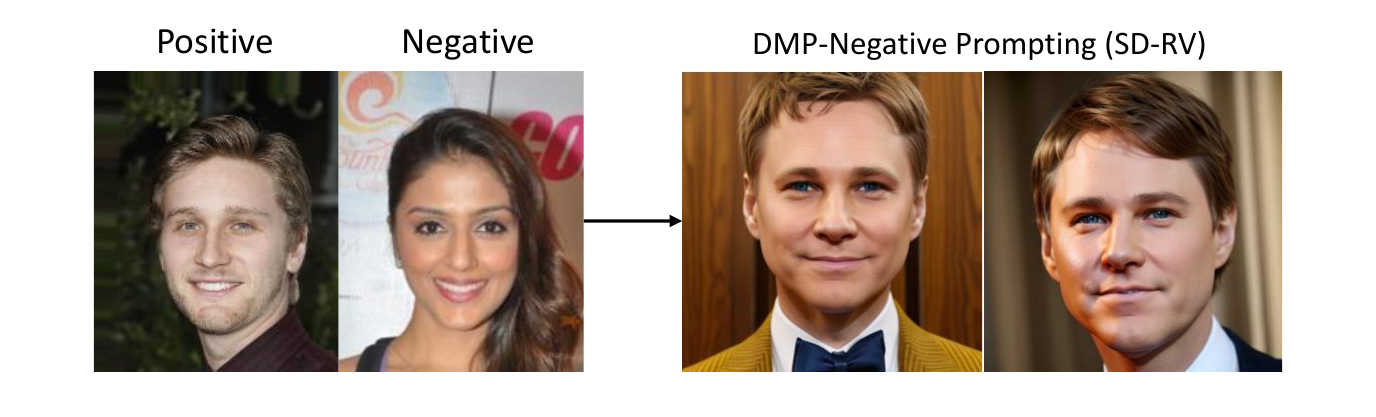}
\caption{\label{fig:qual_syn_id_neg} \textbf{Negative prompting} samples from the SD Realistic Vision model when prompted by the DMPMulti, itself prompted with id-21 (leftmost) as positive and id-24 (second from left) as negative prompt. The generated identity has features opposing to id-24 (chubby cheeks, small eyes, wide nose, etc.)}
\end{figure}

\begin{figure}[t]
\centering
\includegraphics[width=\columnwidth, trim= 200 0 200 0, clip]{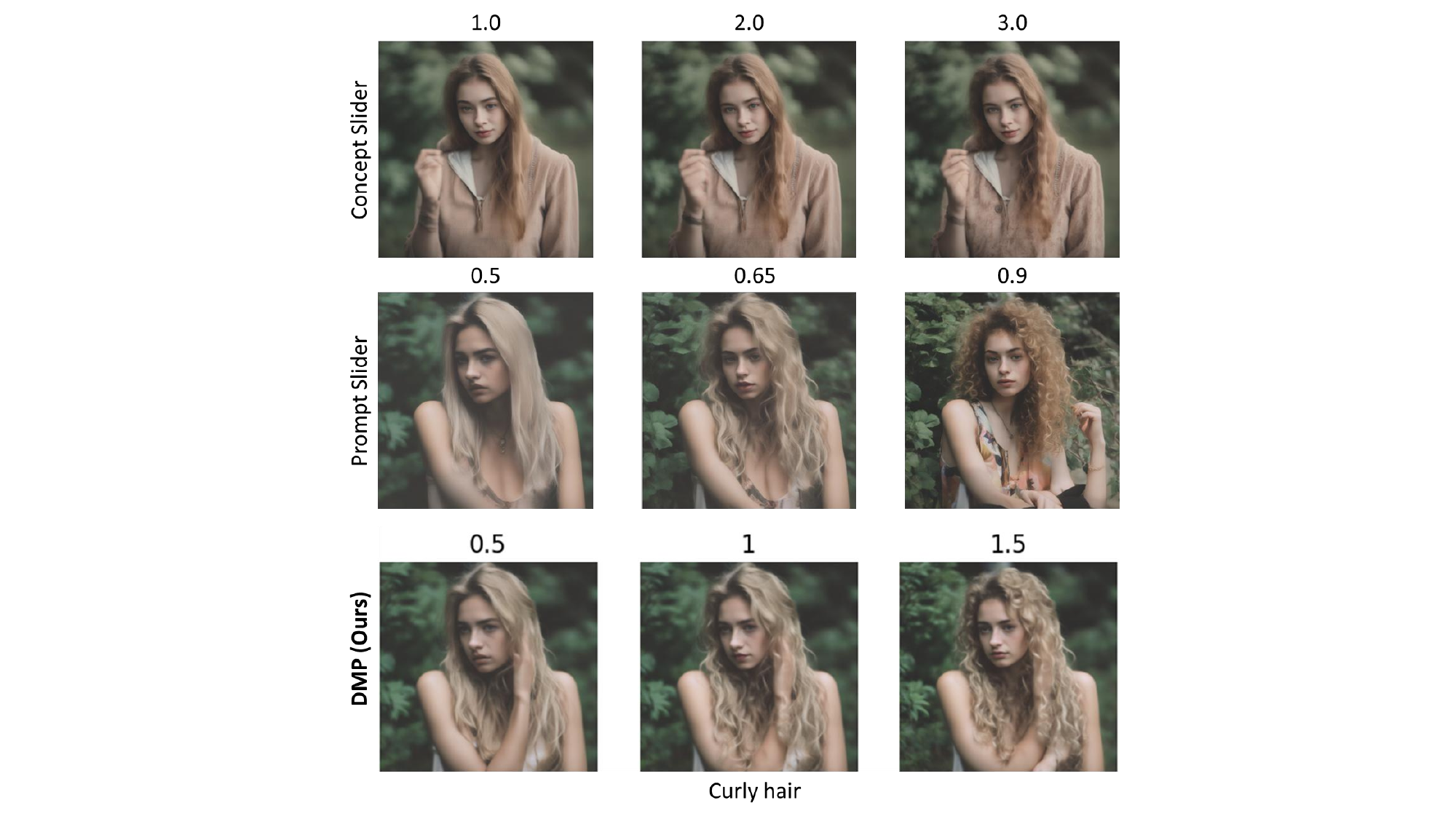}
\caption{\label{fig:qual_syn_vs_prompt_sliders} \textbf{Qualitative comparison} of images synthesized with baseline prompt sliders and concept sliders against DMPSlider. The prompt used for the SD-XL model is \textit{``A photo of a girl"} and the slider trained to control the concept \textbf{``curlyhair"}.}
\end{figure}

\begin{figure}[t]
\centering
\includegraphics[width=\columnwidth, trim= 160 20 150 40, clip]{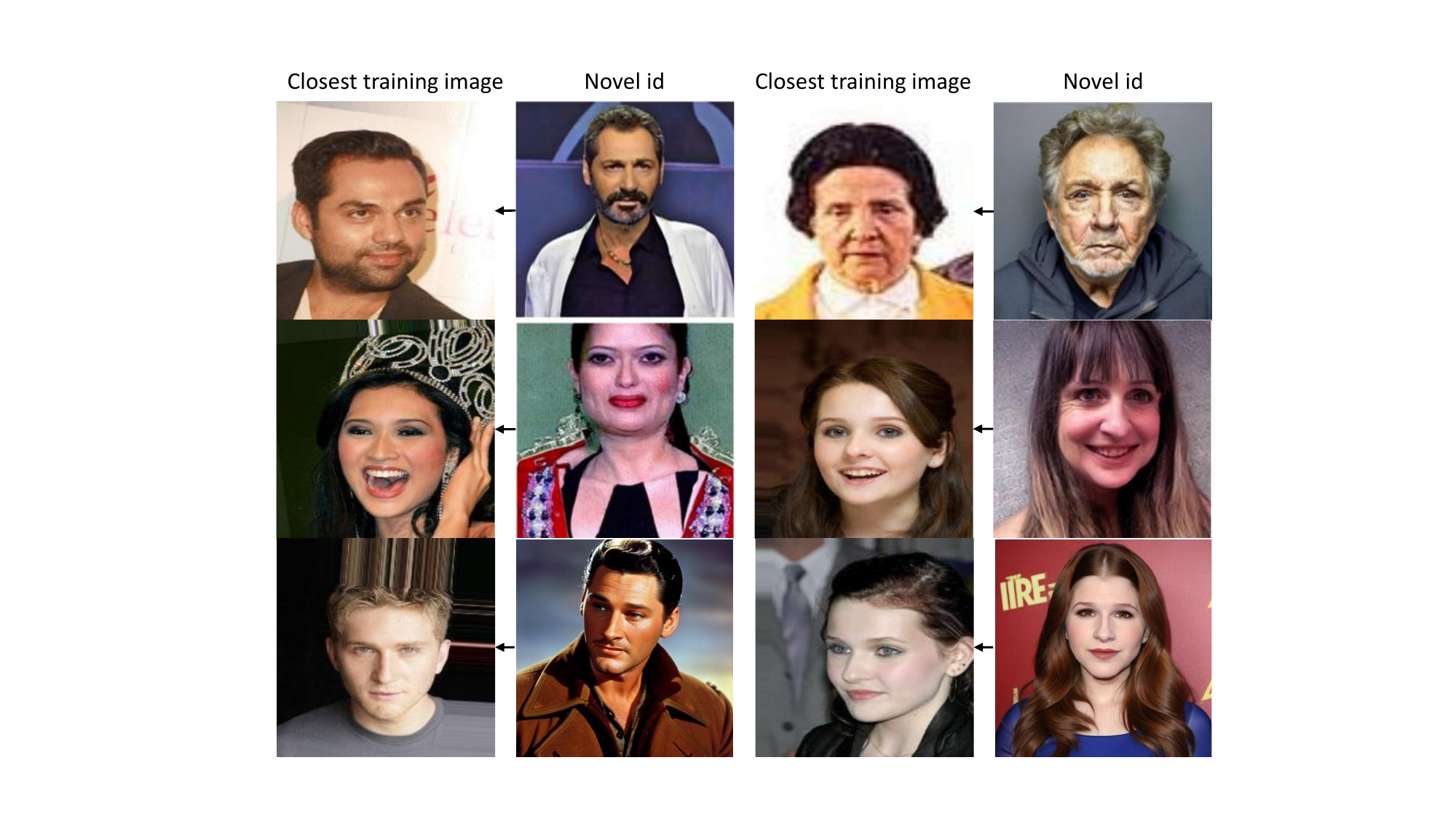}
\caption{\label{fig:supp_novel_closest_training_images}\textbf{Additional Qualitative results: }Nearest neighbors in the training set shown on the left for novel ids sampled by DMPMulti model to the right.}
\end{figure}

\begin{figure}[t]
\centering
\includegraphics[width=\columnwidth, trim= 150 120 150 80, clip]{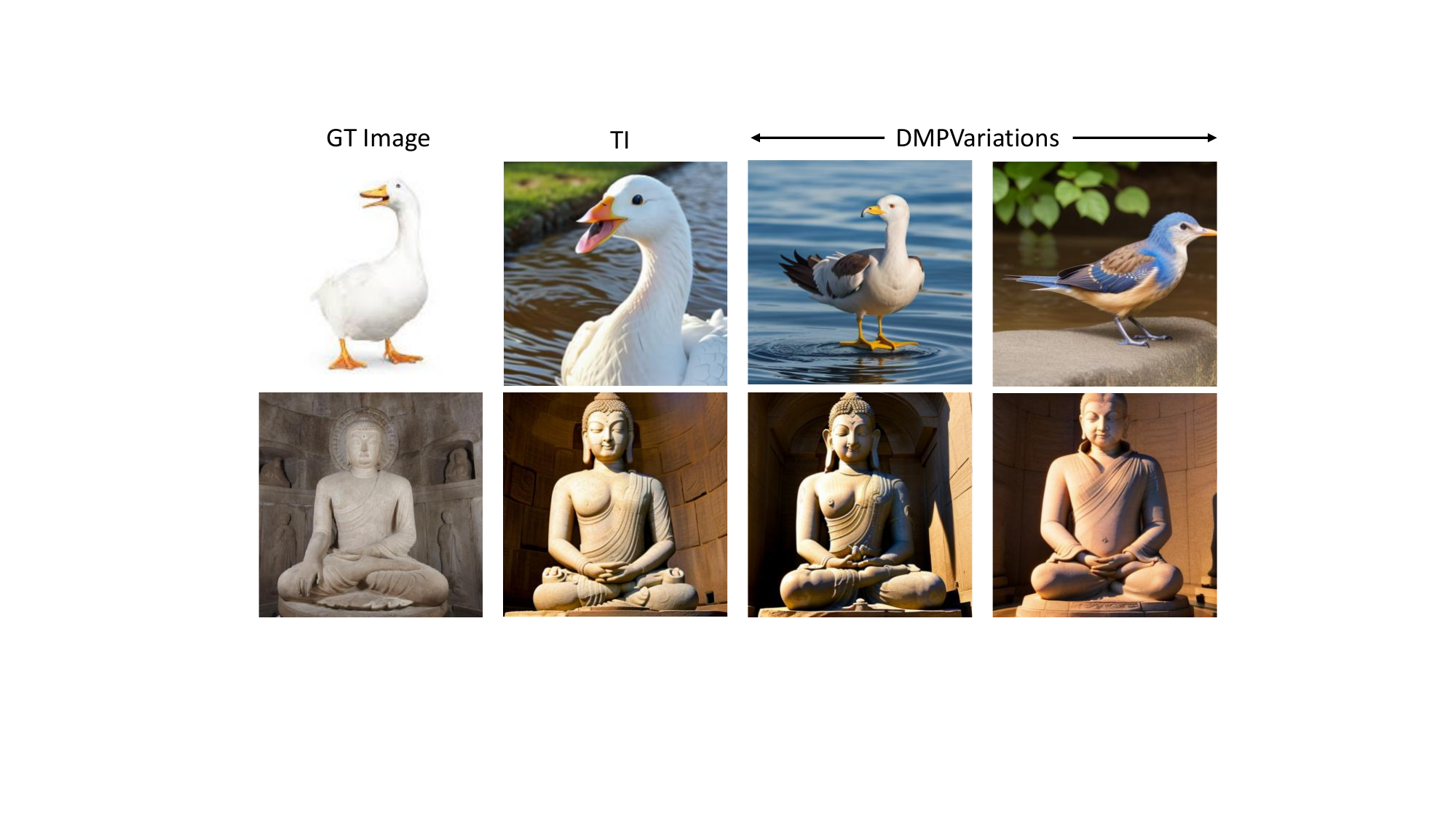}
\caption{\label{fig:mdpvariations-generalization} \textbf{Generalization of DMPVariation model}. Left: groundtruth images. Second: images generated by SD for the original TI prompts and the last two columns are the images for variation prompts sampled by DMP. See Fig.~\ref{fig:supp-mdpvariations-generalization} for additional results.}
\end{figure}

\begin{figure}[t]
\centering
\includegraphics[width=\columnwidth, trim= 200 0 200 0, clip]{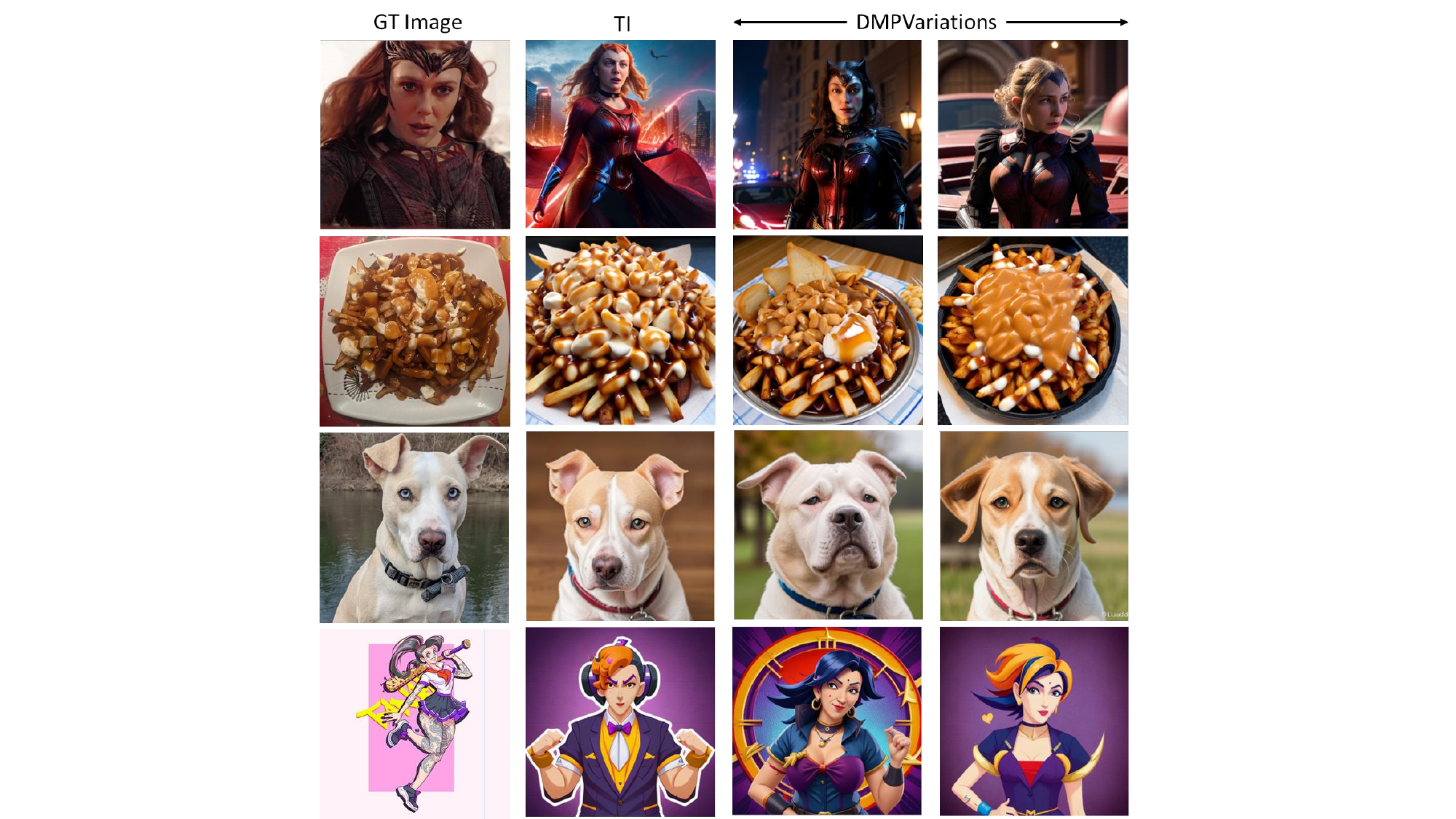}
\caption{\label{fig:supp-mdpvariations-generalization} \textbf{Additional Qualitative results of generalization of DMPVariation model}. Left: groundtruth images. Second: images generated by SD for the original TI prompts and the last two columns are the images for variation prompts sampled by DMP.}
\end{figure}

\begin{figure}[h]
\centering
\includegraphics[width=\columnwidth, trim= 100 100 50 80, clip]{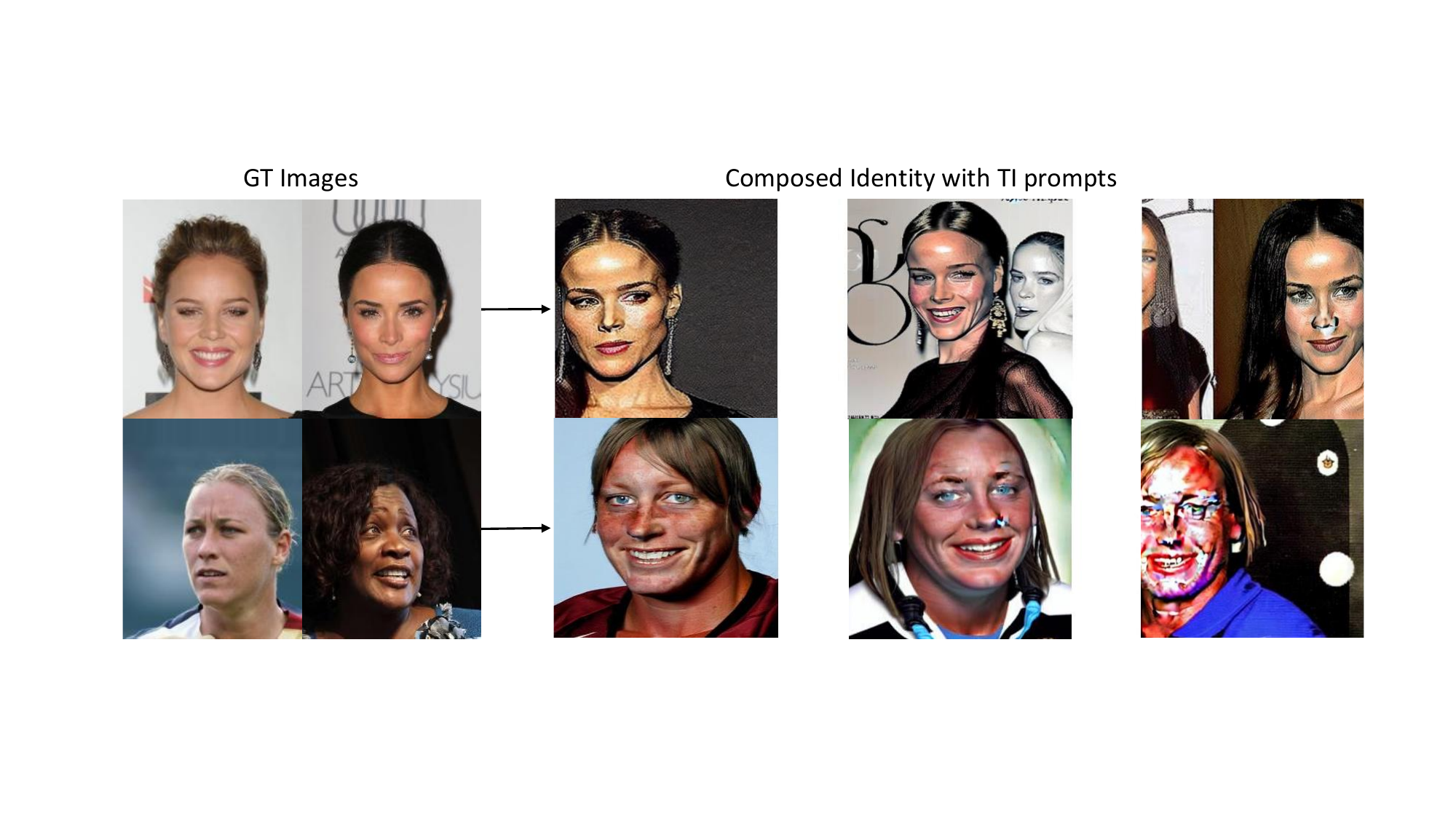}
\caption{\label{fig:supp_ablate_ti_id_compose} \textbf{Ablation study on Identity composition with Textual Inversion:} images generated with TI prompts by composing the Stable Diffusion outputs for the composition of the two identities shown on  the left for each row. TI composition produces distorted identities or repeats the same identities.}
\end{figure}

\begin{figure*}[t]
\centering
\includegraphics[width=\textwidth, trim= 0 200 0 200, clip]{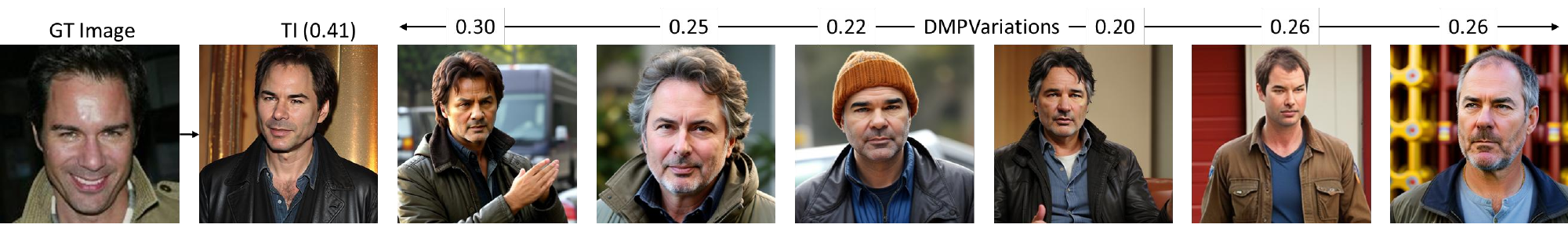}
\caption{\label{fig:teaser_variations} \textbf{DMP Variations of Textual Inversion (TI) Prompts}: The leftmost image is the real groundtruth, the second is generated by TI, and the rest are DMP variations (each image represents a new identity) conditioned on the TI embedding. All images are generated with a fixed seed to the stable diffusion model. The FaceID similarity to the groundtruth image is listed on top of each image. Lower is better for diverse variations.}
\end{figure*}

\subsection{DMP for Variations}\label{app:dmpvariation}

\textbf{Generalization.} 
Figure~\ref{fig:mdpvariations-generalization} shows images synthesized for the variation prompts generated by DMPVariation model, for a common seed. Note that the variation prompts are conditioned by the Textual inversion embedding for the target concept, which is illustrated by a GT image in the figure. The figure shows that DMPVariation produces successful variation prompts for general objects as diverse as birds and statues {\it despite being trained only on CelebA-face identities\/}.

\textbf{Qualitative Results.} 
Figure~\ref{fig:supp-mdpvariations-generalization} shows additional qualitative results for variations generated by DMPVariation model.

Figure~\ref{fig:teaser_variations} shows the synthesized prompts that produce variations of a person, conditioned on the textual inversion embedding of the person. 
Note that, while all images are synthesized with the same SDXL seed, they exhibit a diversity of background scenes, hair patterns, clothing, etc. This is an additional benefit of the natural prompt variability of DMP-based personalization: to increase the diversity of the synthesized images. 

Figure~\ref{fig:supp_comparison_mdp_ti_images} presents qualitative results comparing DMP and TI in generating personalized subject images across diverse contexts. The variation prompts produced by DMP demonstrate greater robustness and generalization, effectively adapting to different contexts while preserving subject identity. In contrast, Textual Inversion tends to overfit to the subject, leading to poor generalization. For instance, it completely fails to generate correct images for the Buddha statue and succeeds in only a single scenario for the duck and dog subjects. Table~\ref{tab:hps_topk} shows the detailed HPSv2 scores for the results presented in Table~\ref{tab:dmpvariation_faceid} of the paper.

\textbf{Identity Composition.}  
Figure \ref{fig:compose_mdpvariation_scales} demonstrates results of combining two celebrities, displayed on the left, using DMPVariation  model to synthesize identity prompts with Eq. 8. The synthesized identity clearly incorporates prominent features from both original faces, such as the nose and chin, resulting in a cohesive blend of attributes.

\noindent\textbf{Subject Composition.}  
Figure \ref{fig:qual_syn_id_compose} illustrates the ability of DMPVariation to generate prompts for combined concepts. In this examples, the prompts elicit the downstream model to produce images that combine the two subjects displayed on the left. This is done by sampling subject prompts using the DMPVariation model and Eq. (\ref{eq:7}). These are then fed to the SDXL model to produce the images on the right, for increasing guidance scale $\eta$. The synthesized subject clearly incorporates prominent features from both, such as the nose and chin, resulting in a cohesive blend of attributes.

\begin{figure*}[t]
\centering
\includegraphics[width=\textwidth, trim= 100 0 100 0, clip]{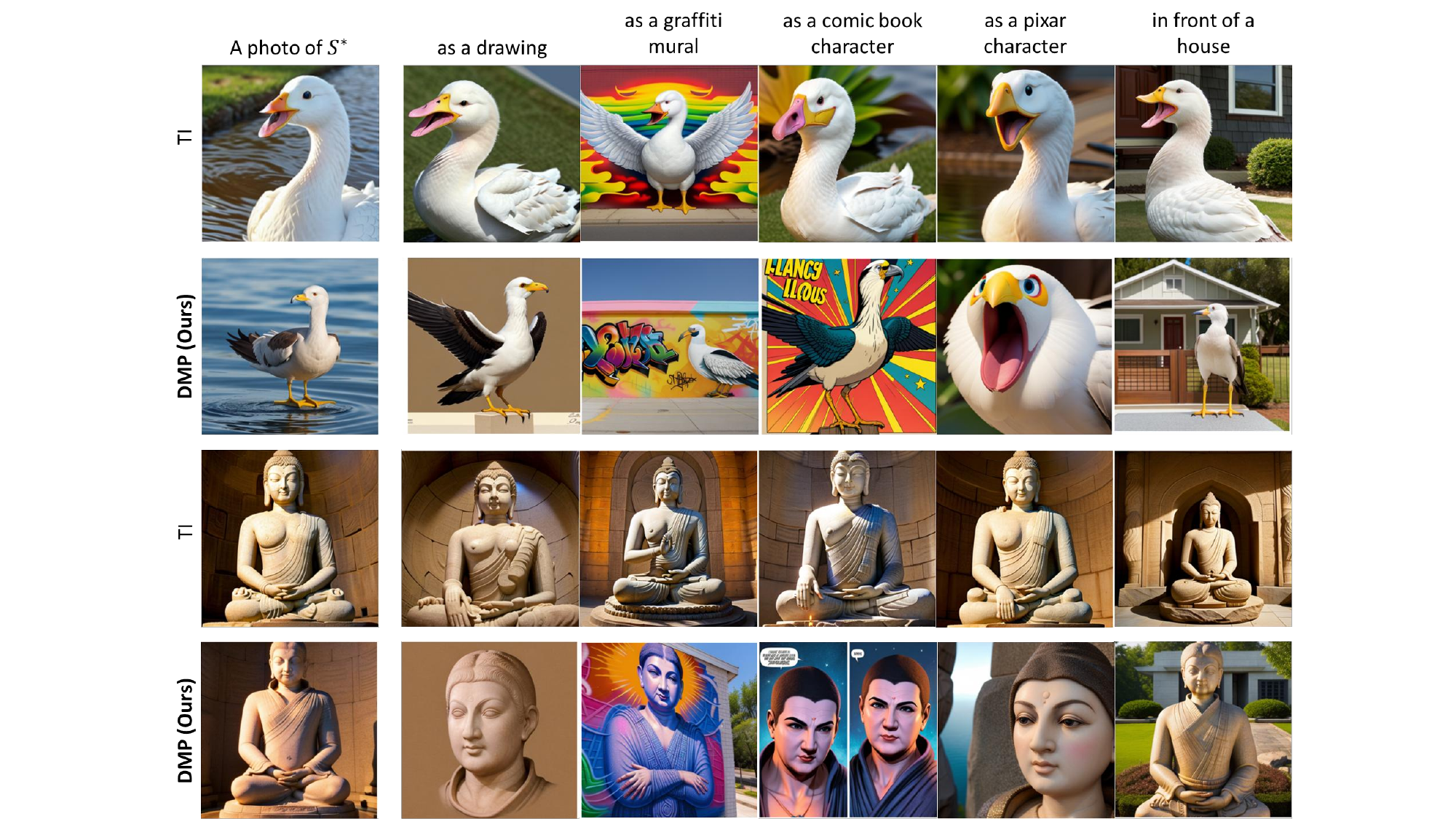}
\caption{\label{fig:supp_comparison_mdp_ti_images} \textbf{Qualitative results} of images generated from prompts synthesized by DMPVariation. Our Diffusion Meta-Prompts are more robust and less overfitted than the baseline textual inversion which fails to generalize. }
\end{figure*}

\begin{figure*}[t]
\centering
\includegraphics[width=\textwidth, trim= 20 232 0 200, clip]{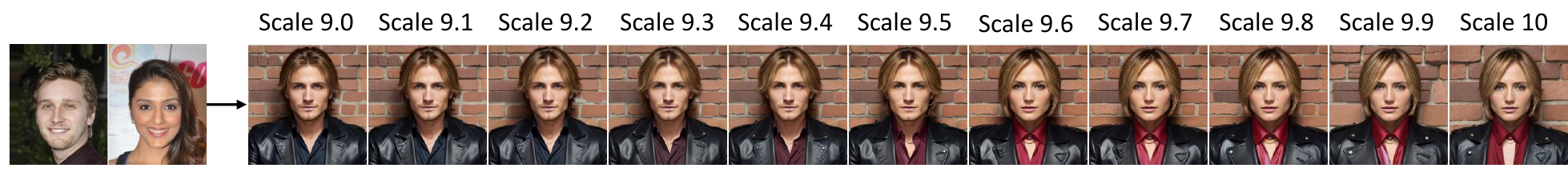}
\caption{\label{fig:compose_mdpvariation_scales} \textbf{Identity composition:} images generated with DMPVariation  model prompts (right)  for increasing guidance scales from 9 to 10 for the composition of the two identities shown on  the left.}
\end{figure*}

\begin{figure*}[t]
\centering
\begin{minipage}[t]{0.44\textwidth}
    \centering
\includegraphics[width=\linewidth, trim= 120 260 120 80, clip]{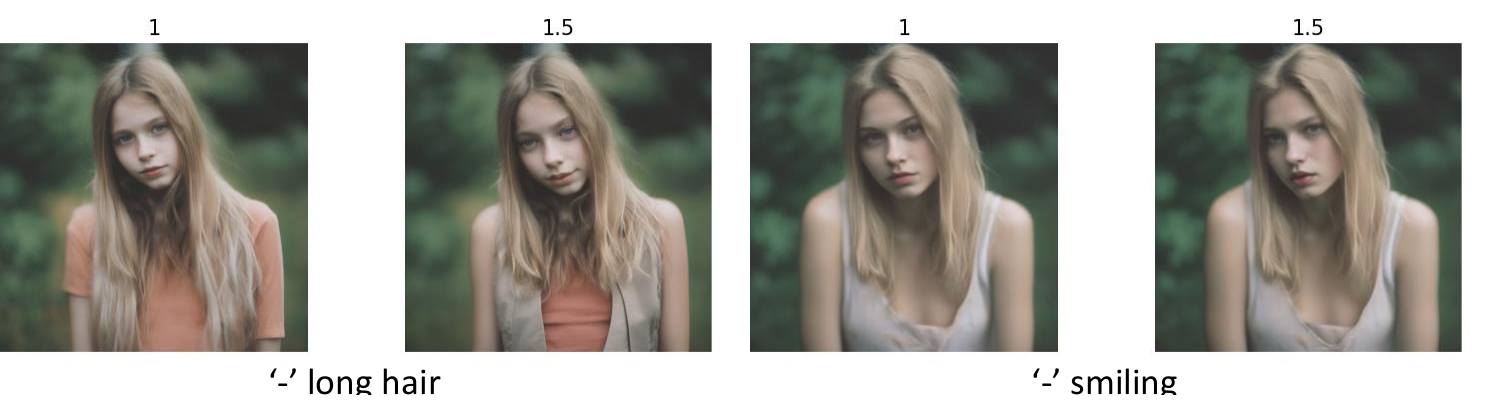}
\caption{\label{fig:qual_syn_neg_sliders} \textbf{Negative Prompt Sliders} images synthesized with sliders sampled by DMPSlider when  negatively prompted for different concepts. The prompt used for the SD-XL model is \textit{"A photo of a girl"}.}
\end{minipage}
\hfill
\begin{minipage}[t]{0.54\textwidth}
    \centering
    \includegraphics[width=\linewidth, trim=100 180 100 180, clip]{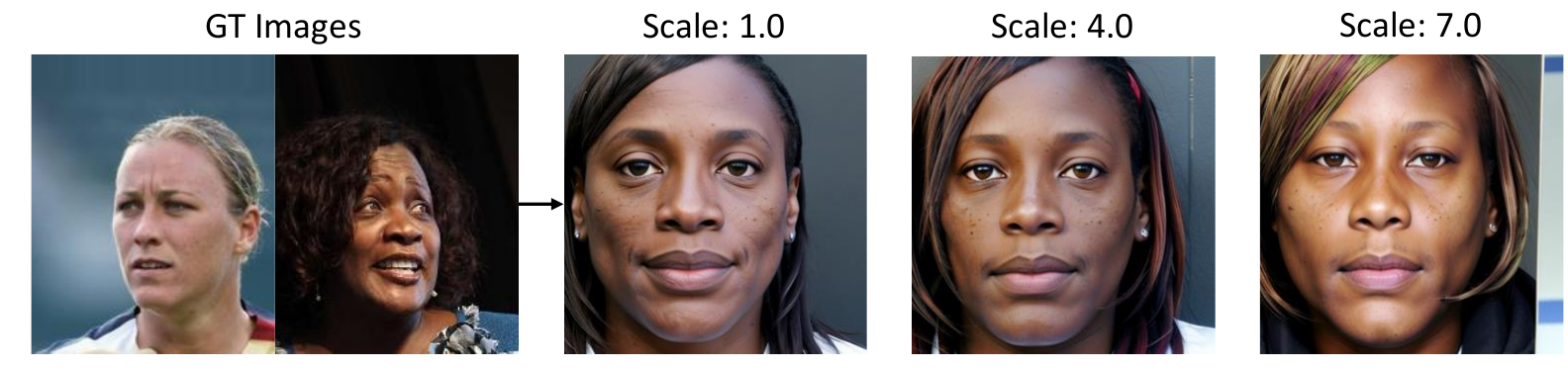}
    \caption{\textbf{Subject composition:} images generated with DMPVariation model prompts (right) for the composition of the two identities shown on the left with the guidance scale denoted at the top. See Fig.~\ref{fig:compose_mdpvariation_scales} for additional results at finer scale.}
    \label{fig:qual_syn_id_compose}
\end{minipage}

\end{figure*}

\textbf{Interpreting Variation Prompts.}
We computed the top-5 nearest-neighbor tokens in the CLIP vocabulary for the prompts sampled by DMPVariation model conditioned on a Textual Inversion embedding of a subject. For a random subject not in the training dataset, the TI prompt embedding returns \texttt{<w>karanjohar</w>}, \texttt{<w>conclude</w>}, \texttt{<w>leaked</w>}, \texttt{<w>prohibition</w>}, \texttt{<w>vijaysethu</w>}. The DMPVariation prompt returns \texttt{<w>karanjohar</w>}, \texttt{<w>pandoramusic</w>}, \texttt{episo}, \texttt{<w>leaked</w>}, \texttt{<w>refriger</w>}. It overlaps with two words out of five showing that the prompt is indeed a variation of the conditioned Textual Inversion prompt. 
Extending this analysis over 66 unseen identities, we found an average of \textbf{2.7 common words in the top-5 and 5.6 in the top-10} nearest-neighbor tokens, demonstrating that the DMP model effectively generalizes while retaining some of the subject-specific characteristics.

\begin{table*}[h]
\centering
\caption{\textbf{Prompts used for evaluating generalization.} Prompts were designed to explore style and concept variations of the subject \texttt{sks} and borrowed from DreamBooth~\citep{ruiz2022dreambooth}.}
\label{tab:prompt_variations_eval}
\begin{adjustbox}{max width=\textwidth}
\begin{tabular}{llll}
\toprule
\multicolumn{4}{c}{\textbf{Prompts}} \\
\midrule
a sks on the beach              & sks flower arrangement         & sks stained glass window    & sks as a witcher \\
A photo of two sks on a boat    & sks Funko Pop                  & sks latte art               & A cubism painting of sks person \\
Manga drawing of sks            & Pointillism painting of sks    & Ukiyo-e painting of sks     & A sks as a knight in plate armor \\
sks as a knight in plate        & Banksy art of sks              & sks piloting a fighter jet  & Greek sculpture of sks \\
Fauvism painting of sks         & Cave mural depicting sks       & sks by Andy Warhol          & sks in the style of Archer \\
Colorful graffiti of sks        & sks as Ziggy Stardust          & sks in a comic book         & Watercolor painting of sks \\
a sand sculpture of sks         & sks in a Santa hat             & sks as a wizard             & a photo of sks\\
\bottomrule
\end{tabular}
\end{adjustbox}
\end{table*}

\subsection{Ablation Studies}\label{sec:ablation_appendix}

The following section lists the setup for the ablation study on sensitivity of DMP to class-name conditions in Table~\ref{tab:replaced-names-steering-comparison}. Table~\ref{tab:class-name-replacement-conditions} shows the corresponding class names provided to DMP at each replacement level.

\paragraph{Original 0\% Condition.}

The original Caltech101 condition is:

\begin{quote}
\texttt{accordion; bonsai; chandelier; euphonium; menorah;}\\
\texttt{saxophone; stapler; yin\_yang; anchor; camera}
\end{quote}

It contains objects, instruments, and symbols drawn from Caltech101.

\paragraph{Replacement Pool.}

The replacement names are different ImageNet classes:

\begin{quote}
\texttt{brambling; goldfinch; house finch; junco;}\\
\texttt{indigo bunting; American robin; bulbul; jay}
\end{quote}

These names are inserted cumulatively and in a fixed order. This controls
which names change and ensures that the only systematic variable is the
replacement percentage.


\begin{table}[htbp]
\centering
\caption{Class-name conditions supplied to DMP at each replacement level.}
\label{tab:class-name-replacement-conditions}
\small
\begin{tabularx}{\textwidth}{
    >{\centering\arraybackslash}p{0.13\textwidth}
    >{\raggedright\arraybackslash}X
}
\toprule
\textbf{Replacement}
&
\textbf{Class-name condition supplied to DMP} \\
\midrule
0\%
&
accordion, bonsai, chandelier, euphonium, menorah, saxophone,
stapler, yin\_yang, anchor, camera
\\[0.5em]

20\%
&
brambling, goldfinch, chandelier, euphonium, menorah, saxophone,
stapler, yin\_yang, anchor, camera
\\[0.5em]

40\%
&
brambling, goldfinch, house finch, junco, menorah, saxophone,
stapler, yin\_yang, anchor, camera
\\[0.5em]

60\%
&
brambling, goldfinch, house finch, junco, indigo bunting,
American robin, stapler, yin\_yang, anchor, camera
\\[0.5em]

80\%
&
brambling, goldfinch, house finch, junco, indigo bunting,
American robin, bulbul, jay, anchor, camera
\\
\bottomrule
\end{tabularx}
\end{table}

\subsubsection{Ablation Study on Alternative methods for Meta-Prompting.}\label{app:ablation-autoregressive}
The first two rows of Table \ref{tab:supp_copromptcomparison} show the ablation study of using Transformer and GPT-2 for modeling the distribution of CoOp prompts across the 11 datasets.  We include the Transformer as a non-generative baseline and GPT-2 as an autoregressive generative model. For the Transformer, we use a pretrained RoBERTa-Base model, which is finetuned using LoRA to predict prompt embeddings from text conditions. A linear layer is added on top of the final layer to produce output embeddings of the target dimension (2048 for CoOp/CoPrompt), and training is done with MSE loss. The table shows that transformer tends to overfit to the base classes as it is able to closely match the performance of the baseline CoOp on the base classes while under-performing on the novel classes leading to a decrease in the overall performance. For GPT-2, we discretize the continuous prompt embeddings by identifying their top-5 nearest tokens in the CLIP embedding space. These tokens are then converted back to text and used to finetune a pretrained GPT-2 model with LoRA, trained to predict the top-5 tokens using standard cross-entropy loss. The model is conditioned on the same text inputs as used in the diffusion counterpart. Although each text condition has 40 associated prompts from different initializations (as described in Section~\ref{sec:coop-coprompt}), the resulting tokens after discretization are almost identical across seeds. This indicates that the variation captured in the continuous embedding space is lost during the discretization process-a known limitation, as discretization inherently reduces information. The table reflects this observation and shows that GPT-2 based modeling is inferior to diffusion since diffusion is much better for modeling continuous distribution of prompts. Moreover, unlike autoregressive methods, diffusion offers multiple benefits such as classifier-free guidance, negative prompting, inversion, editing and composition that are challenging or infeasible with models like GPT-2. These results show that modeling the prompt distribution with diffusion is more effective and flexible than using autoregressive methods.

\subsubsection{Ablation on the number of prompts per concept}
We already include the ablation study when using only 20 prompts per concept for identity synthesis in Table~\ref{tab:comp_ids}. Here, we include  an additional ablation study on the number of prompts for classification tasks in Table~\ref{tab:ablationpromptnumber}. It shows that DMP works well even with as few as 5 or 10 prompts per task/concept. For 5 prompts per class, DMPCoOp obtains +1.2\% gain on new classes and it increases as $n$ increases (+1.6/+2.5/+3.0 for $n=$10/20/40 prompts respectively). The results are consistent with our theoretical guarantee discussed in Appendix A.1 where higher $n$ corresponds to lower downstream risk and better generalization. 
\begin{figure*}[h]
\centering
\scriptsize
\setlength{\tabcolsep}{2pt}
\captionof{table}{Ablation study with different number of prompts per concept. Base2new generalization per dataset: performance for All, Base, and New classes, and HM (Harmonic Mean). The results reported are the average over three seed runs.}
\label{tab:ablationpromptnumber}
\begin{adjustbox}{max width=\textwidth}
\begin{tabular}{@{}l|cccc|cccc|cccc|cccc@{}}
\toprule
 & \multicolumn{4}{c|}{\textbf{(a) Average}} & \multicolumn{4}{c|}{\textbf{(b) ImageNet}} & \multicolumn{4}{c|}{\textbf{(c) Caltech101}} & \multicolumn{4}{c}{\textbf{(d) OxfordPets}} \\ 
\cmidrule{2-17}
\textbf{Method} & All & Base & New & HM & All & Base & New & HM & All & Base & New & HM & All & Base & New & HM \\ 
\midrule
CoOp~\citep{zhou2022coop} & 67.1 & 80.9 & 68.6 & 74.2 & 68.0 & 76.2 & 66.8 & 71.2 & 93.8 & 98.1 & 93.0 & 95.5 & 90.6 & 95.2 & 96.5 & 95.8 \\
DMPCoOp & \textbf{70.1} & 80.3 & \textbf{73.4} & 76.5 & 68.7 & 75.4 & \textbf{68.8} & 72.0 & \textbf{94.8} & \textbf{98.3} & \textbf{95.3} & \textbf{96.8} & 91.7 & 95.4 & 97.3 & 96.3 \\
DMPCoOp (5 prompts) & 69.1 & 81.4 & 71.6 & 75.9 & 68.9 & 76.6 & 68.0 & 72.0 & \textbf{94.8} & 97.9 & 95.2 & 96.5 & 91.1 & 95.4 & 97.1 & 96.2 \\
DMPCoOp (10 prompts) & 69.3 & \textbf{81.8} & 72.0 & 76.3 & 69.0 & \textbf{76.6} & 68.5 & 72.3 & 94.5 & 98.2 & 95.2 & 96.7 & 92.9 & \textbf{96.0} & \textbf{97.5} & \textbf{96.7} \\
DMPCoOp (20 prompts) & 69.1 & 81.2 & 72.9 & \textbf{76.6} & \textbf{69.3} & 76.4 & \textbf{68.8} & \textbf{72.4} & 94.3 & 98.2 & 95.1 & 96.6 & \textbf{93.0} & \textbf{96.0} & \textbf{97.5} & \textbf{96.7} \\
\midrule

 & \multicolumn{4}{c|}{\textbf{(e) StanfordCars}} & \multicolumn{4}{c|}{\textbf{(f) Flowers102}} & \multicolumn{4}{c|}{\textbf{(g) Food101}} & \multicolumn{4}{c}{\textbf{(h) FGVC Aircraft}} \\
\cmidrule{2-17}
\textbf{Method} & All & Base & New & HM & All & Base & New & HM & All & Base & New & HM & All & Base & New & HM \\ 
\midrule
CoOp~\citep{zhou2022coop} & 68.3 & \textbf{75.8} & 68.2 & 71.8 & 72.2 & 96.2 & 66.7 & 78.8 & 83.8 & 89.8 & 88.0 & 88.9 & 25.9 & 37.0 & 30.4 & 33.4 \\
DMPCoOp & \textbf{69.3} & 74.5 & \textbf{72.0} & \textbf{73.2} & \textbf{75.6} & 93.4 & \textbf{72.2} & 81.4 & \textbf{86.4} & 89.6 & 89.9 & 89.7 & 26.3 & 35.7 & 31.6 & 33.5 \\
DMPCoOp (5 prompts) & 68.8 & 75.2 & 70.8 & 72.9 & 74.5 & 96.0 & 71.1 & 81.7 & 84.5 & 89.5 & 89.1 & 89.3 & \textbf{27.9} & \textbf{37.5} & \textbf{33.3} & \textbf{35.3} \\
DMPCoOp (10 prompts) & 68.5 & \textbf{75.8} & 69.9 & 72.7 & 73.7 & 95.7 & 69.8 & 80.7 & 85.5 & 89.8 & \textbf{91.2} & \textbf{90.5} & 27.8 & 37.0 & 32.0 & 34.3 \\
DMPCoOp (20 prompts) & 68.7 & 75.0 & 70.8 & 72.8 & 75.9 & \textbf{96.4} & 72.1 & \textbf{82.5} & 85.4 & \textbf{90.0} & 91.1 & \textbf{90.5} & \textbf{27.9} & 37.3 & 31.6 & 34.2 \\
\midrule

 & \multicolumn{4}{c|}{\textbf{(i) SUN397}} & \multicolumn{4}{c|}{\textbf{(j) DTD}} & \multicolumn{4}{c|}{\textbf{(k) EuroSAT}} & \multicolumn{4}{c}{\textbf{(l) UCF101}} \\  
\cmidrule{2-17}
\textbf{Method} & All & Base & New & HM & All & Base & New & HM & All & Base & New & HM & All & Base & New & HM \\ 
\midrule
CoOp~\citep{zhou2022coop} & 67.2 & \textbf{80.7} & 71.8 & 76.0 & 51.0 & \textbf{78.5} & 50.5 & 61.5 & 49.7 & 78.3 & 58.3 & 66.8 & 68.1 & 84.0 & 64.2 & 72.8 \\
DMPCoOp & 67.3 & 80.4 & 72.9 & 76.5 & \textbf{53.7} & 75.1 & \textbf{55.8} & \textbf{64.0} & \textbf{65.9} & 85.6 & \textbf{80.7} & \textbf{83.1} & 71.7 & 79.6 & 70.7 & 74.9 \\
DMPCoOp (5 prompts) & 66.8 & 79.8 & 71.7 & 75.5 & 52.1 & 76.4 & 54.3 & 63.5 & 59.1 & 85.7 & 67.2 & 75.3 & 71.1 & \textbf{85.6} & 69.6 & 76.8 \\
DMPCoOp (10 prompts) & \textbf{67.6} & 80.3 & \textbf{73.4} & \textbf{76.7} & 52.1 & 75.6 & 53.4 & 62.6 & 60.1 & \textbf{89.0} & 71.7 & 79.4 & 70.7 & 85.3 & 69.2 & 76.4 \\
DMPCoOp (20 prompts) & \textbf{67.6} & 80.4 & 73.3 & \textbf{76.7} & 51.4 & 75.6 & 54.1 & 63.1 & 54.6 & 84.8 & 74.4 & 79.3 & \textbf{71.8} & 82.9 & \textbf{72.6} & \textbf{77.4} \\

\bottomrule
\end{tabular}
\end{adjustbox}
\end{figure*}

\subsubsection{Ablation on the text inputs to DMPCoOp model}\label{app:ablationtextinputs}

We conducted an ablation study using the dataset names as the prompt or text condition to the DMPCoOp model instead of the classnames. Table \ref{table:ablation_coopcomparison} shows that the performance of the model using dataset names is better than the baseline CoOp prompts by 0.3\% on average while it is 0.7\% lower as compared to the model using classnames. This shows that using class-specific names generates prompts with robust generalization than just using a single dataset name as the text condition.

\subsubsection{Ablation on the quality of training data}

To investigate this, we assessed the impact of the quality of training data by training DMP with noisy prompts.  We apply additive Gaussian noise ($n_i$) with a standard deviation of $\sigma = 0.01$ to the $i^\text{th}$ prompt $x(c)_i$ as 
\[
x(c)_i = (1-\sigma)x(c)_i + \sigma n_i.
\] 
The noisy prompt is applied to $x\%$ of the training dataset, where $x \in \{10, 40\}$. 

Table~\ref{tab:noisyprompt} summarizes the experiment where random noise is added to the prompts in the training repository. Meta-prompting is observed to be robust to noise levels ranging from 10\% to 40\% of the training prompts. DMPCoOp trained with 10\% noisy prompts still obtains +1.9\% improvement over the baseline. Further, the average H.M for the DMP model with 10\% noisy prompts is slightly better (76.7 vs 76.5 for DMPCoOp) than the DMPCoOp model trained on clean prompts suggesting that a small amount of noise can also help with generalization. This is similar to image diffusion models, which are also known to be robust to noise added during training.

\begin{figure*}[h]
\centering
\scriptsize
\setlength{\tabcolsep}{2pt}
\captionof{table}{Base2new generalization per dataset: performance for All, Base, and New classes, and HM (Harmonic Mean). The results reported are the average over three seed runs.}
\label{tab:noisyprompt}
\begin{adjustbox}{max width=\textwidth}
\begin{tabular}{@{}l|cccc|cccc|cccc|cccc@{}}
\toprule
 & \multicolumn{4}{c|}{\textbf{(a) Average}} & \multicolumn{4}{c|}{\textbf{(b) ImageNet}} & \multicolumn{4}{c|}{\textbf{(c) Caltech101}} & \multicolumn{4}{c}{\textbf{(d) OxfordPets}} \\ 
\cmidrule{2-17}
\textbf{Method} & All & Base & New & HM & All & Base & New & HM & All & Base & New & HM & All & Base & New & HM \\ 
\midrule
CoOp~\citep{zhou2022coop} & 67.1 & \textbf{80.9} & 68.6 & 74.2 & 68.0 & 76.2 & 66.8 & 71.2 & 93.8 & 98.1 & 93.0 & 95.5 & 90.6 & 95.2 & 96.5 & 95.8 \\ 
\textbf{DMPCoOp} & {\textbf{70.1}} & 80.3 & {\textbf{73.4}} & {76.5} & 68.7 & 75.4 & 68.8 & 72.0 & {\textbf{94.8}} & {\textbf{98.3}} & {\textbf{95.3}} & {\textbf{96.8}} & 91.7 & 95.4 & 97.3 & 96.3 \\
\textbf{DMPCoOp (10\% noise)} & 69.6 & 81.7 & 72.3 & \textbf{76.7} & {\textbf{69.2}} & 76.5 & {\textbf{69.0}} & {\textbf{72.6}} & 94.0 & 98.1 & 94.2 & 96.1 & 91.5 & 95.0 & 96.9 & 95.9 \\
\textbf{DMPCoOp (40\% noise)} & 68.9 & 81.2 & 70.3 & 75.0 & 68.5 & {\textbf{76.6}} & 67.3 & 71.6 & 94.1 & 97.9 & 95.2 & 96.5 & {\textbf{92.8}} & {\textbf{95.7}} & {\textbf{97.9}} & {\textbf{96.8}} \\
\midrule
 & \multicolumn{4}{c|}{\textbf{(e) StanfordCars}} & \multicolumn{4}{c|}{\textbf{(f) Flowers102}} & \multicolumn{4}{c|}{\textbf{(g) Food101}} & \multicolumn{4}{c}{\textbf{(h) FGVC Aircraft}} \\
\cmidrule{2-17}
\textbf{Method} & All & Base & New & HM & All & Base & New & HM & All & Base & New & HM & All & Base & New & HM \\ 
\midrule
CoOp~\citep{zhou2022coop} & 68.3 & 75.8 & 68.2 & 71.8 & 72.2 & 96.2 & 66.7 & 78.8 & 83.8 & 89.8 & 88.0 & 88.9 & 25.9 & 37.0 & 30.4 & 33.4 \\
\textbf{DMPCoOp} & \textbf{69.3} & 74.5 & \textbf{72.0} & \textbf{73.2} & 75.6 & 93.4 & 72.2 & 81.4 & \textbf{86.4} & 89.6 & 89.9 & 89.7 & 26.3 & 35.7 & 31.6 & 33.5 \\
\textbf{DMPCoOp (10\% noise)} & 69.1 & \textbf{76.7} & 69.4 & 72.9 & 77.1 & \textbf{96.4} & 72.8 & 83.0 & 85.6 & \textbf{89.9} & \textbf{91.1} & \textbf{90.5} & \textbf{27.7} & \textbf{37.6} & \textbf{33.8} & \textbf{35.6} \\
\textbf{DMPCoOp (40\% noise)} & 68.3 & 76.6 & 69.1 & 72.7 & \textbf{79.4} & 95.9 & \textbf{75.0} & \textbf{84.2} & 84.6 & 89.4 & 90.3 & 89.8 & 26.4 & 36.6 & 32.8 & 34.6 \\

\midrule
 & \multicolumn{4}{c|}{\textbf{(i) SUN397}} & \multicolumn{4}{c|}{\textbf{(j) DTD}} & \multicolumn{4}{c|}{\textbf{(k) EuroSAT}} & \multicolumn{4}{c}{\textbf{(l) UCF101}} \\  
\cmidrule{2-17}
\textbf{Method} & All & Base & New & HM & All & Base & New & HM & All & Base & New & HM & All & Base & New & HM \\ 
\midrule
CoOp~\citep{zhou2022coop} & 67.2 & \textbf{80.7} & 71.8 & 76.0 & 51.0 & 78.5 & 50.5 & 61.5 & 49.7 & 78.3 & 58.3 & 66.8 & 68.1 & \textbf{84.0} & 64.2 & 72.8 \\
\textbf{DMPCoOp} & \textbf{67.3} & 80.4 & \textbf{72.9} & \textbf{76.5} & \textbf{53.7} & 75.1 & \textbf{55.8} & \textbf{64.0} & \textbf{65.9} & \textbf{85.6} & \textbf{80.7} & \textbf{83.1} & 71.7 & 79.6 & 70.7 & 74.9 \\
\textbf{DMPCoOp (10\% noise)} & 66.2 & \textbf{80.7} & 70.5 & 75.3 & 53.5 & \textbf{79.9} & 51.9 & 62.9 & 59.3 & 84.5 & 71.6 & 77.5 & \textbf{72.5} & 83.0 & \textbf{74.0} & \textbf{78.2} \\
\textbf{DMPCoOp (40\% noise)} & 65.9 & 79.7 & 69.8 & 74.4 & 49.8 & 78.4 & 48.6 & 60.0 & 59.6 & 84.0 & 62.8 & 71.9 & 68.4 & 82.3 & 64.1 & 72.1 \\

\bottomrule
\end{tabular}
\end{adjustbox}

\end{figure*}

\subsection{Ablation on DMPMulti}

We trained separate DMP models for identity synthesis and slider synthesis to compare their performance with the DMPMulti model, which was trained to generate both prompt types simultaneously. Table \ref{tab:slidersynthetic_app} presents the results of the DMPSlider model, trained solely for slider prompt generation. Table \ref{tab:comp_ids_app} presents the results of the DMPIdentity model, trained solely for identity prompt generation. The results indicate that its performance is comparable to that of the DMPMulti model, demonstrating that multi-task training in DMPMulti does not compromise its effectiveness.

\subsubsection{Ablation on Identity Composition with Textual Inversion}

For identity composition, we perform an ablation study using Stable Diffusion with Equation \ref{eq:7}, generating new identities by combining prompts such as "a photo of id-1" and "a photo of id-2." Figure \ref{fig:supp_ablate_ti_id_compose} illustrates the results of this process using Textual Inversion prompts with the Stable Diffusion v1.5 model. The generated images are often noisy, distorted, and tend to replicate the input identities rather than effectively merging their attributes. Additionally, running a full forward diffusion process with multiple identities doubles the inference time from 4 to 8 seconds per image. In contrast, our DMPMulti achieves high-quality identity compositions in approximately 5 seconds, introducing only a 1-second overhead compared to standard Stable Diffusion.

\subsection{DMP for Personalization}\label{app:identitymetaprompts}

The DMPVariation model enables the generation of prompt variations conditioned on a given Textual Inversion prompt, making it possible to train a text-to-prompt meta-diffusion model as a replacement for personalization prompt repositories. Users only require access to existing prompt repositories, as the DMPVariation model can directly generate the necessary intermediate embeddings. We choose the top-k ($k=40$) embeddings generated from the DMPVariation model based on the cosine similarity with the available TI prompts. These embeddings can serve as training data for generating personalized prompts based on textual input. Once trained, this approach eliminates the need to search and retrieve prompts from a database, allowing for on-the-fly prompt generation.

\begin{table*}[t]
\centering
\scriptsize
\begin{tabular}{lrrr}
\hline
\textbf{Concept} & \textbf{DMP} & \textbf{TI} & \textbf{Delta} \\
\hline
ettblackteapot & 0.2905 & 0.2778 & +0.0127 \\
dicoo & 0.2837 & 0.2700 & +0.0137 \\
goku & 0.2830 & 0.2800 & +0.0030 \\
degodsheavy & 0.2827 & 0.2754 & +0.0073 \\
spider-gwen & 0.2825 & 0.2810 & +0.0015 \\
johnny-silverhand & 0.2812 & 0.2760 & +0.0052 \\
blue-haired-boy & 0.2812 & 0.2800 & +0.0012 \\
bullvbear & 0.2798 & 0.2683 & +0.0115 \\
black-waifu & 0.2790 & 0.2715 & +0.0075 \\
a-female-hero-from-the-legend-of-mir & 0.2788 & 0.2710 & +0.0078 \\
freddy-fazbear & 0.2786 & 0.2750 & +0.0036 \\
ldrs & 0.2783 & 0.2642 & +0.0141 \\
chonkfrog & 0.2778 & 0.2756 & +0.0022 \\
concept-art & 0.2770 & 0.2715 & +0.0055 \\
degods & 0.2769 & 0.2686 & +0.0083 \\
hanfu-anime-style & 0.2766 & 0.2703 & +0.0063 \\
joemad & 0.2764 & 0.2673 & +0.0091 \\
dog & 0.2761 & 0.2780 & -0.0019 \\
stuffed-penguin-toy & 0.2760 & 0.2756 & +0.0004 \\
fox-purple & 0.2760 & 0.2734 & +0.0026 \\
colossus & 0.2756 & 0.2634 & +0.0122 \\
chungus-poodl-pet & 0.2754 & 0.2637 & +0.0117 \\
anya-forger & 0.2754 & 0.2737 & +0.0017 \\
furrpopasthetic & 0.2751 & 0.2737 & +0.0014 \\
bob-dobbs & 0.2751 & 0.2664 & +0.0087 \\
eddie & 0.2751 & 0.2607 & +0.0144 \\
arthur1 & 0.2750 & 0.2666 & +0.0084 \\
dragonborn & 0.2750 & 0.2556 & +0.0194 \\
kay & 0.2747 & 0.2605 & +0.0142 \\
tesla-bot & 0.2747 & 0.2634 & +0.0113 \\
borderlands & 0.2747 & 0.2637 & +0.0110 \\
lavko & 0.2744 & 0.2588 & +0.0156 \\
gim & 0.2740 & 0.2659 & +0.0081 \\
hubris-oshri & 0.2740 & 0.2659 & +0.0081 \\
tubby & 0.2737 & 0.2673 & +0.0064 \\
finn-token & 0.2737 & 0.2734 & +0.0003 \\
moxxi & 0.2730 & 0.2722 & +0.0008 \\
altvent & 0.2730 & 0.2676 & +0.0054 \\
omlettehaai & 0.2730 & 0.2600 & +0.0130 \\
jos-de-kat & 0.2730 & 0.2651 & +0.0079 \\
loab-style & 0.2730 & 0.2573 & +0.0157 \\
crinos-form-garou & 0.2727 & 0.2693 & +0.0034 \\
blue-zombie & 0.2727 & 0.2693 & +0.0034 \\
cgdonny1 & 0.2725 & 0.2610 & +0.0115 \\
amogus & 0.2725 & 0.2551 & +0.0174 \\
lucky-luke & 0.2725 & 0.2751 & -0.0026 \\
captain-haddock & 0.2725 & 0.2725 & +0.0000 \\
manga-nov-23 & 0.2725 & 0.2551 & +0.0174 \\
button-eyes & 0.2722 & 0.2683 & +0.0039 \\
bruma & 0.2722 & 0.2632 & +0.0090 \\
ouroboros & 0.2720 & 0.2734 & -0.0014 \\
fursona & 0.2720 & 0.2646 & +0.0074 \\
kanovt & 0.2720 & 0.2522 & +0.0198 \\
doc & 0.2717 & 0.2693 & +0.0024 \\
warhammer-40k-drawing-style & 0.2715 & 0.2727 & -0.0012 \\
nard-style & 0.2715 & 0.2698 & +0.0017 \\
baluchitherian & 0.2715 & 0.2617 & +0.0098 \\
insidewhale & 0.2715 & 0.2630 & +0.0085 \\
irasutoya & 0.2715 & 0.2598 & +0.0117 \\
devonm & 0.2712 & 0.2617 & +0.0095 \\
edgerunners-style-v2 & 0.2712 & 0.2660 & +0.0052 \\
nixeu & 0.2710 & 0.2660 & +0.0050 \\
shek-9-12-opening & 0.2710 & 0.2708 & +0.0002 \\
loab-character & 0.2710 & 0.2656 & +0.0054 \\
ldr & 0.2710 & 0.2666 & +0.0044 \\
malika-favre-art-style & 0.2710 & 0.2632 & +0.0078 \\
drive-scorpion-jacket & 0.2710 & 0.2686 & +0.0024 \\
apulian-rooster-v0-1 & 0.2710 & 0.2351 & +0.0359 \\
fftstyle & 0.2708 & 0.2593 & +0.0115 \\
alf & 0.2708 & 0.2664 & +0.0044 \\
wheelchair & 0.2708 & 0.2730 & -0.0022 \\
obama-self-2 & 0.2705 & 0.2705 & +0.0000 \\
dog-chip & 0.2703 & 0.2676 & +0.0027 \\
cheburashka & 0.2700 & 0.2705 & -0.0005 \\
ihylc & 0.2700 & 0.2580 & +0.0120 \\
cat-toy & 0.2695 & 0.2600 & +0.0095 \\
\hline
\textbf{Average} & 0.2731 & 0.2663 & +0.0068 \\
\hline
\end{tabular}
\caption{HPSv2 comparison for 75 random concepts where each concept is evaluated for 6 different prompts.}
\label{tab:hps_topk}
\end{table*}

\begin{figure*}[t]
\centering
\includegraphics[width=0.8\textwidth, trim= 20 100 20 100, clip]{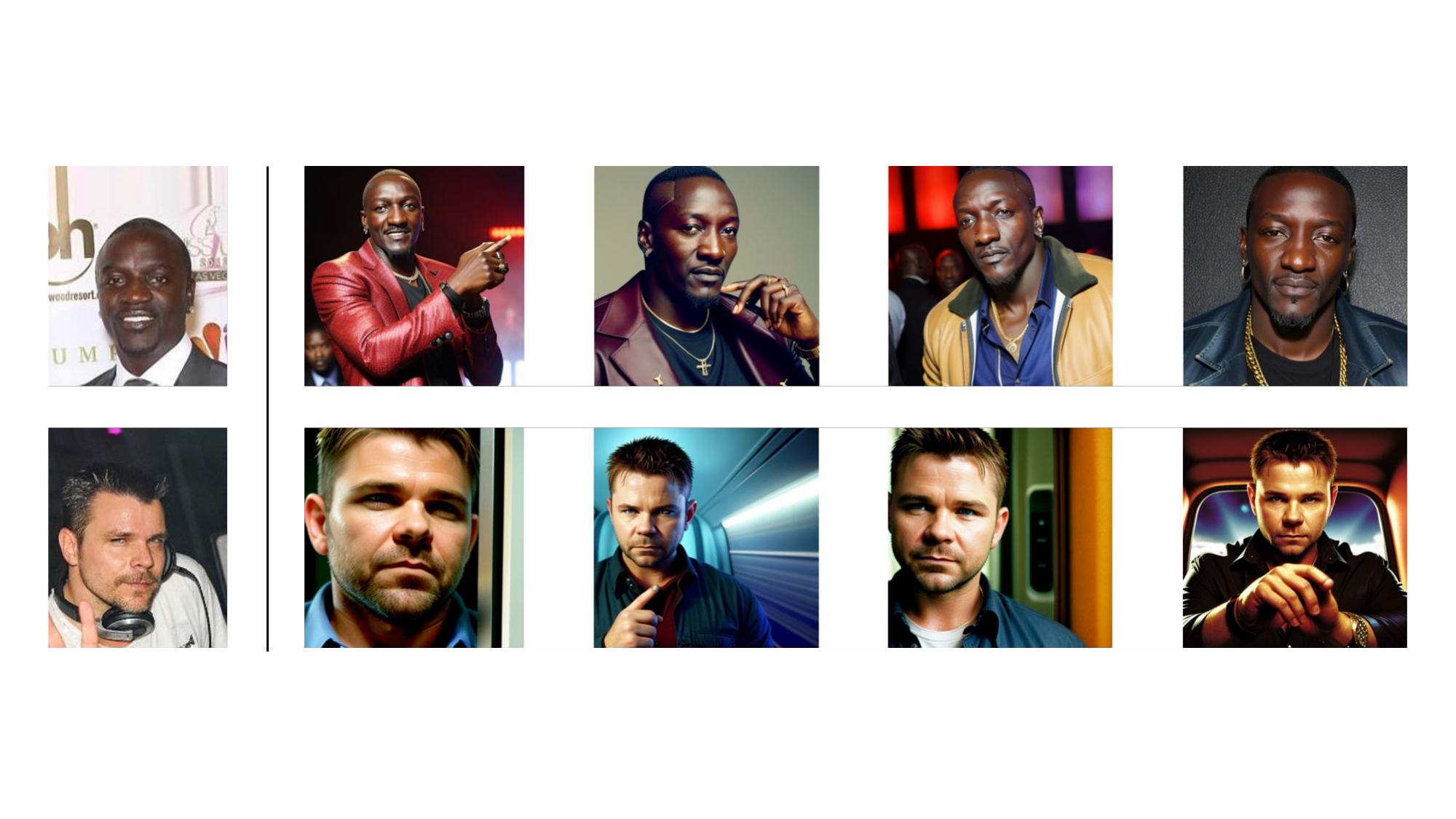}
\includegraphics[width=0.8\textwidth, trim= 20 100 20 100, clip]{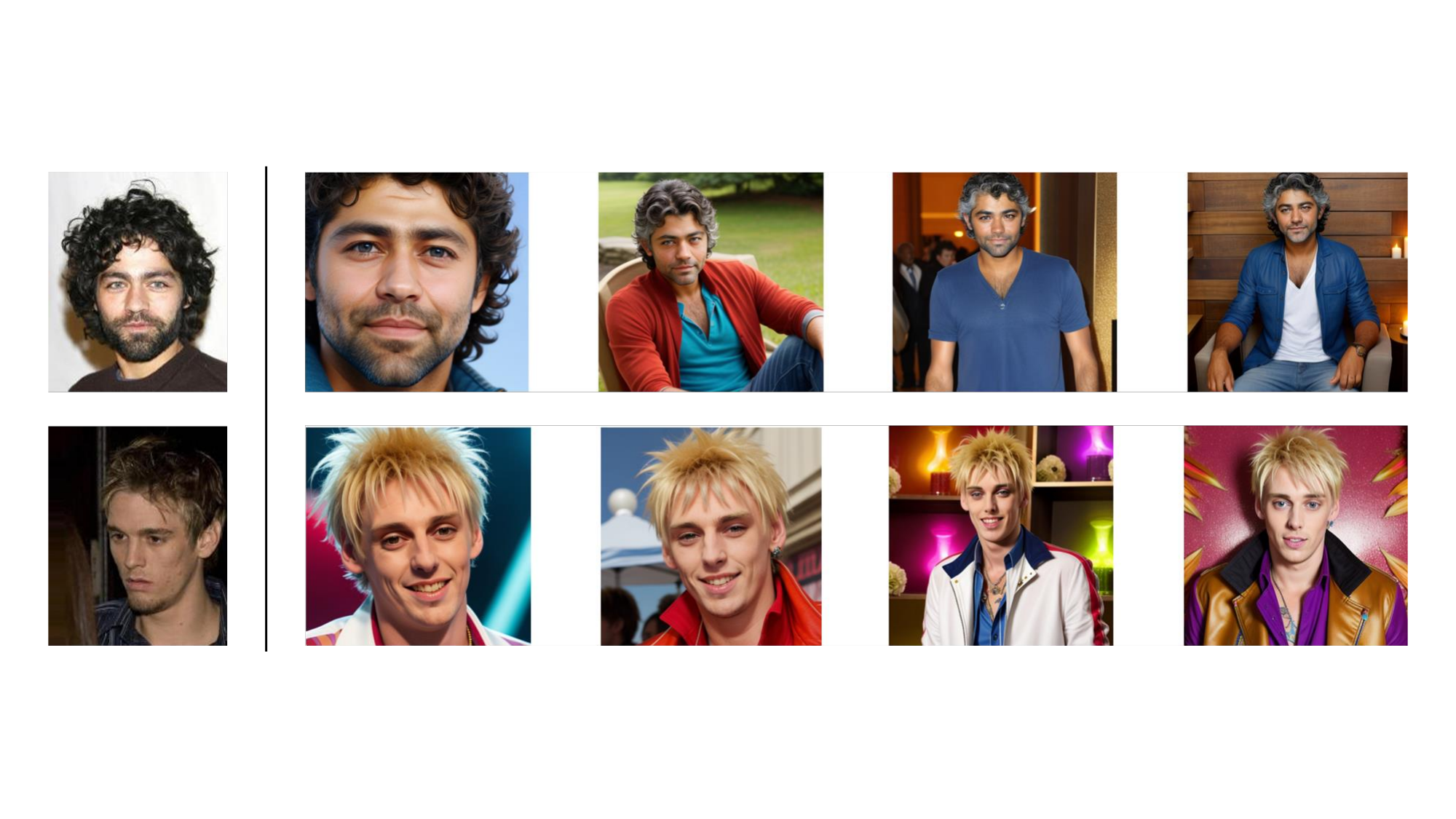}
\includegraphics[width=0.8\textwidth, trim= 20 100 20 100, clip]{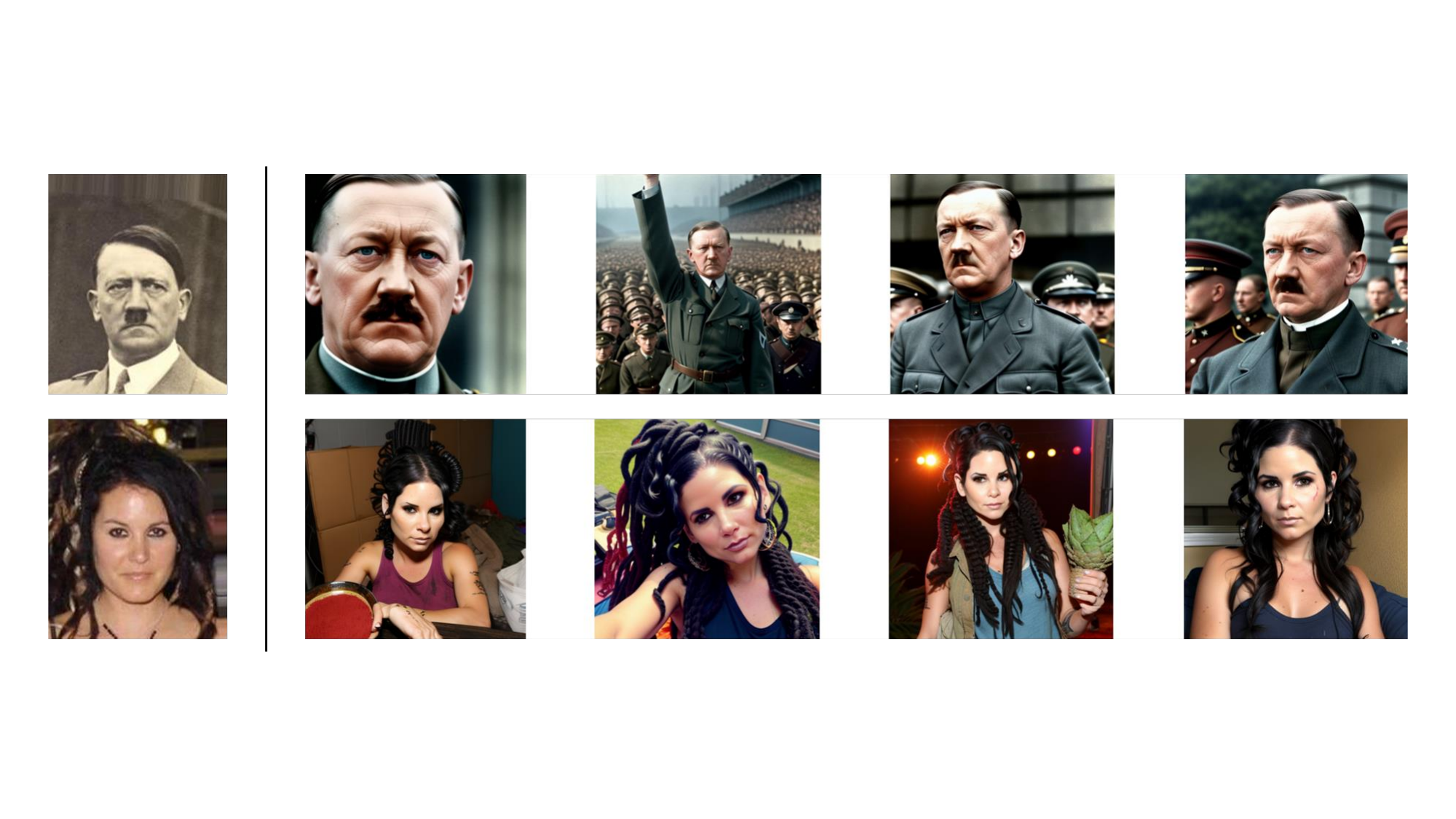}
\includegraphics[width=0.8\textwidth, trim= 20 100 20 100, clip]{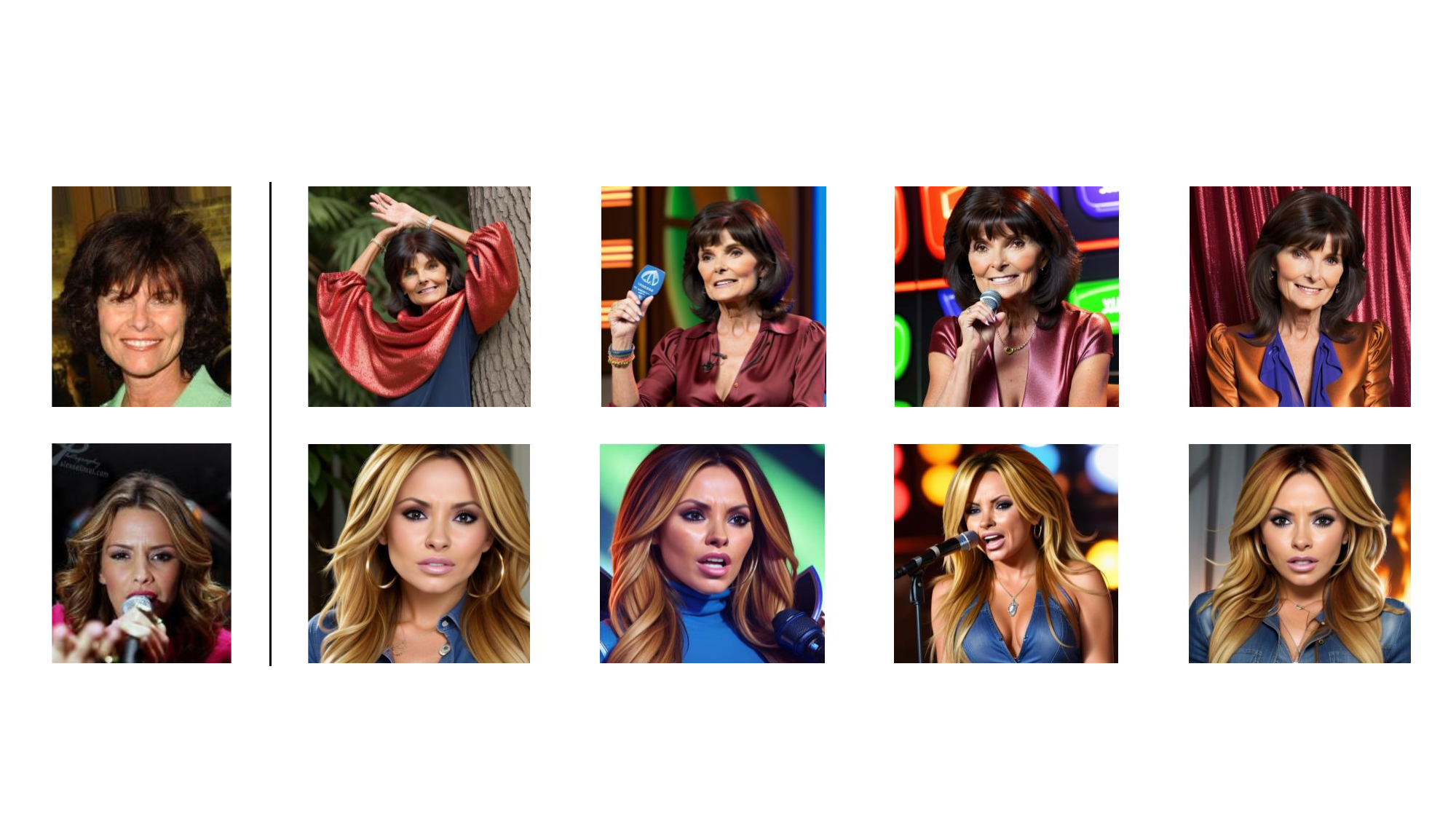}
\caption{\label{fig:supp_qual_syn_id_1} \textbf{Qualitative results} of image synthesis with DMPMulti prompts. Left: groundtruth images. Right: images synthesized by SD with prompts sampled from the DMPMulti  model.}
\end{figure*}

\begin{figure*}[t]
\centering
\begin{minipage}[t]{\columnwidth}
\centering
\includegraphics[width=\linewidth, trim=40 100 40 100, clip]
{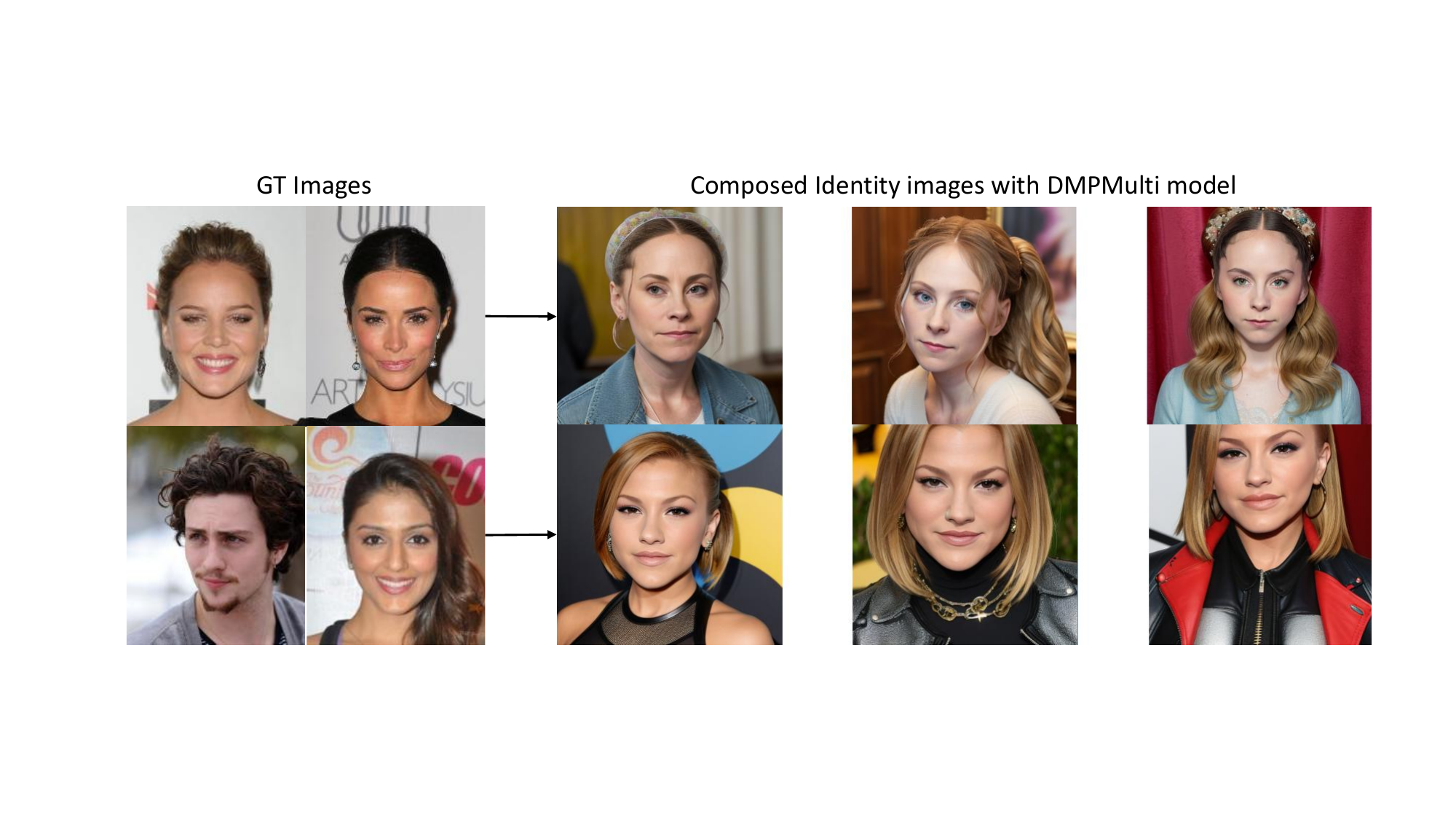}
\caption{\label{fig:supp_qual_syn_id_compose}
\textbf{DMPMulti composition:}
Images generated with DMPMulti model prompts (right) for the composition of the two identities shown on the left.}
\end{minipage}
\hfill
\begin{minipage}[t]{\columnwidth}
\centering
\includegraphics[width=\linewidth, trim=40 100 40 100, clip]
{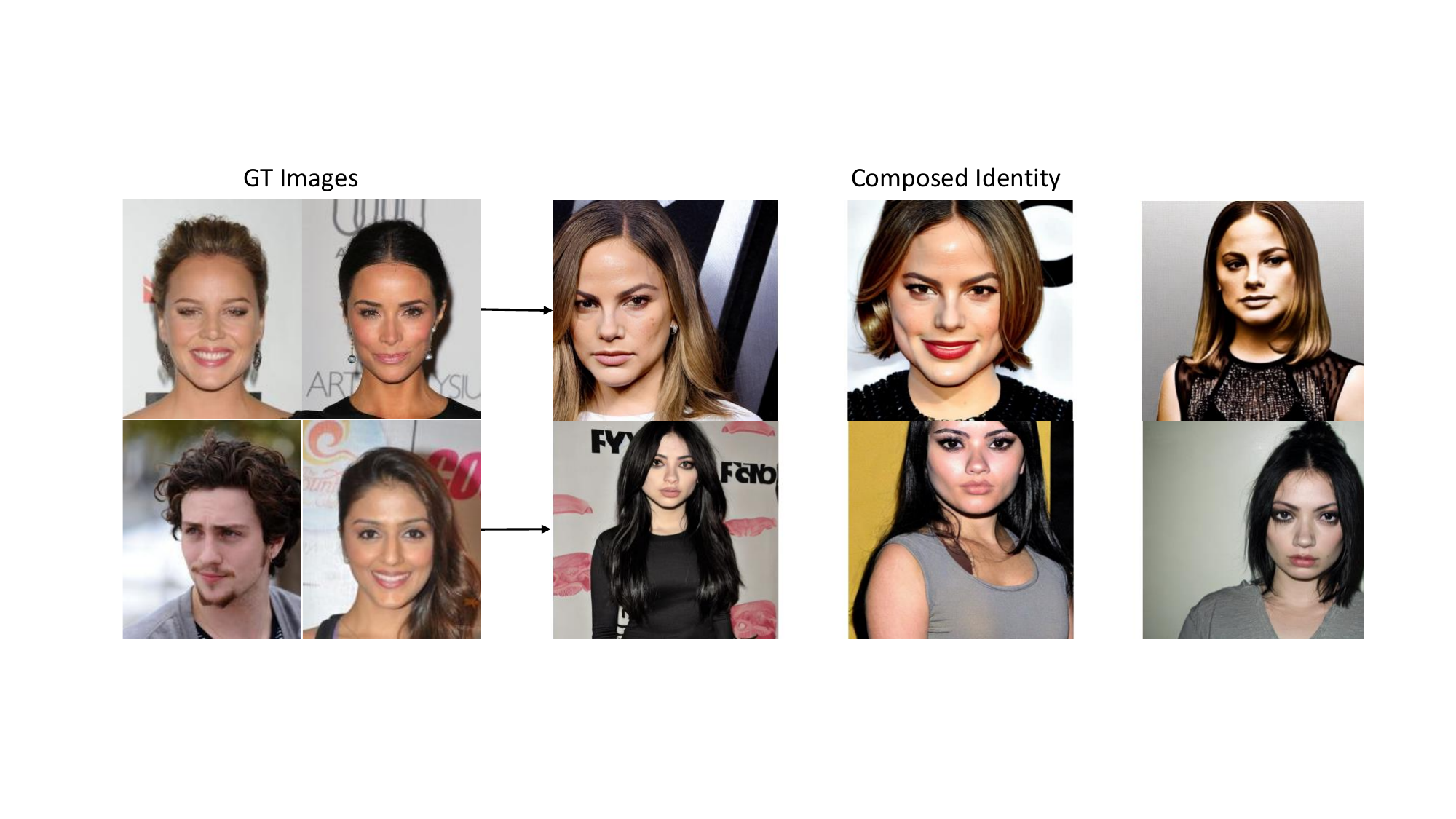}
\caption{\label{fig:supp_qual_syn_id_compose_sd15}
\textbf{DMPMulti composition with SDv1.5 model:}
Prompts synthesized by \textbf{DMP} generalize well to different downstream models sharing the same CLIP text encoder without any re-training.}
\end{minipage}
\end{figure*}

\begin{figure*}[t]
\centering
\includegraphics[width=\textwidth, trim= 0 200 0 200, clip]{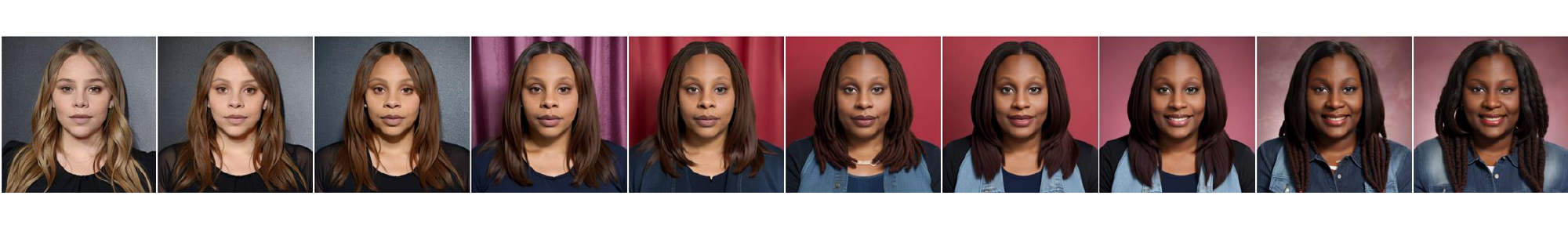}
\caption{\label{fig:supp_qual_syn_id_interpolate} \textbf{Interpolating} between two faces using DMPMulti with concept composition.}
\end{figure*}

\begin{figure*}[t]
\begin{center}
\captionof{table}{\label{tab:hyperparams}Hyperparameter Settings for the DMP models trained across three different tasks namely personalization, concepts and classification. }
\setlength{\tabcolsep}{4pt}
\begin{tabular}{ l| l | l | l | l}
\toprule
 &
{\centering DMPVariation} &
{\centering DMPMulti} &
{\centering DMPCoOp} &
{\centering Autoencoder} 
\\
\midrule
$z$-shape & - & - & 16 x 8 & 16 x 8 \\
$x$-shape & 768 x 1 & 768 x 1 & 2048 x 1 & 2048 x 1 \\
$|Z|$ & - & - & 128 & 128 \\
$|X|$ & 768 & 768 & 2048 & 768 \\
Diffusion steps & 1000 & 1000 & 1000 & 1000 \\
Optimizer & AdamW & AdamW & AdamW & AdamW\\
Noise Schedule & linear & linear & linear & linear\\
Nparams & 33M & 33M & 33M & 1M\\
Channels & 128 & 128 & 128 & 128\\
Depth & 1 & 1 & 1 & 1\\
Channel Multiplier & 1,2,2,2 & 1,2,2,2 & 1,2,2,2  & 1,1,1,1,2,2,2,2\\
Attention resolutions & 16, 8, 4 & 16, 8, 4 & 8, 4, 2 & -\\
Head Channels & 8 & 8 & 8 & 8\\
Batch Size & 320 & 320 & 256 & 128\\
Iterations & 100k & 100k & 100k & 100k \\
Learning Rate & 1e-6 & 1e-6 & 2.0e-7 & 4.5e-6 \\ 
\hline
\end{tabular}

\end{center}
\end{figure*}

\begin{figure*}[h]
\centering
\includegraphics[width=\textwidth, trim= 200 0 200 0, clip]{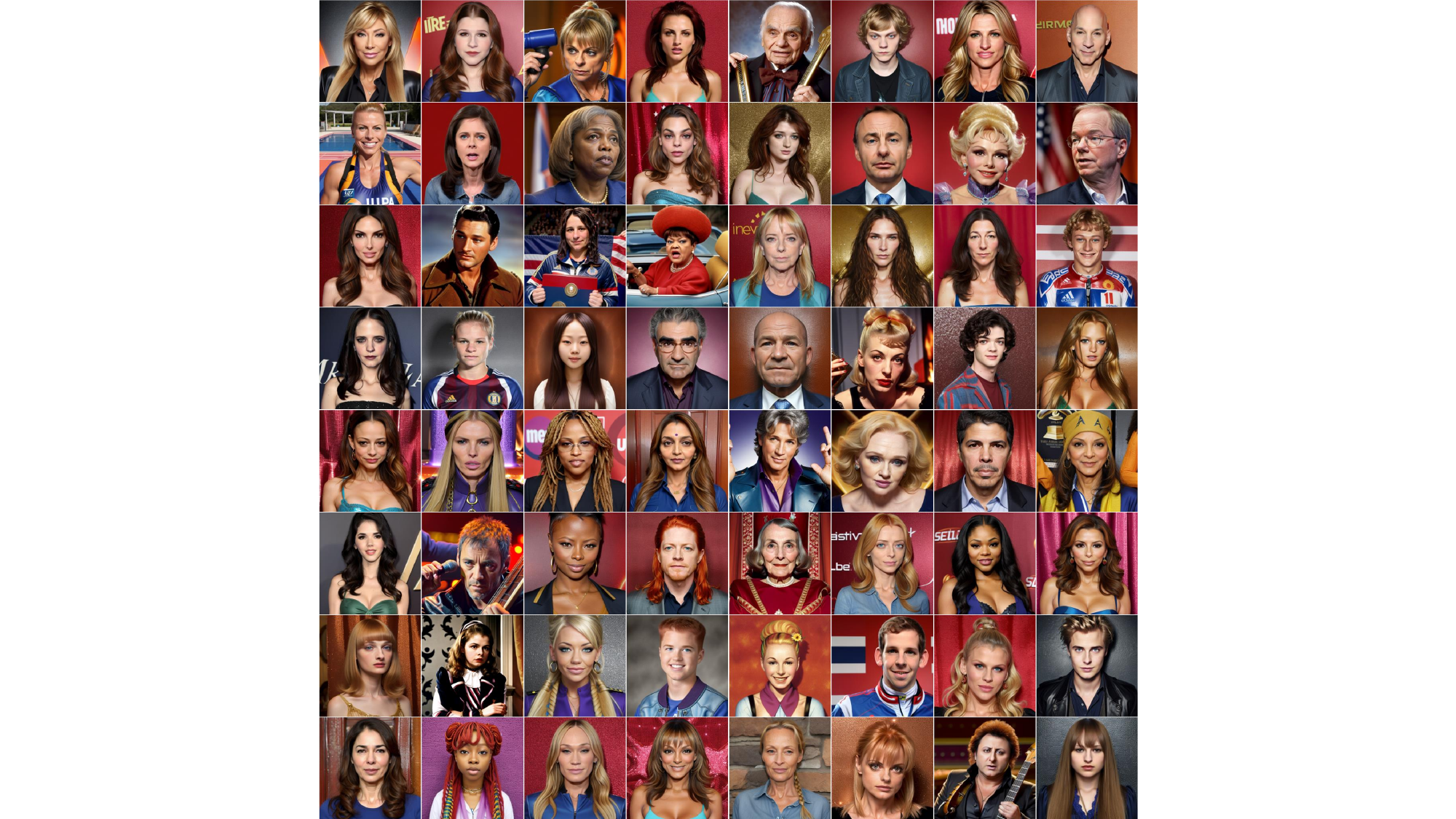}
\caption{\label{fig:qual_syn_id_novel} \textbf{Qualitative results of generating novel identities during inference} using random identity conditioning with DMPMulti  model. The prompt used for the Stable Diffusion Realistic Vision model is \textit{"A photo of a id-x"} where x is the id not present in the training set. The identities do not overfit (\textbf{mean face ID similarity of 0.0102 across training images}). Note that all the images use a fixed seed to the diffusion model.}
\end{figure*}

Figure \ref{fig:qual_syn_id} presents a comparison of the images generated by DMPMulti model for the identity labeled "id-38" and "id-96" respectively. The figure presents three classes of images: groundtruth on the left, synthesized by SD prompted by the original TI prompts in the first and third row of the right side,  and synthesized by SD prompted by the DMPMulti model, itself prompted for``identity-c." The images synthesized using DMPMulti model have quality comparable to those synthesized with TI prompts. 

\textbf{Negative Text Guidance.}  
Figure \ref{fig:qual_syn_id_neg} illustrates the effect of negative prompting by displaying images synthesized when DMPMulti  model is prompted with identity-21 (leftmost) as positive and identity-24 (second from left) as negative prompt. The generated identity exhibits contrasting characteristics, such as fuller cheeks, smaller eyes, and a broader nose—features to those of the negative identity (id-24).

\textbf{Novel Identities.}
Figure \ref{fig:supp_novel_closest_training_images} shows qualitative results of novel identities sampled by DMP and their closest training images. Figure \ref{fig:qual_syn_id_novel} shows additional qualitative results of novel identities sampled by DMPMulti.

\textbf{Identity Prompt Diffusion.}
Figure \ref{fig:supp_qual_syn_id_1} presents the qualitative results of various identities sampled by DMPMulti  model. The figure contains the groundtruth on the left, and synthesized by SD prompted by the DMPMulti, itself prompted for``identity-c" on the right. The images synthesized using both models reflect the original identities in the groundtruth images.

\textbf{Identity Composition.}  
Figure \ref{fig:supp_qual_syn_id_compose} demonstrates additional results of combining two identities, displayed on the left, using DMPMulti to synthesize identity prompts with (\ref{eq:7}). The synthesized identity clearly incorporates prominent features from both original faces, such as the nose and chin, resulting in a cohesive blend of attributes.

Figure \ref{fig:supp_qual_syn_id_compose_sd15} shows the results of identity composition using SDv1.5 checkpoint that uses the same CLIP text encoder as SD-Realistic Vision checkpoint. It shows that DMP performs effectively  without requiring retraining for this version. Since DMP was trained in CLIP text space, it generalizes to all models sharing the CLIP text encoder, eliminating the need for retraining on specific model versions.

\noindent\textbf{Interpolation.}
Figure \ref{fig:supp_qual_syn_id_interpolate} shows the qualitative results of interpolating between two faces using the DMPMulti model with classifier-free guidance scale between 0 to 5. The results show that DMPMulti enables fine-grained interpolation by simply manipulating the guidance scale.

\subsection{Limitations and Future Work}\label{app:limitations}
While the DMP framework unifies and improves prompt generation and generalization, simplifying deployment, the effectiveness of DMP is fundamentally constrained by the quality and expressiveness of the underlying prompt learning method used to construct the training repository. Second, DMP inherits the limitations of text prompts such as lack of fine-grained control and is sensitive to noise seeds similar to image diffusion models with a small variance between the sampled prompts. In our experiments, DMP maintained high accuracy even at different noise seeds with only a standard deviation of 0.7 \% from the mean accuracy for classification tasks.

Future research directions can address these limitations by exploring joint training of the meta-model and the downstream foundation models, potentially overcoming the performance ceiling imposed by existing prompt learning techniques. Other directions for future work can explore ways for distilling LoRA~\citep{hu2022lora} adapters into prompts and the design of a Meta-LoRA model, which synthesizes weight matrices instead of prompts (a more complex problem due to the large parameter cardinality of LoRA weights). 

\subsection{Implementation details}\label{app:implementation}
In this section, we describe the evaluation setup followed in our experiments and the rationale behind choosing the setup. We then describe the hyperparameter settings used to train all DMP models and finally the prompt format used in DMPMulti model.

\textbf{Evaluation Setup.}
For downstream model, we use stable diffusion~\citep{rombach2022high} \texttt{Realistic-Vision-v4} checkpoint using classifier-free guidance with a scale of $4.5$ and $30$ DDIM steps for the image synthesis. For SD-XL~\citep{sdxllorainference} model, we use a scale of $7.5$ with $20$ DDIM steps. Note that, because DMP prompts are introduced in the CLIP text encoder, they can be interchangeably used with any diffusion model using this encoder. Our choice of downstream diffusion model follows the original prompting methods. 

For models other than CoOp, additional weights or head layers are optimized to prompt the deeper layers of the CLIP encoders. Due to the large number of parameters associated with these weights, it is infeasible to train a diffusion model to synthesize these parameters. For example, the projection matrices of MaPLe have 3.55 million parameters while CoPrompt and TAC have 4.65 million parameters each. This is much larger than even the images produced by Stable diffusion (65536 parameter latent). In contrast, all text prompts have only 2048 or fewer parameters (only 256 parameters in the latent space). So, we only synthesize the text prompts attached to the input of CLIP model with DMP. During inference, we replace the learned textual prompt from the baseline method with the prompt sampled with DMP while keeping the other weights of the baseline method fixed. As followed in prompt learning literature, we pick three random seeds and report the average results. We note that different initialization seeds result in minor performance differences which explains the performance difference from the original paper reported results.

\subsubsection{Hyperparameter Settings}
Table \ref{tab:hyperparams} summarizes the detailed hyperparameter settings of the DMP models trained from scratch reported in the main paper.

\subsubsection{Evaluation Prompts for DMPMulti}\label{app:slidereval}

The example format of the prompts used for evaluation is shown below for the concept "Age",

\begin{itemize}
    \item A portrait of a woman with a warm smile, \{\}
    \item A person's face, \{\}
    \item A man sitting on a park bench, reminiscing about his youth, \{\}
    \item A couple of friends enjoying a picnic together, \{\}
    \item A photo of a person, \{\}
\end{itemize}

\end{document}